\documentclass{article}
\usepackage[affil-it]{authblk}

\usepackage{chihao}

\usepackage[margin=1in]{geometry}

\usepackage{graphicx}
\usepackage{threeparttable}
\usepackage{multirow}
\usepackage{array}
\usepackage{makecell}

\usepackage{xifthen}
\usepackage{booktabs}
\usepackage{multirow}
\usepackage{array}
\usepackage{makecell}
\allowdisplaybreaks[4] 
\usepackage{tikz}
\usepackage{rotating}
\usetikzlibrary{math}
\usetikzlibrary{positioning}
\usepackage{enumitem}

\renewcommand{\k}{\!{king}}
\renewcommand{\a}{\!{arm}}

\makeatletter
\DeclareRobustCommand{\checkbold}[1]{
 \edef\@tempa{\math@version}\edef\@tempb{bold}%
 \ifx\@tempa\@tempb%
  \def#1{1}%
 \else
  \def#1{0}%
 \fi}
\makeatother
\newcommand{\redarrow}[1][]{\operatorname{\begin{tikzpicture}[yscale=-0.038,xscale=0.03][#1] 
    \draw [color={rgb, 255:red, 201; green, 0; blue, 0}  ,draw opacity=1 ]   (108.8,128.07) .. controls (106.57,130.59) and (105.75,133.08) .. (105.11,136.04) ;
\draw [shift={(104.72,137.99)}, rotate = 280.8] [color={rgb, 255:red, 201; green, 0; blue, 0}  ,draw opacity=1 ][line width=0.75]    (4.37,-1.32) .. controls (2.78,-0.56) and (1.32,-0.12) .. (0,0) .. controls (1.32,0.12) and (2.78,0.56) .. (4.37,1.32)   ;
\end{tikzpicture}}}

\newcommand{\orangearrow}[1][]{\operatorname{\begin{tikzpicture}[yscale=-0.038,xscale=0.03][#1] 
    \draw [color={rgb, 255:red, 244; green, 164; blue, 96 }  ,draw opacity=1 ]   (28.75,128.27) .. controls (26.52,130.8) and (25.69,133.28) .. (25.06,136.25) ;
    \draw [shift={(24.67,138.2)}, rotate = 280.8] [color={rgb, 255:red, 244; green, 164; blue, 96 }  ,draw opacity=1 ][line width=0.75]    (4.37,-1.32) .. controls (2.78,-0.56) and (1.32,-0.12) .. (0,0) .. controls (1.32,0.12) and (2.78,0.56) .. (4.37,1.32)   ;
\end{tikzpicture}}}

\newcommand{\grayarrow}[1][]{\operatorname{\begin{tikzpicture}[yscale=-0.038,xscale=0.03][#1] 
    \draw [color={rgb, 255:red, 128; green, 128; blue, 128 }  ,draw opacity=1 ]   (104.73,141.11) .. controls (104.85,144.59) and (105.91,146.93) .. (107.42,149.43) ;
\draw [shift={(108.47,151.12)}, rotate = 237.5] [color={rgb, 255:red, 128; green, 128; blue, 128 }  ,draw opacity=1 ][line width=0.75]    (4.37,-1.32) .. controls (2.78,-0.56) and (1.32,-0.12) .. (0,0) .. controls (1.32,0.12) and (2.78,0.56) .. (4.37,1.32)   ;
\end{tikzpicture}}}

\newcommand{\ceil}[1]{\left\lceil #1 \right\rceil}

\usepackage{bbm}

\title{Understanding Memory-Regret Trade-Off for Streaming Stochastic Multi-Armed Bandits}
\author{Yuchen He}
\author{Zichun Ye}
\author{Chihao Zhang}
\affil{Shanghai Jiao Tong University}
\date{}

\begin{document}

\maketitle

\begin{abstract}
  We study the stochastic multi-armed bandit problem in the $P$-pass streaming model. In this problem, the $n$ arms are present in a stream and at most $m<n$ arms and their statistics can be stored in the memory. We give a complete characterization of the optimal regret in terms of $m, n$ and $P$. Specifically, we design an algorithm with $\tilde O\tp{(n-m)^{1+\frac{2^{P}-2}{2^{P+1}-1}} n^{\frac{2-2^{P+1}}{2^{P+1}-1}} T^{\frac{2^P}{2^{P+1}-1}}}$\footnote{In this article, the notations $\tilde O(\cdot), \tilde\Omega(\cdot)$ and $\tilde \Theta(\cdot)$ subsume a logarithmic factor in $n$ and $P$.} regret and complement it with an $\tilde \Omega\tp{(n-m)^{1+\frac{2^{P}-2}{2^{P+1}-1}} n^{\frac{2-2^{P+1}}{2^{P+1}-1}} T^{\frac{2^P}{2^{P+1}-1}}}$ lower bound when the number of rounds $T$ is sufficiently large. Our results are tight up to a logarithmic factor in $n$ and $P$. 
\end{abstract}
\newpage

\tableofcontents

\section{Introduction} 

The stochastic multi-armed bandit (MAB) problem is a widely studied online decision-making problem defined as follows. A player is given $n$ arms. For each $i\in [n]$, the $i$-th arm is associated with a reward distribution $\+D_i$ of mean $\mu_i$. In each round $t\in [T]$, the player picks one arm $A_t$ from the $n$ arms and then observes and gains a reward drawn from its associated reward distribution. The goal is to maximize the \emph{expected} cumulative reward in $T$ rounds, which is equivalent to minimizing the \emph{regret} $R(T) \defeq \E{\sum_{t=1}^T \mu_{\a^*} - \mu_{A_t}}$ where $\a^* = \argmax_{i\in [n]} \mu_i$ is the arm with the largest mean reward. The minimax regret of the MAB problem, namely the worst case regret of the best algorithm, is well-known to be of the order $\Theta(n^{\frac{1}{2}} T^{\frac{1}{2}})$ (see e.g.~\cite{LS20}).

 A recent line of work focuses on the MAB problems in the streaming model with memory constraint. In this model, the arms arrive one by one in a stream. The player has a memory of size $m$ available which is usually less than the number of arms $n$. Only the indices of $m$ arms and their corresponding statistics can be stored by the player and arms not in memory cannot be explored. Once an arm is discarded from the memory, all its information will be forgotten. We also allow $P$ passes over the stream for any $P>0$. We refer the reader to \Cref{sec:model} for a formal definition and some discussions on the model. 

The memory constraint arises new challenges for the exploration-exploitation trade-off in the classic MAB algorithms. In the single-pass setting, the work of~\cite{Wang23} proved that $\Omega\tp{n^{\frac{1}{3}}T^{\frac{2}{3}}}$ regret is necessary if $m\leq \frac{n}{20}$. They also designed an algorithm that achieves $O\tp{n^{\frac{1}{3}}T^{\frac{2}{3}}}$ regret using $\Theta\tp{\log^* n}$ memory. This is in contrast to the $\Theta\tp{n^{\frac{1}{2}}T^{\frac{1}{2}}}$ regret when no memory constraint is present.

The impact of memory constraints can be mitigated by allowing additional passes over the stream. Intuitively, if the number of passes is equal to $T$, then in each round, the player can always wait until the desired arm appears in the stream and explore it. This was formally justified by the results in~\cite{AKP22}. They proved a regret lower bound of $\Omega\tp{4^{-P} T^{\frac{2^P}{2^{P+1}-1}}}$ for $m=o\tp{\frac{n}{P^2}}$ and designed an algorithm with $O\tp{T^{\frac{2^P}{2^{P+1}-1}} \sqrt{nP\log T}}$ regret  when $m=O(1)$. Their results showed that the effect of the number of passes $P$ provides a smooth transition for the dependency on $T$ from $T^{\frac{2}{3}}$ when $P=1$ to $T^{\frac{1}{2}}$ when $P$ is sufficiently large.

However, the previous results do not reveal the whole picture of the trade-off among regret, memory, and the number of passes. First, the results hold when $m$ is small, i.e., $m\le c\cdot n$ for some universal constant $c<1$. The lower bound in~\cite{Wang23} implies that increasing the memory from $\Theta(\log^*n)$ to $\frac{n}{20}$ does not help in reducing the regret. However, it is not known whether further increase in memory beyond $\frac{n}{20}$ would affect the regret, even in the single-pass setting.  Secondly, in the multi-pass setting, there still exists a huge gap between the current upper and lower bounds. The correct dependency on $n$ and $P$ is unknown, and how the memory $m$ affects the regret remains unclear. Clarifying these issues has been left as open problems in~\cite{AKP22} as well.

In this paper, we give an almost complete answer to the relationship between the memory $m$, the number of arms $n$ and the number of passes $P$ in the regret. We design an algorithm with regret at most $\tilde O\tp{(n-m)^{1+\frac{2^{P}-2}{2^{P+1}-1}} n^{\frac{2-2^{P+1}}{2^{P+1}-1}} T^{\frac{2^P}{2^{P+1}-1}}}$ for any $n,m$ and $P$  and complements it with an $\tilde\Omega\tp{(n-m)^{1+\frac{2^{P}-2}{2^{P+1}-1}} n^{\frac{2-2^{P+1}}{2^{P+1}-1}} T^{\frac{2^P}{2^{P+1}-1}}}$ lower bound. Therefore, our results are optimal up to a logarithmic factor in $n$ and $P$. In particular, our results show that if the memory further increases until $n-m=o(n)$, the regret will decrease, while the dependency on $T$ will not be affected.

\subsection{Main results}\label{sec:main_results} 

We present the formal statement of our results below. Our first contribution is achieving the following regret upper bound through a new $P$-pass streaming MAB algorithm.
\begin{theorem}\label{thm:ub-general-m}
    Given a stream with $n$ arms, assuming $T\geq (n+1)^2$, for arbitrary pass number $1\leq P\leq \log\log T - \log\tp{12\log\frac{n}{n-m}}$ and memory size $2\leq m< n$, there exists a $P$-pass algorithm using a memory of $m$ arms with regret 
    \[
        R(T)\leq O\tp{\frac{(n-m)^{1+\frac{2^{P}-2}{2^{P+1}-1}}}{n^{\frac{2^{P+1}-2}{2^{P+1}-1}}} \cdot T^{\frac{2^P}{2^{P+1}-1}} \cdot \tp{\-{ilog}^{(m-1)}(n)}^{\frac{2^P-1}{2^{P+1}-1}}}.\footnote{The notation ilog, meaning \emph{iterative logarithm}, is defined in \Cref{sec:prelim}.}
    \]
\end{theorem}

In the statement of \Cref{thm:ub-general-m}, the requirement for an upper bound on $P$, i.e., $P\leq \log\log T - \log\tp{12\log\frac{n}{n-m}}$, is not important. In fact, if $P$ is beyond this upper bound, our algorithm already achieve the optimal regret of $\tilde O\tp{\sqrt{nT}}$ and the further passes are useless. 

\bigskip
We then complement our algorithm with a tight lower bound. 
\begin{theorem}\label{thm:lb-regret}
    Assume the parameters $P,n,m$ satisfying $1\leq P\leq \log {\log T} - \log \tp{14\log 8(n-m)}$, $T\geq n^2$ and $n> m\geq 2$. For any $P$-pass algorithm $\+A$, there exists a stochastic instance such that the expected regret of $\+A$ is 
    \[
        \Omega\tp{\frac{(n-m)^{1+{\frac{2^P-2}{2^{P+1}-1}}}}{n^{\frac{2^{P+1}-2}{2^{P+1}-1}} \cdot \tp{\log(64nP)}^{\frac{2^P-1}{2^{P+1}-1}}} \cdot T^{\frac{2^P}{2^{P+1}-1}} }
    \]
    when $T$ is sufficiently large.
\end{theorem}
Before this work, the best lower bound result was $\Omega\tp{4^{-P} T^{\frac{2^P}{2^{P+1}-1}}}$ for $m=o\tp{\frac{n}{P^2}}$ (\cite{AKP22}). When $m=o(n)$, from \Cref{thm:lb-regret}, our lower bound is $\tilde\Omega\tp{n^{1-\frac{2^P}{2^{P+1}-1}}T^{\frac{2^P}{2^{P+1}-1}}}$, which brings a considerable improvement compared to the previous one. Our results also fill the gap in characterizing the behavior of regret when the memory size is larger than $\Omega\tp{\frac{n}{P^2}}$. This provides a more precise understanding of the trade-off between memory size and regret. The results in \Cref{thm:lb-regret} also cover the previous single-pass results in~\cite{Wang23} and~\cite{MPK21}. When $m=\Omega(n)$, we provide an improvement of $\tilde\Theta\tp{(n-m)n^{\frac{5}{3}}}$ compared to the lower bound $\Omega\tp{\frac{n^{\frac{1}{3}}T^{\frac{2}{3}}}{m^\frac{7}{3}}}$ in~\cite{MPK21}.

We provide a comparison between our results and previous best results in \Cref{tab:comp}. Our results are tight up to a logarithmic factor in $n$ and $P$ for both single-pass and multi-pass settings. Furthermore, our results generalize previous bound by incorporating the effect of the memory $m$, which is crucial when $m$ is close $n$.
\begin{table}[htbp]
	\centering
	\caption{A Comparison with Previous Best Results}
	\label{tab:comp}
\begin{tabular}{p{1.5cm}<{\centering}m{2.2cm}<{\centering}m{6.3cm}<{\centering}m{3cm}<{\centering}}
	\toprule
    ~& Setting & Upper and lower Bound &  Memory \\
	\midrule
  \multirow{3}{=}{\cite{MPK21}} & \multirow{3}{*}{single-pass} & \multirow{3}{*}{$\Omega\tp{n^{\frac{1}{3}}T^{\frac{2}{3}}m^{-\frac{7}{3}}}$} & \multirow{3}{*}{$2\leq m< n$} \\
  ~ & ~ & ~ & ~ \\
  ~ & ~ & ~ & ~ \\
    \hline
    \multirow{3}{=}{\cite{Wang23}} & \multirow{3}*{single-pass} & $O\tp{n^{\frac{1}{3}} T^{\frac{2}{3}}}$ & $m=\Theta(\log^* n)$ \\
    \cline{3-4}
    ~ & ~ & $O\tp{n^{\frac{1}{3}} T^{\frac{2}{3}} \log n}$ & $m=\Theta(1)$ \\
    \cline{3-4}
    ~ & ~ & $\Omega\tp{n^{\frac{1}{3}} T^{\frac{2}{3}}}$ & $m\leq \frac{n}{20}$ \\
    \hline
    \multirow{2}{1.5cm}{\cite{AKP22}} & \multirow{2}*{multi-pass} & $ O\tp{ T^{\frac{2^P}{2^{P+1}-1}} \sqrt{nP\log T}}$ & $m=\Theta(1)$ \\
    \cline{3-4}
    ~ & ~ & $\Omega\tp{4^{-P} T^{\frac{2^P}{2^{P+1}-1}}}$ & $m\leq \frac{n}{8P(P+1)\log_2 e}$ \\
    \hline
    \multirow{2}*{this work} & \multirow{2}*{multi-pass} & $ \tilde \Theta\tp{\frac{(n-m)^{1+\frac{2^{P}-2}{2^{P+1}-1}} }{n^{\frac{2^{P+1}-2}{2^{P+1}-1}}}\cdot T^{\frac{2^P}{2^{P+1}-1}}}$ & $2\leq m< n$ \\
    \bottomrule
\end{tabular}
\end{table}

\subsection{Overview of our algorithms and techniques}\label{sec:intro-proof}


Our proofs, both upper bounds and lower bounds, reveal some interesting interplay between the memory and the regret which are not well understood before. 

Our first observation is that an optimal algorithm should behave differently for large $m$ ($m\ge \frac{8n}{9}$, say) and small $m$. To see this, notice that as did in~\cite{Wang23}, when $m$ is small, the low regret algorithm is almost equivalent to applying a \emph{best arm identification} (BAI) algorithm in the stream. The task of BAI, as suggested by~\cite{AW20} and~\cite{JHTX21}, can be done optimally with constant memory. Therefore, increasing the size of memory does not help in reducing its complexity. However, minimizing regret is not equivalent to best arm identification in the sense that identifying the best arm is in general harder than playing with low regret. We show that, in the streaming setting the latter is captured by the complexity of the \emph{best arm retention} (BAR) problem, namely to retain a good arm in memory (without necessarily identifying it) at the end of stream. Our results show that the complexity of the BAR problem exhibits a sharp transition phenomenon: when the memory is relatively small, increasing memory size has almost no effect on reducing the difficulty of BAR (and it is in fact equivalent to BAI); when the memory is close to the number of arms, the difficulty of BAR decreases significantly with the increase of the memory, and thus making the streaming MAB problem easier. This explains the $n-m$ terms in our regret bounds. 

Our algorithm for small $m$ also outperforms previous ones in multi-pass setting~\cite{AKP22}. The algorithm in ~\cite{AKP22} achieves $O\tp{T^{\frac{2^P}{2^{P+1}-1}}\sqrt{nP\log T}}$ regret using $O(1)$ memory via guaranteeing that the best arm is retained in memory with high probability during each pass. Comparing to their algorithm, we employ the explore-then-commit framework\footnote{A strategy in the \emph{explore-then-commit framework} first tries to retain the best arm in the memory during the initial rounds (referred to as the exploration phase) and then play arms in the memory using an optimal MAB algorithm in the remaining rounds (referred to as the exploitation phase).} and introduce a new measure for tracking the algorithm's progress over passes. Specifically, we notice that bounding the \emph{expected mean reward gap between the optimal arm and the best arm retained by the algorithm} in each pass is sufficient to guarantee a low regret in expectation. This saves unnecessary regret costs to attain high probability results, and is the main ingredient for our optimal algorithm, in both small and large memory case. Using this measure, we can shave off the $\log T$ factor in~\cite{AKP22}, and give the correct dependency on $n$, $m$ and $P$.

On the other hand, our proof of the lower bound essentially reveals that any algorithm with optimal regret must solve the BAR task well at each pass, which aligns with our algorithm. To be specific, we establish a lower bound on the number of rounds for exploration at each pass on some hard instances, which matches the behavior of our algorithm. Such a lower bound reflects the eternal exploration-exploitation trade-off in online learning algorithms. Let us sketch the proof idea.
 When there is only one pass, if an algorithm spends too much time exploring each arm, it may incur significant regret if the last arm in the stream happens to be the optimal one. However, if the earlier arms are sampled too few times, via a likelihood argument, we can show that the algorithm may fail the BAR task on some hard instance and may potentially discard the optimal arm among them with significant probability. This results in large regret during the exploitation phase after the stream ends. The advantage of the multi-pass setting appears to avoid this issue. For example, if one more pass is allowed, we can sample fewer times at the initial pass to quickly gain a rough outlook of all the arms. However, we can design a harder instance such that (1) the knowledge obtained from the first pass about the instance is negligible and (2) the second pass may fail the BAR task on this new instance and incur large regret if the number of rounds for explorations at this pass is not large enough. We eventually generalize the above argument for algorithms with arbitrary $P$ passes and establish the desired lower bound. 


\paragraph{Comparison with lower bounds in~\cite{AKP22}}

Our approach significantly differs from, and has advantage over that in~\cite{AKP22}. The work~\cite{AKP22} proved a regret lower bound by designing a distribution over hard instances, which incur large regret in expectation for every \emph{deterministic} algorithm and applying Yao's principle (\cite{Yao77}). In their construction, 
the $n$ arms are divided into $P+1$ disjoint subsets $\set{\+K_j}_{j\in[P+1]}$, with each subset containing a randomly selected arm $\+I_j\in \+K_j$ that has biased mean reward with a certain probability. At the end of the first pass, they considered the event of all these selected arms being discarded with little valid information gained. If this bad event occurs and the selected arm in $\+K_{P+1}$ happens to be unbiased, this problem can then be reduced to a problem with $\frac{P}{P+1}\cdot n$ arms (with those arms in $\+K_p$ being ignored) in the remaining $P-1$ passes. They showed that if the regret is assumed to be small, the probability of these events is at least $\frac{1}{4}$. This ultimately leads to a $\frac{1}{4^P}$ term in their lower bound by using the inductive method. 

Although the lower bound in~\cite{AKP22} is tight in $T$, it is suboptimal in both $P$ and $n$ for different reasons. Our proof remedy both. At a high level, the suboptimality in $P$ mainly arises from the fact the $n$ arms are divided into $P+1$ parts, where only the first $P-p+2$ parts are effectively utilized in the analysis of the $p$-th pass. Instead of relying on Yao's principle, we directly prove lower bounds for randomized algorithms via a likelihood argument by providing hard instances for each algorithm. 
Specifically, we construct a family of hard instances and show that for any \emph{randomized} algorithm, there exists at least one difficult instance within this family. This intuition is further explained in \Cref{sec:lb-informal}. 

The suboptimality in $n$ comes from the fact the analysis in ~\cite{AKP22} did not consider the dependency of $n$ (as well as $m$) in the regret caused by an algorithm. We incorporate this dependency by capturing the hardness of best arm retention problem. A similar idea also appeared in a recent work (\cite{AW24}) under a different setting. Furthermore, our analysis also breaks the limitation of $m=O\tp{\frac{n}{P(P+1)}}$ in \cite{AKP22} and provides a full characterization of the memory-regret trade-off.


\subsection{Related work}\label{sec:related_work}
The MAB problem was first introduced in~\cite{Rob52}. The regret bound for MAB has been proven to be $\Theta(\sqrt{nT})$ in~\cite{AB09} for both stochastic and adversarial cases. Another classic online decision problem is known as learning with expert advice. Different with MAB, in each round, the player can observe the rewards of all arms, rather than just the chosen one. The regret bound for learning with expert advice is $\Theta(\sqrt{T\log n})$ (\cite{FS97}).

The work of~\cite{LSPY18} first considered the MAB problem under streaming model. They derived an instance-sensitive upper bound using $O(\log T)$ passes and $O(1)$ memory. The work of~\cite{CK20} gave a generalized upper bound of $O\tp{\frac{n^{\frac{3}{2}}}{m}\sqrt{T\log \frac{T}{nm}}}$ for arbitrary memory size $2\leq m<n$ in $O(\log T)$ passes. A subsequent work~\cite{Rat21} investigated the case of $O(\log\log T)$ passes. The works of~\cite{MPK21} and~\cite{Wang23} studied the single-pass scenario and~\cite{Wang23} obtained tight regret bounds of $\Theta\tp{n^{\frac{1}{3}}T^{\frac{2}{3}}}$ when $\log^* n\leq m \leq \frac{n}{20}$1`'. \cite{AKP22} provided the first general bound with regard to the pass number and obtained an upper bound of $O\tp{T^{\frac{2^P}{2^{P+1}-1}} \sqrt{nP\log T}}$ using $O(1)$ memory. They also proved a lower bound of $\Omega\tp{4^{-P} T^{\frac{2^P}{2^{P+1}-1}}}$ for $m=o\tp{\frac{n}{P^2}}$.

The pure exploration version of the MAB problem under streaming model has also gained much attention in recent years. Numerous studies have focused on the impact of reducing memory size on the sample complexity of BAI problem (\cite{AW20,FOP20,MPK21,JHTX21,AW22}). This problem was formally introduced in~\cite{AW20}. They proposed an algorithm that, given the gap $\Delta$ between the best and second-best arms, uses optimal $O\tp{\frac{n}{\Delta^2}\log \frac{1}{\delta}}$ samples to find the best arm with probability of at least $1-\delta$ by storing only two arms.  The $(\eps,\delta)$-PAC algorithms with $O(\log^* n)$ arm memory were also proposed in~\cite{AW20} and~\cite{MPK21} achieving optimal sample complexity of $O\tp{\frac{n}{\eps^2}\log \frac{1}{\delta}}$ (the algorithm can output an $\eps$-optimal arm with probability at least $1-\delta$). Later, an $(\eps,\delta)$-PAC algorithm using constant memory was proposed in~\cite{JHTX21}.


In this work, the memory constraint limits the number of arms that can be stored. There are also some studies that do not restrict the number of arms and focus on the real memory required for MAB problem (\cite{XZ21}) or learning with expert advice (\cite{SWXZ22, PZ23,PR23,WZZ23}). These works typically focus on adversarial settings. These models are not directly comparable with the model of this work. Both of them have their own research significance and application value.

\subsection{Organization of the paper}
We formally introduce the multi-armed bandit in streaming model and related preliminaries in \Cref{sec:prelim}. 

Then we first present and analyze a streaming MAB algorithm for the special case $m=n-1$ in \Cref{sec:ub-l-simple} to showcase our key idea in the large memory algorithm. The algorithms and their analyses for general large memory ($m\geq \frac{8n}{9}$) and small memory ($m< \frac{8n}{9}$) cases are given in \Cref{sec:ub-l,sec:ub-s} respectively. 

In the remaining part of this paper, we provide the analysis for the lower bound. Before delving into the full proof, we offer an informal elaboration on the exploration-exploitation trade-off in streaming MAB games using $m=n-1$ as an example in \Cref{sec:lb-informal} which showcases our main idea. After that, a rigorous proof for our lower bound will be provided in \Cref{sec:lb}. 

\section{Preliminaries}\label{sec:prelim}

Let us first fix some notations. Let $\bb N$ be the set of all non-negative integers and $\bb R$ be the set of all real numbers. Additionally, we use $\mathbb{R}_{\geq 0}$ to represent the set of all non-negative real numbers. For any $n\in \bb N$, use $[n]$ to denote the set $\set{1,2,\dots,n}$. In this paper, unless otherwise specified, our logarithms are defined as natural logarithms with base $e$. For any $a\geq 1$, let $\-{ilog}^{(0)}(a)=a$ and $\-{ilog}^{(k)}(a)$ be the \emph{iterated logarithm} of order $k$ for integer $k\ge 1$, that is, $\-{ilog}^{(k)}(a)=\max\set{\log\tp{\-{ilog}^{(k-1)}(a)}, 1}$.

\subsection{Multi-armed bandit in streaming model}\label{sec:model}
\subsubsection{Multi-armed bandit}
In the problem of multi-armed bandit (MAB), the player is given a set of $n$ arms, denoted as $[n]$. A $T$-round decision game will be conducted as follows: in each round $t\in[T]$, 
\begin{itemize}
    \item the player picks an arm $A_t\in[n]$ to pull based on the information observed in so far;
    \item the environment choose a reward vector $r_t\in [0,1]^n$;
    \item the player observes and gains reward of the chosen arm $r_t(A_t)$.
\end{itemize}

The reward can be given \emph{stochastically} or \emph{adversarially}. In the stochastic environment considered in this work, there is a fixed distribution $\+D_i$ of mean $\mu_i$ for each $\a_i$. In each round, $r_t(i)$ is independently  drawn from this distribution. 

The best arm, $\a^*$, refers to the arm with the largest mean reward. The player's objective is to minimize the difference between the cumulative reward of the best arm and the player's own cumulative reward. That is, the player aims to design an algorithm $\+A$ to minimize the expected regret
\[
    R(T,\+A) = \E{\sum_{t=1}^T \mu_{\a^*} - \mu_{A_t}}.
\]
We also write $R(T,\+A)$ as $R(T)$ when $\+A$ is clear from the context.

There have been many algorithms achieving a regret of $O\tp{\sqrt{nT}}$ for the MAB problem, such as \emph{online stochastic mirror descent} (OSMD) and \emph{follow the regularized leader} (FTRL). The following result is given in~\cite{LG21}. 
\begin{proposition}[\cite{LG21}, Theorem 11]\label{prop:ub-l-OSMD}
    The OSMD algorithm with specific parameters on an MAB game satisfies that $R(T)$ is bounded by $\sqrt{2\abs{\+S}T}$ where $\+S$ is the arm set and $T$ is the total rounds. 
\end{proposition}
In this work, the algorithm for MAB is treated as a black box. The details of the algorithm pertaining to \Cref{prop:ub-l-OSMD} is in \Cref{sec:OSMD-detail}.

\subsubsection{Streaming MAB}\label{sec:sMABprelim}
Then we formally define the stochastic MAB problem in the $P$-pass streaming model for any integer $P\ge 1$. 

\paragraph{The mechanism of streaming MAB} In the streaming MAB problem, the $n$ arms arrive in a stream, and available arms are stored in the memory $\+M$ which has a maximum capacity of $m$ arms for some $m\le n$. The algorithm can only pull the arms stored in memory and store the statistics of those arms in memory. For each arm in the memory, one word\footnote{A word is of size $\Theta(\log nT)$ bits} is used to store its statistics such as the identity and \emph{empirical mean} of the arm.

In the round $t\in [T]$, the player selects an arm $A_t$, observes information and then gains the reward of the chosen arm. In each round, the player either choose an arm from memory directly, or read new arms from stream into memory and then choose one. When the memory is full, the player must discard some arms to make space before incorporating new ones. In other words, in each round $t\in [T]$, the algorithm acts in two stages, which handles arm storage (drop arms in memory, incorporating new arms in stream) and arm sampling (pull an arm in memory, observe and gain rewards) respectively. We emphasize that in one round, the player can drop and read any number of arms (including zero), as long as they are available in the stream, while exactly one pull is allowed.

Once an arm is dropped, all its information, including its identity, is forgotten. A discarded arm will not reappear in the single-pass setting. In this work, we consider a general setting with $P$ passes ($P\geq 1$). In each pass, the $n$ arms will pass through the stream one time. Therefore, we may expect a discarded arm to appear again in the next pass. The order of arms in different passes can vary. We say a pass ends once the last arm in the pass is read into the memory. This means that the explorations after this point are treated as samples in the next pass.

\paragraph{Streaming stochastic MAB} Our work only focuses on the stochastic arms. The reward of $\a_i$, $r_t(i)$, is drawn from a hidden distribution $\+D_i$ with mean $\mu_i$ independently in each round. Similar to the MAB problem, let $\a^*$ be the arm with the largest reward mean. The player's objective is to design an algorithm $\+A$ to minimize the regret, which is defined as $R(T,\+A) = \E{\sum_{t=1}^T \mu_{\a^*} - \mu_{A_t}}$. 
An algorithm is allowed to have its own randomness. We use $\Pr[H]{\cdot}$ or $\E[H]{\cdot}$ to denote the probability or expectation when the input instance is $H$, where the randomness comes from both the input instance and the algorithm itself.

\paragraph{Exploration and exploitation phase} We introduce terminologies to describe the behavior of an algorithm for streaming stochastic MAB on a specific instance. We divide the $T$ rounds of the game into two phases, the \emph{exploration phase} and \emph{exploitation phase}. The rounds \emph{before} the arrival of the last arm in the last pass constitute the phase of exploration, and the remaining rounds form the phase of exploitation.

We further classify the rounds in the exploration phase into $P$ groups based on which pass they belong to. The rounds of explorations in pass $p$ begins from the arrival of the last arm in pass $p-1$ (or from the start of the game if $p=1$) and ends once the last arm of pass $p$ is read into memory. 
We will use a random variable $L_p$ to denote the number of rounds in pass $p$ for any $p\in [P]$. As a result, without lose of generality, any algorithm for streaming MAB on an instance behaves in the following way:
\begin{itemize}
    \item For each $p=1,2,\dots,P$, explore $L_p$ rounds.
    \item For the remaining $T-\sum_{p\in [P]} L_p$ rounds, play arms in the memory (without interacting with the stream since there is no arm left).
\end{itemize}
The process is illustrated in Figure~\ref{fig:algorithm} in which each interval represents a round.
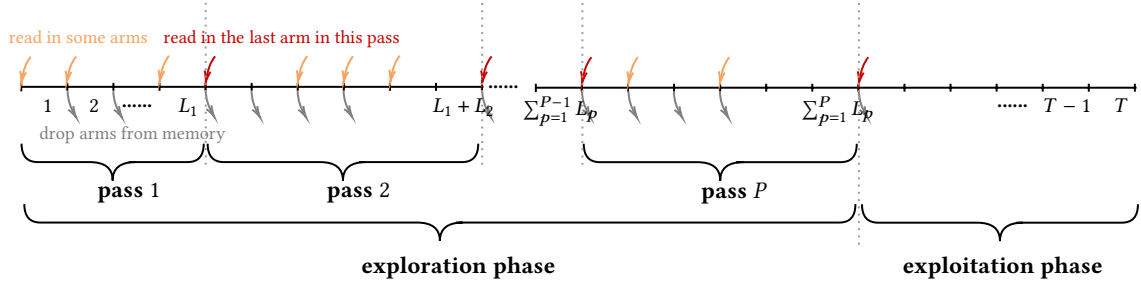
\begin{figure}\label{fig:algorithm}
    \centering
\tikzset{every picture/.style={line width=0.75pt}} 

\begin{tikzpicture}[x=0.83pt,y=0.83pt,yscale=-1.2,xscale=1.05]

\draw  (24.67,139.67) -- (225.25,139.67) (44.67,138.17) -- (44.67,141.17)(64.67,138.17) -- (64.67,141.17)(84.67,138.17) -- (84.67,141.17)(104.67,138.17) -- (104.67,141.17)(124.67,138.17) -- (124.67,141.17)(144.67,138.17) -- (144.67,141.17)(164.67,138.17) -- (164.67,141.17)(184.67,138.17) -- (184.67,141.17)(204.67,138.17) -- (204.67,141.17)(224.67,138.17) -- (224.67,141.17) ;
\draw    (248,140) -- (508.74,140) (268,138.5) -- (268,141.5)(288,138.5) -- (288,141.5)(308,138.5) -- (308,141.5)(328,138.5) -- (328,141.5)(348,138.5) -- (348,141.5)(368,138.5) -- (368,141.5)(388,138.5) -- (388,141.5)(408,138.5) -- (408,141.5)(428,138.5) -- (428,141.5)(448,138.5) -- (448,141.5)(468,138.5) -- (468,141.5)(488,138.5) -- (488,141.5)(508,138.5) -- (508,141.5) ;
\draw    (248,138.5) -- (248,141.67) ;
\draw    (24.67,138.2) -- (24.67,141.37) ;
\draw [color={rgb, 255:red, 244; green, 164; blue, 96 }  ,draw opacity=1 ]   (28.75,128.27) .. controls (26.52,130.8) and (25.69,133.28) .. (25.06,136.25) ;
\draw [shift={(24.67,138.2)}, rotate = 280.8] [color={rgb, 255:red, 244; green, 164; blue, 96 }  ,draw opacity=1 ][line width=0.75]    (4.37,-1.32) .. controls (2.78,-0.56) and (1.32,-0.12) .. (0,0) .. controls (1.32,0.12) and (2.78,0.56) .. (4.37,1.32)   ;
\draw [color={rgb, 255:red, 201; green, 0; blue, 0}  ,draw opacity=1 ]   (108.8,128.07) .. controls (106.57,130.59) and (105.75,133.08) .. (105.11,136.04) ;
\draw [shift={(104.72,137.99)}, rotate = 280.8] [color={rgb, 255:red, 201; green, 0; blue, 0}  ,draw opacity=1 ][line width=0.75]    (4.37,-1.32) .. controls (2.78,-0.56) and (1.32,-0.12) .. (0,0) .. controls (1.32,0.12) and (2.78,0.56) .. (4.37,1.32)   ;
\draw   (24.8,161.27) .. controls (24.8,165.94) and (27.13,168.27) .. (31.8,168.27) -- (54.33,168.27) .. controls (61,168.27) and (64.33,170.6) .. (64.33,175.27) .. controls (64.33,170.6) and (67.66,168.27) .. (74.33,168.27)(71.33,168.27) -- (96.86,168.27) .. controls (101.53,168.27) and (103.86,165.94) .. (103.86,161.27) ;
\draw [color={rgb, 255:red, 244; green, 164; blue, 96 }  ,draw opacity=1 ]   (48.45,128.17) .. controls (46.22,130.69) and (45.4,133.18) .. (44.76,136.14) ;
\draw [shift={(44.37,138.09)}, rotate = 280.8] [color={rgb, 255:red, 244; green, 164; blue, 96 }  ,draw opacity=1 ][line width=0.75]    (4.37,-1.32) .. controls (2.78,-0.56) and (1.32,-0.12) .. (0,0) .. controls (1.32,0.12) and (2.78,0.56) .. (4.37,1.32)   ;
\draw [color={rgb, 255:red, 128; green, 128; blue, 128 }  ,draw opacity=1 ]   (44.79,141.34) .. controls (44.91,144.82) and (45.97,147.16) .. (47.48,149.66) ;
\draw [shift={(48.53,151.35)}, rotate = 237.5] [color={rgb, 255:red, 128; green, 128; blue, 128 }  ,draw opacity=1 ][line width=0.75]    (4.37,-1.32) .. controls (2.78,-0.56) and (1.32,-0.12) .. (0,0) .. controls (1.32,0.12) and (2.78,0.56) .. (4.37,1.32)   ;
\draw [color={rgb, 255:red, 128; green, 128; blue, 128 }  ,draw opacity=1 ]   (104.73,141.11) .. controls (104.85,144.59) and (105.91,146.93) .. (107.42,149.43) ;
\draw [shift={(108.47,151.12)}, rotate = 237.5] [color={rgb, 255:red, 128; green, 128; blue, 128 }  ,draw opacity=1 ][line width=0.75]    (4.37,-1.32) .. controls (2.78,-0.56) and (1.32,-0.12) .. (0,0) .. controls (1.32,0.12) and (2.78,0.56) .. (4.37,1.32)   ;
\draw [color={rgb, 255:red, 128; green, 128; blue, 128 }  ,draw opacity=1 ]   (64.62,141.34) .. controls (64.75,144.82) and (65.81,147.16) .. (67.31,149.66) ;
\draw [shift={(68.36,151.35)}, rotate = 237.5] [color={rgb, 255:red, 128; green, 128; blue, 128 }  ,draw opacity=1 ][line width=0.75]    (4.37,-1.32) .. controls (2.78,-0.56) and (1.32,-0.12) .. (0,0) .. controls (1.32,0.12) and (2.78,0.56) .. (4.37,1.32)   ;
\draw [color={rgb, 255:red, 244; green, 164; blue, 96 }  ,draw opacity=1 ]   (88.79,127.83) .. controls (86.56,130.36) and (85.73,132.84) .. (85.1,135.8) ;
\draw [shift={(84.71,137.76)}, rotate = 280.8] [color={rgb, 255:red, 244; green, 164; blue, 96 }  ,draw opacity=1 ][line width=0.75]    (4.37,-1.32) .. controls (2.78,-0.56) and (1.32,-0.12) .. (0,0) .. controls (1.32,0.12) and (2.78,0.56) .. (4.37,1.32)   ;
\draw   (105.57,161.27) .. controls (105.57,165.94) and (107.9,168.27) .. (112.57,168.27) -- (154.29,168.27) .. controls (160.96,168.27) and (164.29,170.6) .. (164.29,175.27) .. controls (164.29,170.6) and (167.62,168.27) .. (174.29,168.27)(171.29,168.27) -- (216,168.27) .. controls (220.67,168.27) and (223,165.94) .. (223,161.27) ;
\draw [color={rgb, 255:red, 201; green, 0; blue, 0}  ,draw opacity=1 ]   (228.8,128.07) .. controls (226.57,130.59) and (225.75,133.08) .. (225.11,136.04) ;
\draw [shift={(224.72,137.99)}, rotate = 280.8] [color={rgb, 255:red, 201; green, 0; blue, 0}  ,draw opacity=1 ][line width=0.75]    (4.37,-1.32) .. controls (2.78,-0.56) and (1.32,-0.12) .. (0,0) .. controls (1.32,0.12) and (2.78,0.56) .. (4.37,1.32)   ;
\draw [color={rgb, 255:red, 128; green, 128; blue, 128 }  ,draw opacity=1 ]   (224.73,141.11) .. controls (224.85,144.59) and (225.91,146.93) .. (227.42,149.43) ;
\draw [shift={(228.47,151.12)}, rotate = 237.5] [color={rgb, 255:red, 128; green, 128; blue, 128 }  ,draw opacity=1 ][line width=0.75]    (4.37,-1.32) .. controls (2.78,-0.56) and (1.32,-0.12) .. (0,0) .. controls (1.32,0.12) and (2.78,0.56) .. (4.37,1.32)   ;
\draw [color={rgb, 255:red, 244; green, 164; blue, 96 }  ,draw opacity=1 ]   (168.64,128.07) .. controls (166.41,130.59) and (165.58,133.08) .. (164.95,136.04) ;
\draw [shift={(164.56,137.99)}, rotate = 280.8] [color={rgb, 255:red, 244; green, 164; blue, 96 }  ,draw opacity=1 ][line width=0.75]    (4.37,-1.32) .. controls (2.78,-0.56) and (1.32,-0.12) .. (0,0) .. controls (1.32,0.12) and (2.78,0.56) .. (4.37,1.32)   ;
\draw [color={rgb, 255:red, 128; green, 128; blue, 128 }  ,draw opacity=1 ]   (164.56,141.11) .. controls (164.68,144.59) and (165.74,146.93) .. (167.25,149.43) ;
\draw [shift={(168.3,151.12)}, rotate = 237.5] [color={rgb, 255:red, 128; green, 128; blue, 128 }  ,draw opacity=1 ][line width=0.75]    (4.37,-1.32) .. controls (2.78,-0.56) and (1.32,-0.12) .. (0,0) .. controls (1.32,0.12) and (2.78,0.56) .. (4.37,1.32)   ;
\draw [color={rgb, 255:red, 128; green, 128; blue, 128 }  ,draw opacity=1 ]   (124.62,141.34) .. controls (124.75,144.82) and (125.81,147.16) .. (127.31,149.66) ;
\draw [shift={(128.36,151.35)}, rotate = 237.5] [color={rgb, 255:red, 128; green, 128; blue, 128 }  ,draw opacity=1 ][line width=0.75]    (4.37,-1.32) .. controls (2.78,-0.56) and (1.32,-0.12) .. (0,0) .. controls (1.32,0.12) and (2.78,0.56) .. (4.37,1.32)   ;
\draw [color={rgb, 255:red, 244; green, 164; blue, 96 }  ,draw opacity=1 ]   (188.95,127.83) .. controls (186.72,130.36) and (185.9,132.84) .. (185.26,135.8) ;
\draw [shift={(184.87,137.76)}, rotate = 280.8] [color={rgb, 255:red, 244; green, 164; blue, 96 }  ,draw opacity=1 ][line width=0.75]    (4.37,-1.32) .. controls (2.78,-0.56) and (1.32,-0.12) .. (0,0) .. controls (1.32,0.12) and (2.78,0.56) .. (4.37,1.32)   ;
\draw [color={rgb, 255:red, 128; green, 128; blue, 128 }  ,draw opacity=1 ]   (144.62,141.34) .. controls (144.75,144.82) and (145.81,147.16) .. (147.31,149.66) ;
\draw [shift={(148.36,151.35)}, rotate = 237.5] [color={rgb, 255:red, 128; green, 128; blue, 128 }  ,draw opacity=1 ][line width=0.75]    (4.37,-1.32) .. controls (2.78,-0.56) and (1.32,-0.12) .. (0,0) .. controls (1.32,0.12) and (2.78,0.56) .. (4.37,1.32)   ;
\draw [color={rgb, 255:red, 244; green, 164; blue, 96 }  ,draw opacity=1 ]   (148.79,128) .. controls (146.56,130.52) and (145.73,133.01) .. (145.1,135.97) ;
\draw [shift={(144.71,137.92)}, rotate = 280.8] [color={rgb, 255:red, 244; green, 164; blue, 96 }  ,draw opacity=1 ][line width=0.75]    (4.37,-1.32) .. controls (2.78,-0.56) and (1.32,-0.12) .. (0,0) .. controls (1.32,0.12) and (2.78,0.56) .. (4.37,1.32)   ;
\draw   (268.8,161.3) .. controls (268.8,165.97) and (271.13,168.3) .. (275.8,168.3) -- (318.04,168.3) .. controls (324.71,168.3) and (328.04,170.63) .. (328.04,175.3) .. controls (328.04,170.63) and (331.37,168.3) .. (338.04,168.3)(335.04,168.3) -- (380.29,168.3) .. controls (384.96,168.3) and (387.29,165.97) .. (387.29,161.3) ;
\draw [color={rgb, 255:red, 201; green, 0; blue, 0 }  ,draw opacity=1 ]   (272.14,128.4) .. controls (269.91,130.92) and (269.08,133.41) .. (268.45,136.37) ;
\draw [shift={(268.06,138.32)}, rotate = 280.8] [color={rgb, 255:red, 201; green, 0; blue, 0}  ,draw opacity=1 ][line width=0.75]    (4.37,-1.32) .. controls (2.78,-0.56) and (1.32,-0.12) .. (0,0) .. controls (1.32,0.12) and (2.78,0.56) .. (4.37,1.32)   ;
\draw [color={rgb, 255:red, 128; green, 128; blue, 128 }  ,draw opacity=1 ]   (268.06,141.44) .. controls (268.18,144.92) and (269.24,147.26) .. (270.75,149.77) ;
\draw [shift={(271.8,151.45)}, rotate = 237.5] [color={rgb, 255:red, 128; green, 128; blue, 128 }  ,draw opacity=1 ][line width=0.75]    (4.37,-1.32) .. controls (2.78,-0.56) and (1.32,-0.12) .. (0,0) .. controls (1.32,0.12) and (2.78,0.56) .. (4.37,1.32)   ;
\draw [color={rgb, 255:red, 201; green, 0; blue, 0}  ,draw opacity=1 ]   (392.3,128.4) .. controls (390.07,130.92) and (389.25,133.41) .. (388.61,136.37) ;
\draw [shift={(388.22,138.32)}, rotate = 280.8] [color={rgb, 255:red, 201; green, 0; blue, 0}  ,draw opacity=1 ][line width=0.75]    (4.37,-1.32) .. controls (2.78,-0.56) and (1.32,-0.12) .. (0,0) .. controls (1.32,0.12) and (2.78,0.56) .. (4.37,1.32)   ;
\draw [color={rgb, 255:red, 128; green, 128; blue, 128 }  ,draw opacity=1 ]   (388.23,141.44) .. controls (388.35,144.92) and (389.41,147.26) .. (390.92,149.77) ;
\draw [shift={(391.97,151.45)}, rotate = 237.5] [color={rgb, 255:red, 128; green, 128; blue, 128 }  ,draw opacity=1 ][line width=0.75]    (4.37,-1.32) .. controls (2.78,-0.56) and (1.32,-0.12) .. (0,0) .. controls (1.32,0.12) and (2.78,0.56) .. (4.37,1.32)   ;
\draw [color={rgb, 255:red, 244; green, 164; blue, 96 }  ,draw opacity=1 ]   (292.3,128.4) .. controls (290.07,130.92) and (289.25,133.41) .. (288.61,136.37) ;
\draw [shift={(288.22,138.32)}, rotate = 280.8] [color={rgb, 255:red, 244; green, 164; blue, 96 }  ,draw opacity=1 ][line width=0.75]    (4.37,-1.32) .. controls (2.78,-0.56) and (1.32,-0.12) .. (0,0) .. controls (1.32,0.12) and (2.78,0.56) .. (4.37,1.32)   ;
\draw [color={rgb, 255:red, 128; green, 128; blue, 128 }  ,draw opacity=1 ]   (288.23,141.44) .. controls (288.35,144.92) and (289.41,147.26) .. (290.92,149.77) ;
\draw [shift={(291.97,151.45)}, rotate = 237.5] [color={rgb, 255:red, 128; green, 128; blue, 128 }  ,draw opacity=1 ][line width=0.75]    (4.37,-1.32) .. controls (2.78,-0.56) and (1.32,-0.12) .. (0,0) .. controls (1.32,0.12) and (2.78,0.56) .. (4.37,1.32)   ;
\draw [color={rgb, 255:red, 128; green, 128; blue, 128 }  ,draw opacity=1 ]   (307.89,141.44) .. controls (308.02,144.92) and (309.08,147.26) .. (310.58,149.77) ;
\draw [shift={(311.63,151.45)}, rotate = 237.5] [color={rgb, 255:red, 128; green, 128; blue, 128 }  ,draw opacity=1 ][line width=0.75]    (4.37,-1.32) .. controls (2.78,-0.56) and (1.32,-0.12) .. (0,0) .. controls (1.32,0.12) and (2.78,0.56) .. (4.37,1.32)   ;
\draw [color={rgb, 255:red, 244; green, 164; blue, 96 }  ,draw opacity=1 ]   (331.97,128.4) .. controls (329.74,130.92) and (328.91,133.41) .. (328.28,136.37) ;
\draw [shift={(327.89,138.32)}, rotate = 280.8] [color={rgb, 255:red, 244; green, 164; blue, 96 }  ,draw opacity=1 ][line width=0.75]    (4.37,-1.32) .. controls (2.78,-0.56) and (1.32,-0.12) .. (0,0) .. controls (1.32,0.12) and (2.78,0.56) .. (4.37,1.32)   ;
\draw [color={rgb, 255:red, 128; green, 128; blue, 128 }  ,draw opacity=1 ]   (327.89,141.44) .. controls (328.02,144.92) and (329.08,147.26) .. (330.58,149.77) ;
\draw [shift={(331.63,151.45)}, rotate = 237.5] [color={rgb, 255:red, 128; green, 128; blue, 128 }  ,draw opacity=1 ][line width=0.75]    (4.37,-1.32) .. controls (2.78,-0.56) and (1.32,-0.12) .. (0,0) .. controls (1.32,0.12) and (2.78,0.56) .. (4.37,1.32)   ;
\draw   (26,184.25) .. controls (26,188.92) and (28.33,191.25) .. (33,191.25) -- (196.36,191.25) .. controls (203.03,191.25) and (206.36,193.58) .. (206.36,198.25) .. controls (206.36,193.58) and (209.69,191.25) .. (216.36,191.25)(213.36,191.25) -- (379.71,191.25) .. controls (384.38,191.25) and (386.71,188.92) .. (386.71,184.25) ;
\draw   (389.57,184.3) .. controls (389.57,188.97) and (391.9,191.3) .. (396.57,191.3) -- (439.45,191.3) .. controls (446.12,191.3) and (449.45,193.63) .. (449.45,198.3) .. controls (449.45,193.63) and (452.78,191.3) .. (459.45,191.3)(456.45,191.3) -- (502.33,191.3) .. controls (507,191.3) and (509.33,188.97) .. (509.33,184.3) ;
\draw [color={rgb, 255:red, 155; green, 155; blue, 155 }  ,draw opacity=1,dotted ]   (104.72,107.75) -- (104.72,134.44) -- (104.72,148.44) -- (104.72,170.75) ;
\draw [color={rgb, 255:red, 155; green, 155; blue, 155 }  ,draw opacity=1,dotted ]   (224.72,107.8) -- (224.72,170.8) ;
\draw [color={rgb, 255:red, 155; green, 155; blue, 155 }  ,draw opacity=1,dotted ]   (268.22,107.8) -- (268.22,170.8) ;
\draw [color={rgb, 255:red, 155; green, 155; blue, 155 }  ,draw opacity=1,dotted ]   (388.22,107.8) -- (388.22,200.75) ;

\draw (32.7,143) [font=\bfseries] node [anchor=north west][inner sep=0.75pt]   [align=left] {{\footnotesize $1$}};
\draw (52.8,143) [font=\bfseries] node [anchor=north west][inner sep=0.75pt]   [align=left] {{\footnotesize$ 2$}};
\draw (67,145) [font=\bfseries] node [anchor=north west][inner sep=0.75pt]   [align=left] {{\footnotesize ......}};
\draw (497.9,143) [font=\bfseries] node [anchor=north west][inner sep=0.75pt]   [align=left] {{\footnotesize $ T$}};
\draw (467,143) [font=\bfseries] node [anchor=north west][inner sep=0.75pt]   [align=left] {{\footnotesize ${T-1}$}};
\draw (447,145) [font=\bfseries] node [anchor=north west][inner sep=0.75pt]   [align=left] {{\footnotesize ......}};
\draw (18,117) node [anchor=north west][inner sep=0.75pt]   [align=left] {{\scriptsize \textcolor[rgb]{0.96,0.64,0.38}{read in some arms}}};
\draw (226,138) [font=\bfseries] node [anchor=north west][inner sep=0.75pt]   [align=left] {{\footnotesize ......}};
\draw (56.4,175) [font=\bfseries] node [anchor=north west][inner sep=0.75pt]   [align=left] {{\small pass $1$}};
\draw (91.6,143) [font=\bfseries] node [anchor=north west][inner sep=0.75pt]   [align=left] {{\footnotesize ${L_1}$}};
\draw (85.06,117) node [anchor=north west][inner sep=0.75pt]   [align=left] {{\scriptsize \textcolor[RGB]{201,0,0}{read in the last arm in this pass }}};
\draw (156.07,175) [font=\bfseries] node [anchor=north west][inner sep=0.75pt]   [align=left] {{\small pass $2$}};
\draw (202,143) [font=\bfseries] node [anchor=north west][inner sep=0.75pt]   [align=left] {{\footnotesize ${L_1+L_2}$}};
\draw (318.54,175) [font=\bfseries] node [anchor=north west][inner sep=0.75pt]   [align=left] {{\small pass $P$}};
\draw (171,203) [font=\bfseries] node [anchor=north west][inner sep=0.75pt]   [align=left] {{\small exploration phase}};
\draw (241,142) [font=\bfseries] node [anchor=north west][inner sep=0.75pt]   [align=left] {{\footnotesize ${\sum_{p=1}^{P-1} L_p}$}};
\draw (362,142) [font=\bfseries] node [anchor=north west][inner sep=0.75pt]   [align=left] {{\footnotesize ${\sum_{p=1}^P L_p}$}};
\draw (406,203) [font=\bfseries] node [anchor=north west][inner sep=0.75pt]   [align=left] {{\small exploitation phase}};
\draw (31.31,153.5) node [anchor=north west][inner sep=0.75pt]   [align=left] {{\scriptsize \textcolor[rgb]{0.5,0.5,0.5}{drop arms from memory}}};
\end{tikzpicture}

\caption{Each cell in the diagram represents one round, with inward arrows $\orangearrow$, $\redarrow$ and outward arrows $\grayarrow$ denoting reading some arms from the stream and dropping some from the memory respectively. The symbol $\redarrow$ indicates the algorithm reading in the last arm of this pass, signifying the end of this pass. Any multi-pass streaming algorithm can be formalized as this figure described: it decomposes into an exploration phase and an exploitation phase, with each pass in the exploration phase consuming $L_p$ rounds for some (possibly random) $L_p$.}
\end{figure}

\paragraph{More remarks on the memory model} We consider the general memory size $2\leq m<n$. In a less restrictive memory model, where we are not limited to storing the information related to only the arms in memory, it is possible to achieve sublinear regret through explore-then-commit strategies with multiple passes using a single memory slot. However, in the standard setting where oblivion is compulsory, one memory slot must be reserved to pull the arriving arm and another memory slot must be used to store a good arm encountered so far. Otherwise, it is easy to see a regret lower bound linear in $T$ for both multi-pass and single-pass case when $m=1$. Hence, we only focus on the situations $m\geq 2$. 

It is worth noting that our \emph{lower bounds} actually apply to that less restrictive memory model. In other words, our lower bound holds even if the algorithm is allowed to store any other information, as long as it satisfies the requirement of storing at most $m$ arms. 
This indicates that, when $2\leq m<n$, the main challenge posed by memory constraints is the limitation on the number of arms that can be stored and can be pulled by the algorithm at each round, rather than the restriction on storing additional statistics.

A recent line of work studies the memory-regret trade-off for the learning with expert advice problem (\cite{SWXZ22,PZ23,PR23}). Their model differs from ours in three aspects: (1) unlike the stochastic reward in our model, their feedback is adversarial in nature; (2) the entire reward vector is revealed to the player in each round; (3) the player has the freedom to  pick any arm in each round, which corresponds to the case that $P=T$ in our model. An interesting problem is to combine these two models, considering the adversarial bandit feedback under the streaming model with limited number of passes. We leave this as an open problem for further investigation.

\subsubsection{Best arm identification and best arm retention} 
\paragraph{Best arm identification} There is a \emph{pure exploration} version of MAB, called \emph{best arm identification} (BAI). In the BAI problem, our objective is to find the arm with the maximum reward with high probability. In each round, the player can either choose an arm to observe its reward, or terminate the game and output an arm index.
For $\eps\in(0,1)$, an arm is $\eps$-optimal or $\eps$-best if its mean reward $\mu> \mu_{\a^*} - \eps$. For any fixed $\eps,\delta\in(0,1)$, we say an algorithm is $(\eps,\delta)$-probably approximately correct, or $(\eps,\delta)$-PAC for short, if it can identify an $\eps$-optimal arm with probability at least $1-\delta$. The sample complexity of the algorithm is defined as the number of rounds before its termination.

\paragraph{Best arm retention} We also define another pure exploration version of streaming MAB, named as \emph{best arm retention} (BAR). In BAR, the algorithm needs to output a set of arms of size $m<n$. Our objective is to retain the best arm in the output set with high probability. Unlike BAI, the key point of this problem is not to identify the best arm but only to ensure that the arms discarded are all suboptimal. In this work, we mainly care about the expected mean reward gap between the optimal arm in the output set and $\a^*$.

This problem was also mentioned in~\cite{AW22} as the \emph{trapping the best arm}. We will demonstrate the close relationship between BAR and streaming MAB, as well as algorithms for BAR. One can verify that the algorithm in \Cref{sec:ub-l-warmup} is actually a PAC algorithm for BAR and the sample complexity is tight (up to logarithm factors) with certain parameters.

\paragraph{Streaming BAI and streaming BAR} The BAI and BAR problem can also be considered in the streaming setting as well. Streaming BAI and BAR are pure exploration version of streaming MAB. The $n$ arms come in a stream and the algorithm can only store $m<n$ arms. The memory model of streaming BAI and BAR is the same with that described in \Cref{sec:sMABprelim}.



\subsection{Concentration inequalities}
We will use the following generalization of Hoeffding's inequality. The proof of this lemma is provided in \Cref{subsec: pf-lb-concentration}.
\begin{lemma}\label{lem:lb-concentration}
    Let $X_1,\dots,X_N$ be $N$ independent random variables defined on a common probability space and taking values in $[a,b]$. Assume $\E{X_t}=0$ for any $t\in[N]$. Then for any $s>0$,
    \[
        \Pr{\max_{1\leq t\leq N} \abs{\sum_{j=1}^t X_j} \geq s}\leq 2\exp{-\frac{2s^2}{N(b-a)^2}}.
    \]
\end{lemma}

\section{Upper Bound Analysis for the $m=n-1$ Case}\label{sec:ub-l-simple}

In this section, we present our algorithm and its analysis when $m=n-1$. This is a special case of our more general algorithm in \Cref{sec:ub-l}. However, this case already showcases our main idea, while both the algorithm and its analysis are much simpler than in the general case.

As mentioned in \Cref{sec:intro-proof}, in each pass, our algorithm is facing a BAR problem, namely to pick an arm to drop in this pass. Our key observation is that the following quantity turns out to be a correct measure for the quality of the BAR solution in our setting: \emph{the expected mean reward of the optimal arm in the memory at the end of each pass}.

Hence, we claim that a good algorithm should guarantee the following two properties in each pass $p\in [P]$:
\begin{itemize}
    \item For $p>1$, the regret caused by exploration in pass $p$ should be small compared to $\k_{p-1}$, where $\k_{p-1}$ is the best arm in memory at the end of the previous pass;
    \item $\E{\mu_{\a^*} - \mu_{\k_p}}$ is decreasing in $p$, where $\a^*$ is the optimal arm (among all $n$ arms) and $\k_p$ is the best arm in memory at the end of pass $p$.
\end{itemize}
In fact, the quantity $\E{\mu_{\a^*} - \mu_{\k_p}}$ implicitly serve as a potential function to measure the progress of our algorithm in the analysis. 

To fulfill the first requirement, we only need to identify $\k_{p-1}$ (or an arm comparable to $\k_{p-1}$) and discard those arms that perform significantly worse than $\k_{p-1}$ as early as possible during subsequent sampling. For the second requirement, we can reduce the problem to a version of the BAR problem for which we design an optimal algorithm.



To highlight our main idea, we further simplify the streaming model and strengthen the player by allowing the algorithm to remember the identity of the dropped arms. This capability is not essential, as we will circumvent it using a trick in the general setting. We also assume that $n=o(\sqrt{T})$ and $n, T$ are sufficiently large to simplify calculations and avoid the tedium of border cases. 

\subsection{The algorithm}\label{subsec:ub-l-simple-algo}
Let $\a^*$ denote the optimal arm (the one with the largest empirical mean) among all $n$ arms. For any arm with name $\!a$, let $\mu_{\!a}$ be its mean reward. Let $\+M_p$ be the set of arms in memory at the beginning of pass $p$ for any $p\in [P]$. Let $\+M_{P+1}$ denote the arms in memory in the exploitation phase. 

\paragraph{Overview of the algorithm} Our main task in each pass is to solve a BAR problem while reducing regret. At the beginning of the first pass, we first read in $n-1$ arms. Choose $2$ among them arbitrarily and run a $\textsc{FindBest}$ subroutine on the two arms to identify the better one, denoted as $\a'_1$. Then we drop the other arm to incorporate the last upcoming arm in stream.

For subsequent passes $p$, the memory is full at the beginning of the pass. We first perform a shorter $\textsc{FindBest}$ subroutine on all the arms in memory to find a good arm $\a'_{p,1}$. Then we choose two other arms together with $\a'_{p,1}$ to perform a longer $\textsc{FindBest}$ subroutine. Note that there is only one unvisited arm in stream in this pass. We then drop the worse arm in the second $\textsc{FindBest}$ subroutine and read in the remaining arm.

\paragraph{The description of the algorithm} Let $\lambda_p = \frac{2^{P-p+1}-1}{2^{P+1}-1}$ for any $p\in [P]$. Set the following parameters: 
\begin{itemize}
    \item the sample times in the first pass $L_1 = \ceil{2^{-2P}\cdot n^{-2\lambda_1}T^{1-\lambda_1}}$;
    \item for any $p\in[P]\setminus\set{1}$, 
          \begin{itemize}
            \item[$\circ$] the sample times of the second $\textsc{FindBest}$ in the $p$-th pass $L_{p,2}= \ceil{2^{-2P+2p-2}\cdot n^{-2\lambda_p}T^{1-\lambda_p}}$;
            \item[$\circ$] the sample times of the first $\textsc{FindBest}$ in the $p$-th pass $L_{p,1}= n^3\cdot L_{p-1,2}$ for $p>2$ and $L_{p,1}=n^3\cdot L_{1}$ for $p=2$.
          \end{itemize}  
\end{itemize}
The choice of $L_{p,1}$ is to ensure that, the first $\textsc{FindBest}$ is sufficient to figure out the $\k_{p-1}$ left from the previous pass. The second $\textsc{FindBest}$, the sample times of which constitutes the main portion of pass $p$, determines how well the BAR problem is solved. The larger $L_{p,2}$ is, the smaller mean reward gap between $\k_p$ and $\a^*$ we can get. In our choice, as $p$ increases, $L_{p,2}$ becomes larger as well. 
This indicates the strategy of our algorithm: explore prudently when we have little information and intensify the exploration as information accumulates (or equivalently, pass number increases).
For example, when $P=3$, we assign $L_{1}=\Theta\tp{n^{-\frac{14}{15}} T^{\frac{8}{15}}}$, $L_{2,2}=\Theta\tp{n^{-\frac{6}{15}}T^{\frac{12}{15}}}$ and $L_{3,2}=\Theta\tp{n^{-\frac{2}{15}} T^{\frac{14}{15}}}$, and expect the exploration in pass $3$ to retain better arms compared to pass $2$ and exploration in pass $2$ to retain better arms than pass $1$.


The algorithm is given in \Cref{algo:large-m-simple}.  In \Cref{algo:large-m-simple}, the $\textsc{FindBest}$ subroutine allows us to identify the optimal arm with small regret meanwhile. This property of $\textsc{FindBest}$ subroutine will be frequently used in upper bound analysis for both large and small memory cases. The $\textsc{MirrorDescent}$ subroutine used here is treated as a black box. A detailed description for it is provided in \Cref{sec:OSMD-detail}. In fact, $\textsc{MirrorDescent}$ can be substituted with any other $n$-arm and $T$-round bandit algorithms achieving a regret of $O\tp{\sqrt{nT}}$. The power of the subroutine is proved in Lemma~\ref{lem:ub-bandit-subroutine-simple}.

\begin{algorithm}[h]
    \caption{A subroutine using online stochastic mirror descent to find a good arm}
    \label{algo:bandit-subroutine}
    \Input{a set of arms $\+S$ and the number of rounds $L$}\\
    \Output{an arm in $\+S$}
\begin{algorithmic}[1]
    \Procedure{\textsc{FindBest}}{$\+S,L$}
    \State Run \Call{MirrorDescent}{$\+S,L$} 
    \ForAll{$\a_i\in\+S$}
        \State $T_i\gets$ the number of times $\a_i$ has been sampled in the \Call{MirrorDescent}{} routine
    \EndFor
    \State Sample an arm $\a'\in \+S$ with law $\Pr{\a'=\mbox{the }i\mbox{-th arm}}=\frac{T_i}{L}$
    \State \Return $\a'$
    \EndProcedure
\end{algorithmic}
\end{algorithm}

\begin{algorithm}[ht]
    \caption{Multi-pass algorithm for MAB with large memory ($m=n-1$)}
    \label{algo:large-m-simple}
    \Input{number of passes $P$, time horizon $T$, memory size $m=n-1$}
\begin{algorithmic}[1]
    \For{$p=1,2,\dots, P$}
        \If{$p=1$}
            \State Read in the first $n-1$ arms\;
            \State Choose $2$ arms from $\+M_1$ uniformly at random and let $\+S_1$ be the set of the two arms\; \label{line:ub-l-simple-sample1}
            \State $\a'_{1}$ = \Call{FindBest}{$\+S_1,L_{1}$}\;
            \State Drop $\+S_1\setminus \set{\a'_{1}}$ and read in the remaining one arm in stream\;\label{line:ub-l-simple-read1}
        \Else
            \State $\a'_{p,1}$ = \Call{FindBest}{$\+M_p,L_{p,1}$}\;
            \State Choose two arms from $\+M_p\setminus \set{\a'_{p,1}}$ uniformly at random and let $\+S_p$ be the set of these two arms plus $\a'_{p,1}$\; \label{line:ub-l-simple-sample}
            \State $\a'_{p,2}$ = \Call{FindBest}{$\+S_p,L_{p,2}$}\;
            \State Choose one arm from $\+S_p\setminus\set{\a'_{p,1},\a'_{p,2}}$ to drop\;\label{line:ub-l-simple-drop}
            \State Read in the remaining one arm in stream\;\label{line:ub-l-simple-read2}
        \EndIf
    \EndFor
    \State Run $\textsc{MirrorDescent}$ on the remaining $n-1$ arms until the game ends\;
\end{algorithmic}
\end{algorithm}

\begin{lemma}\label{lem:ub-bandit-subroutine-simple}
    Let $\a^*_{\+S}$ be the optimal arm in $\+S$. What ever the input arms are, \Cref{algo:bandit-subroutine} guarantees that $\E{\mu_{\a^*_{\+S}} - \mu_{\a'}}\leq \sqrt{\frac{2\abs{\+S}}{L}}$.
\end{lemma}
\begin{proof}
    We have
    \begin{align*}
        \E{\mu_{\a^*_{\+S}} - \mu_{\a'}} &= \E{\sum_{i\in \+S} \Pr{\a' = \a\ i}\cdot (\mu_{\a^*_{\+S}} - \mu_{i})} \\
        &\leq \frac{\E{\sum_{i\in \+S} T_i\cdot (\mu_{\a^*_{\+S}} - \mu_{i})}}{L}\leq \frac{\sqrt{2\abs{\+S}L}}{L}
    \end{align*}
    where the last inequality follows from Proposition~\ref{prop:ub-l-OSMD}.
\end{proof}

A point we should emphasize is that our algorithm executes the $\textsc{FindBest}$ subroutine twice in each pass (except for the first pass). The purpose of the first one is to identify the good arm left from the previous pass, $\a'_{p,1}$. By including $\a'_{p,1}$ in $S_p$, we can ensure that the cumulative regret in the second $\textsc{FindBest}$ is not too large. The objective of the second $\textsc{FindBest}$ is to discard some less promising arms to make room for new arrivals. 

In the first pass or the second $\textsc{FindBest}$ subroutine of the $p$-th pass, we only operate on $O(1)$ arms rather than all the arms in memory. This reduces the size of the $\textsc{FindBest}$ subroutine from $O(n)$ arms to $O(1)$ arms and does not affect accuracy. The reduction in scale allows for a reduction in regret compared to directly performing the subroutine on all arms in memory to preserve a good $\a'_{p,2}$. The benefit of doing so can be seen in the proof of Lemma~\ref{lem:ub-l-simple-kingb}.

Another point we need to mention is the difference between the first pass and the subsequent passes. Note that the first pass only includes one iteration of $\textsc{FindBest}$ subroutine. This is because before the start of the first pass, we do not have any prior knowledge or information. In contrast, before the start of the $p$-th pass (where $p>1$), we may already have a good arm stored in memory. Therefore, we need the first $\textsc{FindBest}$ to identify this arm. But this step does not provide any benefit for the first pass.

\subsection{The analysis}
Let $\eps_1=\sqrt{\frac{1}{L_1}}$ and let $\eps_p=\sqrt{\frac{1}{L_{p,2}}}$ for $p>1$. Let $\-{king}_p$ be the best arm in memory when pass $p$ ends, i.e., after Line~\ref{line:ub-l-simple-read1} or Line~\ref{line:ub-l-simple-read2} is executed. For an arm with name $\!a$, we let $\Delta_{\!a} = \mu_{\a^*} - \mu_{\!a}$. Then we have the following lemma, which guarantees that our algorithm solves the BAR problem well in each pass.

\begin{lemma}\label{lem:ub-l-simple-kingb}
$\E{\Delta_{\k_p}}\leq O\tp{\frac{\eps_p}{n}}$. 
\end{lemma}
\begin{proof}
    Recall that $\+M_p$ is the set of arms in memory at the beginning of pass $p$. If the optimal arm $\a^*$ is not in $\+M_p$, according to \Cref{algo:large-m-simple}, it must be in $\+M_{p+1}$ and therefore $\Delta_{\-{king}_p}=0$. Hence, we only need to consider the case when $\a^*$ is in memory when we drop arms from $S_p$.

    When $p=1$, we have
    \begin{align*}
        \E{\Delta_{\-{king}_1}}&= \Pr{\a^*\in \+S_1}\cdot \E{\mu_{\a^*}-\mu_{\a'_1} \mid \a^*\in \+S_1}\\
        &\leq \frac{2}{n-1}\cdot \E{\mu_{\a^*}-\mu_{\a'_1} \mid \a^*\in \+S_1} =O\tp{\frac{\eps_1}{n}},
    \end{align*}
    where the third inequality holds due to Lemma~\ref{lem:ub-bandit-subroutine-simple}.

    For $p>1$, $\a^*$ is dropped only if $\a^*\in \+S_p\setminus\set{\a'_{p,2},\a'_{p,1}}$. Therefore, 
    \begin{align}
        \E{\Delta_{\-{king}_p}}&\leq \Pr{\a^*\in \+S_p\setminus\set{\a'_{p,1}}}\cdot \E{\mu_{\a^*} - \mu_{\a'_{p,2}}\mid \a^*\in \+S_p\setminus\set{\a'_{p,1}}} \notag \\ 
        &\overset{(\spadesuit)}{\leq} \Pr{\a^*\in \+S_p\mid \a^*\neq \a'_{p,1}}\Pr{\a^*\neq\a'_{p,1}}  \cdot \sqrt{\frac{6}{L_{p,2}}}  \notag \\
        & \leq \frac{2}{n-2} \cdot\sqrt{\frac{6}{L_{p,2}}} = O\tp{\frac{\eps_p}{n}}, \notag
    \end{align}
    where $(\spadesuit)$ follows from Lemma~\ref{lem:ub-bandit-subroutine-simple}.
\end{proof}

We remark that it is crucial to include additional $2$ arms in the Line~\ref{line:ub-l-simple-sample} of \Cref{algo:large-m-simple} instead of only $1$ arm. This allows us to keep $\a'_{p,1}$ in the memory in Line~\ref{line:ub-l-simple-drop}. By doing so, the premise condition for $\a^*$ being discarded is that $\a^*\in  \+S_p\setminus\set{\a'_{p,1}}$. The probability of this event can be bounded by $O\tp{\frac{1}{n}}$, and the saving here is also key to our tight regret bound in terms of $n$.

With the guarantee in \Cref{lem:ub-l-simple-kingb}, we can then bound the expected regret in each pass and deduce a total regret bound of \Cref{algo:large-m-simple} in the following theorem.
\begin{theorem}[regret bound for \Cref{algo:large-m-simple}]
    For any input instance, the expected regret of \Cref{algo:large-m-simple} in pass $p$ is $O\tp{2^{-P+p}n^{\frac{2-2^{P+1}}{2^{P+1}-1}}T^{\frac{2^P}{2^{P+1}-1}}}$.

    Furthermore, the total regret of \Cref{algo:large-m-simple} satisfies 
    \[
        R(T)\leq O\tp{n^{\frac{2-2^{P+1}}{2^{P+1}-1}}T^{\frac{2^P}{2^{P+1}-1}}}.
    \]
\end{theorem}
\begin{proof}
    Let $R_p$ denote the expected regret generated in the $p$-th pass. When $p=1$, we have $R_1\leq L_1$. When $p>1$, we compute the regret of the two $\textsc{FindBest}$ subroutines separately. For the first one, the total regret can be decomposed into two parts: the regret of the OSMD process with respect to $\k_{p-1}$, which can be bounded by $\sqrt{2nL_{p,1}}$ according to Proposition~\ref{prop:ub-l-OSMD}; and the regret generated due to the gap between $\k_{p-1}$ and $\a^*$, which can be bounded by $L_{p,1}\cdot \E{\Delta_{\-{king}_{p-1}}}$.
    
    For the second subroutine, we can do a similar decomposition. The expected regret of this process is no larger than $\sqrt{6L_{p,2}} + L_{p,2}\cdot \tp{\sqrt{\frac{2n}{L_{p,1}}} + \E{\Delta_{\-{king}_{p-1}}}}$, where $\sqrt{\frac{2n}{L_{p,1}}}$ is the bound for the gap between the mean of optimal arm in $\+S_p$ and $\k_{p-1}$. This bound is guaranteed by Lemma~\ref{lem:ub-bandit-subroutine-simple}. 

    Recall that $L_1=O\tp{2^{-2P}\cdot n^{-2\lambda_1}T^{1-\lambda_1}}$, $L_{p,1} =\begin{cases} n^3\cdot L_{p-1,2}, &p>2 \\ n^3\cdot L_1, & p=2\end{cases}$ and $L_{p,2}=O\tp{2^{-2P+2p}\cdot n^{-2\lambda_p}T^{1-\lambda_p}}$ for $p>1$ where $\lambda_p=\frac{2^{P-p+1}-1}{2^{P+1}-1}$. Combining the two subroutines together,
    \begin{align*}
        R_p & \leq L_{p,1}\cdot \E{\Delta_{\-{king}_{p-1}}} + \sqrt{2nL_{p,1}} + L_{p,2}\cdot \tp{\sqrt{\frac{2n}{L_{p,1}}} +  \E{\Delta_{\-{king}_{p-1}}}}+ \sqrt{6L_{p,2}}\leq O\tp{2^{-P+p}\cdot n^{\frac{2-2^{P+1}}{2^{P+1}-1}}T^{\frac{2^P}{2^{P+1}-1}}}.
    \end{align*}

    For the exploitation phase, the regret can also be decomposed into the part of OSMD and the regret generated by $\Delta_{\k_P}$. Similarly, these terms can be bounded by $\E{\Delta_{\-{king}_{P}}}\cdot T + \sqrt{2n T}$.
    
    In total, we have
    \begin{align*}
        R(T)&\leq  L_1 + \sum_{p=2}^P O\tp{2^{-P+p}n^{\frac{2-2^{P+1}}{2^{P+1}-1}}T^{\frac{2^P}{2^{P+1}-1}}} + \E{\Delta_{\-{king}_{P}}}\cdot T + \sqrt{2n T} =O\tp{n^{\frac{2-2^{P+1}}{2^{P+1}-1}}T^{\frac{2^P}{2^{P+1}-1}}}.
    \end{align*}
\end{proof}

\section{Upper Bound Analysis for the Large Memory Case}\label{sec:ub-l}

In this section we prove \Cref{thm:ub-general-m} when $m\ge \frac{8n}{9}$. The algorithm and its analysis are a natural generalization of the $m=n-1$ case studied in \Cref{sec:ub-l-simple}. Remember that in order to simplify the presentation, in \Cref{algo:large-m-simple} the algorithm is allowed to remember the identity of the dropped arm in each pass. So here we must address the following two issues arising in the generalization:
\begin{itemize}
    \item How to choose more than one inferior arms to drop in each pass?
    \item How to modify the algorithm so that we do not need to remember the identity of the dropped armed?
\end{itemize}


\subsection{The algorithm}\label{subsec:ub-l-algo}
Consider the game with \(n\) arms, $P$ ($1\leq P\leq \log\log T - \log\tp{12\log\frac{n}{n-m}}$) passes and memory size of $m$ ($m\geq \frac{8n}{9}$). Assume $T\geq (n+1)^2$. We adopt the notations defined in \Cref{sec:ub-l-simple}.
Let $\a^*$ denote the optimal arm (with the largest empirical mean). For any arm with name $\!a$ let $\mu_{\!a}$ be its mean reward. Let $\+M_p$ be the set of arms in memory at the beginning of pass $p$ for any $p\in [P]$. Let $\+M_{P+1}$ denote the arms in memory in the exploitation phase. 

\paragraph{Overview of the algorithm} 
Similar to \Cref{algo:large-m-simple}, we basically implement a BAR routine to select $n-m$ arms to drop at each pass. The idea for BAR is simple, and the intuition is explained in detail with an offline BAR algorithm in \Cref{sec:ub-l-warmup}. However, since in our model then algorithm cannot remember the identity of the dropped arms, we have to implement a streaming-friendly BAR algorithm, which complicates things a bit. 

At the beginning of the first pass, the algorithm reads in $m$ arms. It randomly selects an arm set $\+S_1$ of size $n-m+1$ from the memory to perform a $\textsc{FindBest}$ subroutine to identify a good arm, denoted as $\a'_{1}$. It then discards the $n-m$ arms from the set $\+S_1\setminus\set{\a'_1}$ to make space for the remaining $n-m$ arms in the stream.

For the subsequent pass $p$ ($p>1$), the memory is full at the beginning of the pass. We first execute a shorter $\textsc{FindBest}$ on all the arms in memory to find a good arm, $\a'_{p,1}$. Then, we randomly discard $\frac{m}{2}$ arms from the remaining $m-1$ arms. For each arm that arrives next, we put it into memory as long as memory is not full. Note that if we find an arriving arm is already in memory, the newly arriving arm and the one already in memory are essentially the same arm, so they only occupy one position. After including new $\frac{m}{2}$ arms in the stream, the memory becomes full. We randomly select $n-m+1$ arms from the newly included $\frac{m}{2}$ arms. Let $\+S_p$ denote  the set of these $n-m+1$ arms together with $\a'_{p,1}$. Then we execute a longer $\textsc{FindBest}$ on $\+S_p$ in order to select a good arm $\a'_{p,2}$ to keep. Then we discard $n-m$ arms from $\+S_p\setminus\set{\a'_{p,1},\a'_{p,2}}$. At this point, we know that in the remaining pass, there are only $n-m$ arms that are not in current memory. We then sequentially read in these $n-m$ arms and end this pass.

\paragraph{The description of the algorithm} Set $\lambda_p = \frac{2^{P-p+1}-1}{2^{P+1}-1}$ for any $p\in [P]$. Set the following parameters:
\begin{itemize}
    \item the sample times in the first pass $L_1 = \ceil{2^{-2P}\cdot (n-m)^{3\lambda_1}n^{-2\lambda_1}T^{1-\lambda_1}}$;
    \item for any $p\in[P]\setminus\set{1}$, 
          \begin{itemize}
            \item[$\circ$] the sample times of the second $\textsc{FindBest}$ in the $p$-th pass $L_{p,2}= \ceil{2^{-2P+2p-2}\cdot (n-m)^{3\lambda_p}n^{-2\lambda_p}T^{1-\lambda_p}}$;
            \item[$\circ$] the sample times of the first $\textsc{FindBest}$ in the $p$-th pass $L_{p,1}=\ceil{\frac{m^3 L_{p-1,2}}{(n-m)^3 }}$ for $p>2$ and $L_{p,1}=\ceil{\frac{m^3 L_{1}}{(n-m)^3 }}$ for $p=2$.
          \end{itemize}  
\end{itemize}
To achieve tight dependency on $n$ and $m$, the parameters for large memory cases inevitably entail some intricacy compared to the simple $m=n-1$ case. But in general, the values of these parameters are set for the same purpose as explained in \Cref{subsec:ub-l-simple-algo}. The full algorithm is represented in \Cref{algo:large-m}. 


\begin{algorithm}[ht]
    \caption{Multi-pass algorithm for MAB with large memory ($m\ge \frac{8n}{9}$)}
    \label{algo:large-m}
    \Input{number of passes $P$, time horizon $T$, memory size $m$}
\begin{algorithmic}[1]
    \For{$p=1,2,\dots, P$}
        \If{$p=1$}
            \State Read in the first $m$ arms\;
            \State Sample $n-m+1$ arms from $\+M_1$ uniformly at random and let $\+S_1$ be the set of these arms\; \label{line:ub-l-sample1}
            \State $\a'_{1}\gets$ \Call{FindBest}{$\+S_1,L_{1}$}\;
            \State Drop the arms in $\+S_1\setminus\set{\a'_{1}}$\;
            \State Read in the remaining $n-m$ arms\;\label{line:ub-l-read1}
        \Else
            \State $\a'_{p,1}$ = \Call{FindBest}{$\+M_p,L_{p,1}$}\;
            \State Choose $\frac{m}{2}$ arms from $\+M_p\setminus \set{\a'_{p,1}}$ uniformly at random to drop\;\label{line:ub-l-drop0}
            \State Read in $\frac{m}{2}$ arms that are not currently in memory\; \label{line:ub-l-read}
            \State Choose $n-m+1$ arms from the new $\frac{m}{2}$ arms and let $\+S_p$ be the set of these arms plus  $\a'_{p,1}$\;\label{line:ub-l-sample}
            \State $\a'_{p,2}\gets$ \Call{FindBest}{$\+S_p,L_{p,2}$}\;
            \State Choose $n-m$ arms in $\+S_p\setminus\set{\a'_{p,2},\a'_{p,1}}$ uniformly at random to drop\;\label{line:ub-l-drop}
            \State Read in the remaining $n-m$ arms that are not currently in memory\;\label{line:ub-l-read2}
        \EndIf
    \EndFor
    \State Run $\textsc{MirrorDescent}$ on the remaining $m$ arms until the game ends\;
\end{algorithmic}
\end{algorithm}


Recall the two issues mentioned at the beginning of this section which stand as the main differences between the $m=n-1$ case algorithm in \Cref{sec:ub-l-simple} and the algorithm we are going to design here: 
\begin{itemize}
    \item How to choose more than one inferior arms to drop?
    \item How to modify the algorithm so that we do not need to remember the identity of the dropped armed?
\end{itemize}

Let us explain how the two issues are addressed in our new algorithm. We resolve the first one by incorporating $n-m+1$ arms together with $\a'_{p,1}$ to perform the second $\textsc{FindBest}$ in pass $p$ (p>1). This allows us to drop $n-m$ arms while keeping $\a'_{p,1}$ and $\a'_{p,2}$ in memory (see Line~\ref{line:ub-l-sample} to Line~\ref{line:ub-l-drop}). The effect of this operation will be proved in \Cref{lem:ub-l-kingb}.

For the second problem, note that there is a critical distinction compared to \Cref{algo:large-m-simple}: one arm can be dropped twice in a single pass. To be specific, it is possible that one arm is in memory at the beginning of pass $p$ ($p>1$) and is chosen to be dropped in Line~\ref{line:ub-l-drop0} of \Cref{algo:large-m}. Subsequently, it is read into memory in the process of Line~\ref{line:ub-l-read} but unfortunately, is dropped again in Line~\ref{line:ub-l-drop}. The fundamental reason for such a phenomenon is the requirement that the identity of dropped arms should be forgotten. By discarding $\frac{m}{2}$ arms after the first $\textsc{FindBest}$, this condition can be ensured to be satisfied without incurring additional regret cost.

\subsection{The analysis}\label{sec:ub-l-analysis}

Let $\eps_1=\sqrt{\frac{n-m}{L_1}}$ and let $\eps_p=\sqrt{\frac{n-m}{L_{p,2}}}$ for $p>1$. Let $\-{king}_p$ be the best arm in memory when pass $p$ ends, i.e., after Line~\ref{line:ub-l-read1} or Line~\ref{line:ub-l-read2} is executed. For an arm with name $\!a$, we let $\Delta_{\!a} = \mu_{\a^*} - \mu_{\!a}$. The following lemma proves the effect of our algorithm on solving the BAR problem, which is a generalized version of \Cref{lem:ub-l-simple-kingb}.

\begin{lemma}\label{lem:ub-l-kingb}
$\E{\Delta_{\k_p}}\leq \frac{10(n-m)\eps_p}{m}$. 
\end{lemma}
\begin{proof}
    Recall that $\+M_p$ is the set of arms in memory at the beginning of pass $p$. If the optimal arm $\a^*$ is not in $\+M_1$ for pass $1$, or $\a^*$ is not in memory at Line~\ref{line:ub-l-sample} for pass $p>1$, according to \Cref{algo:large-m}, it must be in $\+M_{p+1}$ and therefore $\Delta_{\-{king}_p}=0$. Hence, we only need to consider the case when $\a^*$ is in memory when we drop arms from $S_p$.

    When $p=1$, we have
    \begin{align*}
        \E{\Delta_{\-{king}_1}}&= \Pr{\a^*\in \+S_1}\cdot \E{\mu_{\a^*}-\mu_{\a'_1} \mid \a^*\in \+S_1}\\
        &\leq \frac{n-m+1}{m-1}\cdot \E{\mu_{\a^*}-\mu_{\a'_1} \mid \a^*\in \+S_1}\\
        &\leq \frac{n-m+1}{m-1}\cdot \sqrt{\frac{2(n-m+1)}{L_1}} \leq \frac{10(n-m)\eps_1}{m},
    \end{align*}
    where the third inequality holds due to Lemma~\ref{lem:ub-bandit-subroutine-simple}.
    
    For $p>1$, if $\a^*$ is in memory when we drop arms from $S_p$, it is dropped only if $\a^*\in \+S_p\setminus\set{\a'_{p,2},\a'_{p,1}}$. Then we have 
    \begin{align}
        \E{\Delta_{\-{king}_p}}&\leq \Pr{\a^*\in \+S_p\setminus\set{\a'_{p,1}}}\cdot \E{\mu_{\a^*} - \max\set{ \mu_{\a'_{p,1}},\mu_{\a'_{p,2}}}\mid \a^*\in \+S_p\setminus\set{\a'_{p,1}}} \notag \\ 
        &= \Pr{\a^*\in \+S_p\mid \a^*\neq \a'_{p,1}}\Pr{\a^*\neq\a'_{p,1}}  \notag \\
        &\quad \cdot \E{\mu_{\a^*} - \max\set{ \mu_{\a'_{p,1}},\mu_{\a'_{p,2}}}\mid \a^*\in \+S_p\setminus\set{\a'_{p,1}}} \notag \\
        & = \frac{2(n-m+1)}{m} \cdot\Pr{\a^*\neq\a'_{p,1}} \E{\mu_{\a^*} - \max\set{ \mu_{\a'_{p,1}},\mu_{\a'_{p,2}}}\mid \a^*\in \+S_p\setminus\set{\a'_{p,1}}}\notag\\
        & \leq \frac{2(n-m+1)}{m} \cdot \E{\mu_{\a^*} - \mu_{\a'_{p,2}}\mid \a^*\in \+S_p\setminus\set{\a'_{p,1}}} \notag \\
        &\overset{(\spadesuit)}{\leq} \frac{2(n-m+1)}{m}\sqrt{\frac{2(n-m+2)}{L_{p,2}}} \leq \frac{10(n-m)}{m}\sqrt{\frac{n-m}{L_{p,2}}} \leq \frac{10(n-m)\eps_p}{m}, \notag
    \end{align}
    where $(\spadesuit)$ follows from Lemma~\ref{lem:ub-bandit-subroutine-simple} since Lemma~\ref{lem:ub-bandit-subroutine-simple} holds no matter whether $\a^*=\a'_{p,1}$ or not. In total, we have $\E{\Delta_{\-{king}_p}} \leq \frac{10(n-m)\eps_p}{m}$ for any $p\in[P]$.
\end{proof}

Similar to \Cref{algo:large-m-simple}, it is crucial to include an additional $n-m+1$ arms in the Line~\ref{line:ub-l-sample} of \Cref{algo:large-m} instead of $n-m$ arms. This allows us to keep $\a'_{p,1}$ in the memory in Line~\ref{line:ub-l-drop}. By doing so, the premise condition for $\a^*$ being discarded is that $\a^*\in  \+S_p\setminus\set{\a'_{p,1}}$. The probability of this event can be bounded by $\frac{n-m+1}{n-1}$, and the saving here is also key to our tight regret bound in terms of $n$ and $m$.

Now we are ready to prove \Cref{thm:ub-general-m} when $m\geq \frac{8n}{9}$.
\begin{theorem}[regret bound for \Cref{algo:large-m}]
    For any input instance, the expected regret of \Cref{algo:large-m-simple} in pass $p$ is 
    \[
        O\tp{2^{-P+p} \cdot(n-m)^{1+\frac{2^{P}-2}{2^{P+1}-1}}n^{\frac{2-2^{P+1}}{2^{P+1}-1}} T^{\frac{2^P}{2^{P+1}-1}}}. 
    \]

    Furthermore, the total regret of \Cref{algo:large-m} satisfies 
    \[
        R(T)\leq O\tp{(n-m)^{1+\frac{2^{P}-2}{2^{P+1}-1}}n^{\frac{2-2^{P+1}}{2^{P+1}-1}} T^{\frac{2^P}{2^{P+1}-1}}},
    \]
    when $1\leq P\leq \log\log T - \log\tp{12\log\frac{n}{n-m}}$ and $T\geq (n+1)^2$.
\end{theorem}
\begin{proof}
    Let $R_p$ denote the expected regret generated in the $p$-th pass. When $p=1$, we have $R_1\leq L_1$. When $p>1$, we compute the regret of the two $\textsc{FindBest}$ subroutines separately. For the first one, the total regret can be decomposed into two parts: the regret of the OSMD process with respect to $\k_{p-1}$, which can be bounded by $\sqrt{2mL_{p,1}}$ according to Proposition~\ref{prop:ub-l-OSMD}; and the regret generated due to the gap between $\k_{p-1}$ and $\a^*$, which can be bounded by $L_{p,1}\cdot \E{\Delta_{\-{king}_{p-1}}}$.
    
    For the second subroutine, we can do a similar decomposition. The expected regret of this process is no larger than 
    \[
        \sqrt{2(n-m+2)L_{p,2}} + L_{p,2}\cdot \tp{\sqrt{\frac{2m}{L_{p,1}}} + \E{\Delta_{\-{king}_{p-1}}}},
    \] 
    where $\sqrt{\frac{2m}{L_{p,1}}}$ is the bound for the gap between the mean of optimal arm in $\+S_p$ and $\k_{p-1}$. This bound is guaranteed by Lemma~\ref{lem:ub-bandit-subroutine-simple}. 

    Recall that $\eps_1=\sqrt{\frac{n-m}{L_1}}$, $\eps_p=\sqrt{\frac{n-m}{L_{p,2}}}$ and $L_{p,1} = \ceil{\frac{m^3}{(n-m)^2 \eps_{p-1}^2}}$ for $p>1$. We set $L_1=\ceil{2^{-2P}\cdot (n-m)^{3\lambda_1}n^{-2\lambda_1}T^{1-\lambda_1}}$ and $L_{p,2}=\ceil{2^{-2P+2p-2}\cdot (n-m)^{3\lambda_p}n^{-2\lambda_p}T^{1-\lambda_p}}$ for $p>1$ where $\lambda_p=\frac{2^{P-p+1}-1}{2^{P+1}-1}$.
    From Lemma~\ref{lem:ub-l-kingb}, we know that $\E{\Delta_{\-{king}_{p}}}\leq \frac{10(n-m)\eps_p}{m}$. Combining the two subroutines together,
    \begin{align*}
        R_p & \leq L_{p,1}\cdot \E{\Delta_{\-{king}_{p-1}}} + \sqrt{2mL_{p,1}} + L_{p,2}\cdot \tp{\sqrt{\frac{2m}{L_{p,1}}} +  \E{\Delta_{\-{king}_{p-1}}}}+ \sqrt{2(n-m+2)L_{p,2}}\\
        &\leq \tp{L_{p,1} + 2L_{p,2}}\cdot \frac{10 (n-m)\eps_{p-1}}{m} + \sqrt{2mL_{p,1}} + \sqrt{2(n-m+2)L_{p,2}}\\
        &= O\tp{2^{-P+p} (n-m)^{1+\frac{2^{P}-2}{2^{P+1}-1}}n^{\frac{2-2^{P+1}}{2^{P+1}-1}} T^{\frac{2^P}{2^{P+1}-1}}}.
    \end{align*}

    For the exploitation phase, the regret can also be decomposed into the part of OSMD and the regret generated by $\Delta_{\k_P}$. Similarly, these terms can be bounded by $\E{\Delta_{\-{king}_{P}}}\cdot T + \sqrt{2m T}$. In total, we have
    \begin{align*}
        R(T)&\leq  L_1 + \sum_{p=2}^P R_p + \frac{10 (n-m)\eps_P}{m}\cdot T + \sqrt{2m T}\\
        &\leq  L_1 + \sum_{p=2}^P O\tp{2^{-P+p} (n-m)^{1+\frac{2^{P}-2}{2^{P+1}-1}}n^{\frac{2-2^{P+1}}{2^{P+1}-1}} T^{\frac{2^P}{2^{P+1}-1}}} + \frac{20 (n-m)\eps_P}{n}\cdot T + \sqrt{2m T}\\
        &\leq  \sum_{p=1}^P O\tp{2^{-P+p} (n-m)^{1+\frac{2^{P}-2}{2^{P+1}-1}}n^{\frac{2-2^{P+1}}{2^{P+1}-1}} T^{\frac{2^P}{2^{P+1}-1}}} + O\tp{(n-m)^{1+\frac{2^{P}-2}{2^{P+1}-1}}n^{\frac{2-2^{P+1}}{2^{P+1}-1}} T^{\frac{2^P}{2^{P+1}-1}}} + \sqrt{2m T}\\
        &=O\tp{(n-m)^{1+\frac{2^{P}-2}{2^{P+1}-1}}n^{\frac{2-2^{P+1}}{2^{P+1}-1}} T^{\frac{2^P}{2^{P+1}-1}}}.
    \end{align*}
\end{proof}

\section{Upper Bound Analysis for the Small Memory Case}\label{sec:ub-s}
In this section we present our algorithm and its analysis for the case when the memory $m<\frac{8n}{9}$. As previously mentioned, in the regime where $n-m=\Omega(n)$, increasing the memory contributes minimally to solving the BAR problem, whose sample complexity is of the same order as that of the BAI problem. Therefore, when $m$ is small, we can directly leverage existing BAI algorithms to construct our algorithm.

Our algorithm still follows the explore-then-commit framework. We aim to identify an $\eps_p$ optimal arm, $\k_p$, after the $p$-th pass, where $\eps_p$ decreases as $p$ increases. During the exploitation phase, we exploit $\k_P$. To this end, we first implement a streaming BAI algorithm in each pass. However, we cannot trivially use the BAI algorithm by feeding in every arriving arm since we have to account for the regret caused by the BAI algorithm itself. As a result, 
we incorporate an admission mechanism for each incoming arm. Only when the arm is not significantly inferior compared to the current best arm, do we allow it to participate in the competition of BAI. This ensures that the BAI process does not result in excessive regret. 

\subsection{A black box streaming BAI algorithm}

Our algorithm for streaming MAB  will use a streaming BAI algorithm as a black box and leave its construction and analysis in the appendix. In this section, we specify the properties needed for the streaming BAI algorithm. 

The streaming BAI problem we consider only contains a single pass. Let $\-{BAI}(\eps,\delta)$ be an $(\eps, \delta)$-PAC streaming algorithm. Let $\+S\subseteq[n]$ be the set of arms that are input into $\-{BAI}(\eps,\delta)$. This means that the arms in stream faced by $\-{BAI}(\eps,\delta)$ are only those arms in $\+S$. We allow $\+S$ to be a random set, and therefore, the size of $\+S$ is not known beforehand. Let $\a^*_{\+S}$ be the optimal arm in $\+S$.

\begin{proposition}\label{prop:BAIblackbox}
    There exists a $(\eps,\delta)$-PAC streaming algorithm $\-{BAI}(\eps,\delta)$ with input arm set $\+S$ satisfying:
    \begin{enumerate}
        \item \textbf{Correctness}: $\-{BAI}(\eps,\delta)$ returns an arm $\k\in \+S$ such that $\E{\mu_{\a^*_{\+S}} - \mu_{\k}}\leq \eps$.
        \item \textbf{Storage}: $\-{BAI}(\eps,\delta)$ uses a memory size of at most $m$. Furthermore, there is one position among $m$ memory slots specifically allocated for storing the newly arrived arm. 
        \item \textbf{Benchmark arm}: $\-{BAI}(\eps,\delta)$ maintains a good arm $\a^*_{\!{max}}$ in memory during its process and $\a^*_{\!{max}}$ always satisfies $\E{\mu_{k_1}-\mu_{\a^*_{\!{max}}}}\leq \eps$, where $k_i$ is the $i$-th arm fed into $\-{BAI}(\eps,\delta)$.
        \item \textbf{Regret guarantee}: When $\+S$ is a random set and for every $i$, $\mu-\E{\mu_{k_i}}\leq \eps'$ for some fixed numbers $\mu\in(0,1),\eps'\in(\eps,1)$, 
        the expected regret generated by the $\-{BAI}(\eps,\delta)$ process with regard to an arm with mean $\mu$ is bounded by $O\tp{\frac{n\eps'}{\eps^2}\tp{\log\tp{\frac{1}{\delta}} + \-{ilog}^{(m-1)}(n)}}$. \footnote{Here the arm with mean $\mu$ does not have to be in $\+S$. Let $\tau$ be the sample complexity of the BAI algorithm, $A_t$ be the index of arm played in round $t$ for $t\in[\tau]$. This regret is defined as $\E{\sum_{t=1}^{\tau} \mu-\mu_{A_t}}$. This expectation includes the randomness of the instance, the algorithm and random set $\+S$.}
    \end{enumerate}
\end{proposition}

The first property tells how well the BAI problem is solved by $\-{BAI}(\eps,\delta)$, which is the key to keep low regret in our multi-pass streaming MAB algorithm.

The second property is about the memory usage of $\-{BAI}(\eps,\delta)$. The reason for emphasizing that particular memory slot for the arriving arm is because this memory will be reused in the streaming MAB algorithm. This will be further elaborated on in the algorithm description of \Cref{sec:ub-s-MAB}.

The third property is needed here because we will actually use $\a^*_{\!{max}}$ as a benchmark to design the admission mechanism. This is further explained in the algorithm overview in \Cref{sec:ub-s-MAB}.

The last property describes the regret of the BAI algorithm. This property states that if all input arms are good compared to the arm with mean reward $\mu$, then the total regret with regard to this arm generated by that BAI process can be bounded. Together with the admission mechanism, we can guarantee the arms fed into $\-{BAI}(\eps,\delta)$ are comparable with the optimal arm and thus property $4$ indicates a regret upper bound for the BAI subroutines in our streaming MAB algorithm.

We will give a proof of \Cref{prop:BAIblackbox} in \Cref{sec:ub-s-BAI} by showing that a multi-level $(\eps,\delta)$-PAC algorithm designed in \cite{AW20} and \cite{MPK21} satisfies all these properties.

\subsection{The algorithm for streaming MAB}\label{sec:ub-s-MAB}
Now we describe the $P$-pass algorithm for streaming MAB with $n$ arms. The memory size is $m$ where $2\leq m< \frac{8n}{9}$. Let $r=\min\set{\lfloor\log^*(n+1)\rfloor, m-1}$. Define $\lambda_p=\frac{2^{P-p+1}-1}{2^{P+1}-1}$ for $p\in[P]$ and choose $\delta=\frac{1}{4}$. Let $\a^*$ be the optimal arm in stream. Let $\set{\eps_p}_{p\in[P]\cup \set{0}}$ be some parameters to be set later. 

\paragraph{Overview of the algorithm} Basically, our algorithm for streaming MAB uses $\-{BAI}(\eps,\delta)$ as a subroutine by feeding arms into $\-{BAI}(\eps_p,\delta)$ in the pass $p$ and exploit the $\eps_p$-best arm returned by the algorithm. However, as mentioned at the beginning of this section, we have to control the regret contributed by $\-{BAI}(\eps_p,\delta)$ itself, so we have to ensure those arms fed into $\-{BAI}(\eps_p,\delta)$ are not too bad. To this end, we first compare it with the best arm we found in the pass $p-1$ through a $\textsc{FindBest}$ subroutine described in \Cref{algo:bandit-subroutine}. Let $\k_{p-1}$ denote the arm returned from pass $p-1$. We want to use ``mean reward gap to $\k_{p-1}$'' as a metric to evaluate the quality of arms in the next pass, but we do not want to allocate additional memory specifically for storing $\k_{p-1}$. Therefore, we feed $\k_{p-1}$ as the first arm in pass $p$ to BAI($\eps_p,\delta$). 
The third property of $\-{BAI}(\eps_p,\delta)$ in \Cref{prop:BAIblackbox} ensures that $\a^*_{\!{max}}$ has comparable performance to $\k_{p-1}$. We actually use $\a^*_{\!{max}}$ as the filter to determine whether the newly arrived arm can enter the $\-{BAI}(\eps_p,\delta)$ process.

\paragraph{The description of the algorithm} 
Recall that $\-{ilog}^{(k)}(a)=\max\set{\log\tp{\-{ilog}^{(k-1)}(a)}, 1}$ for any $a\geq 1$, $k\in \bb N^+$. 
Set the following parameters:
\begin{itemize}
    \item $\eps_0=1$ and $\eps_p={2^{P-p+1}\cdot \tp{\frac{(n+1)\cdot \-{ilog}^{(m-1)}(n+1)}{T}}^{\frac{1-\lambda_p}{2}}}$ for each $p\in[P]$;
    \item  for each $p\in[P]$, set $s^{(p)} = \ceil{\frac{2^{5}}{\eps_p^2}\tp{\log \frac{2^{3}\cdot\-{ilog}^{(r-1)}(n+1)}{\delta}}}$, which is the sample times of the $\textsc{FindBest}$ subroutine in pass $p$.
\end{itemize}

The algorithm is described in \Cref{algo:small-m}. One detail to note is that the $\-{BAI}(\eps_p,\delta)$ process may have to deal with a stream with $n+1$ arms in the worst case, since $\k_{p-1}$ is actually the first arm the $\-{BAI}(\eps_p,\delta)$ process faces. Therefore, when applying the properties of $\-{BAI}(\eps_p,\delta)$ in subsequent analysis, we should regard its input arm set $S$ as a subset of $[n+1]$ rather than $[n]$. We also need to clarify that although \Cref{algo:small-m} seems to require an additional memory slot due to Line~\ref{line:readin}, the total memory size can actually be the same with the BAI algorithm. This is because we have designated one specific memory slot for storing the arriving arm in stream for $\-{BAI}(\eps,\delta)$ (property $2$ in \Cref{prop:BAIblackbox}). This position can be utilized to store the arm in Line~\ref{line:readin}.

Recall that we assume $P\leq \log\log T - \log\tp{12\log\frac{n}{n-m}}$ and $T\geq (n+1)^2$. For the sake of simplicity, our algorithm and analysis only consider the case where the game will not end after pass $P$. It is easy to verify that if the $T$-round game has already been completed before entering exploitation phase, our conclusions and analysis still hold.


\begin{algorithm}[ht]
    \caption{Multi-pass algorithm for MAB with small memory ($m\le \frac{8n}{9}$)}
    \label{algo:small-m}
    \Input{number of passes $P$, time horizon $T$}
\begin{algorithmic}[1]
    \State let $\k_0$ be the first arm in pass $1$
    \For{$p=1,2,\dots,P$}
        \State initialize a \Call{BAI}{$\eps_p,\delta$} instance and feed $\k_{p-1}$ into \Call{BAI}{$\eps_p,\delta$}
        \For{ each arriving arm $\a_i$}\label{line: new arm}
        \State read $\a_i$ into memory \label{line:readin}
        \State $a_i^{(p)}\gets$\Call{FindBest}{$\set{\a^*_{\!{max}}, \a_i}$,$s^{(p)}$}\label{line:callfindbest}
        \If{$a_i^{(p)}\neq \a_i$}
            \State go to Line~\ref{line: new arm}
        \Else
            \State feed $\a_i$ into \Call{BAI}{$\eps_p,\delta$}
            \State update $\a^*_{\!{max}}$ \Comment{$\a^*_{\!{max}}$ is the good arm maintained by \Call{BAI}{$\eps_p,\delta$}}
        \EndIf
        \EndFor
    \State $\k_p\gets$\Call{BAI}{$\eps_p,\delta$} 
    \EndFor
    \State pull $\k_P$ until the game ends
\end{algorithmic}
\end{algorithm}

The $\textsc{FindBest}$ subroutine is also used in designing the algorithm for the large memory case. We reiterate its effect below.
\begin{lemma}[Lemma~\ref{lem:ub-bandit-subroutine-simple} restated]\label{lem:findbest-perform}
Let $\a^*_{\+S}$ be the optimal arm in $\+S$. Then the $\textsc{FindBest}$ subroutine (\Cref{algo:bandit-subroutine}) guarantees that $\E{\mu_{\a^*_{\+S}} - \mu_{\a'}}\leq \sqrt{\frac{2\abs{\+S}}{L}}$, where $\a'$ is the arm returned by \Cref{algo:bandit-subroutine}.
\end{lemma}

Recall that $\k_p$ is the output arm of the BAI($\eps_p,\delta$) procedure at the end of pass $p$ ($p\geq 1$). Lemma~\ref{lem: B-pass eps_b opt} demonstrates that with a filtering process to all arms in stream, we can still guarantee that the arm returned by BAI($\eps_p,\delta$) is, in expectation, a $2\eps_p$-optimal arm with regard to $\a^*$.

\begin{lemma}\label{lem: B-pass eps_b opt}
    At the end of pass $p$, we have $\E{\mu_{\a^*}- \mu_{\k_p}}\leq 2\eps_p$.
\end{lemma}
\begin{proof}

Let us consider the point when $\a^*$ arrives ($\a_i=\a^*$). Then
\[
    \E{\mu_{\a^*}-\mu_{\k_p}} = \E{\mu_{\a^*} - \mu_{a_i^{(p)}}} + \E{\mu_{a_i^{(p)}} - \mu_{\k_p}}.
\]
It follows from Lemma~\ref{lem:findbest-perform} that $\E{\mu_{\a^*} - \mu_{a_i^{(p)}}}\le \sqrt{\frac{4}{s^{(p)}}} \le \eps_p$.

To bound $\E{\mu_{a_i^{(p)}} - \mu_{\k_p}}$, we note that in either cases ($a_i^{(p)}=\a_{\!{max}}^*$ or $a_i^{(p)}=\a^*$ after Line~\ref{line:callfindbest}), $a_i^{(p)}$ has been, or will be fed into $\textsc{BAI}(\eps_p,\delta)$. Therefore, by property $1$ in \Cref{prop:BAIblackbox} of the BAI($\eps_p,\delta$) process, we have $\E{\mu_{a_i^{(p)}} - \mu_{\k_p}}\le \eps_p$.

\end{proof} 

The following lemma shows the role of the $\textsc{FindBest}$ subroutine in \Cref{algo:small-m}. With this operation, it is ensured that every arm fed into BAI($\eps_p,\delta$) is a $4\eps_{p-1}$-optimal arm in expectation, and then we can utilize its fourth property to bound the total regret.

\begin{lemma}\label{lem:Delta_aip}
    For any $p\in[P]$ and $i\in[n]$, assuming $T\geq (n+1)^2$, we have $\E{\mu_{\a^*} - \mu_{a_i^{(p)}}}\leq 4\eps_{p-1}$.
\end{lemma}
\begin{proof}
    At the time that $\a_i$ arrives in pass $p$, we have 
    \[
        \E{\mu_{\a^*} - \mu_{a_i^{(p)}}} = \E{\mu_{\a^*} - \mu_{\k_{p-1}}} + \E{\mu_{\k_{p-1}} - \mu_{\a^*_{\!{max}}}} + \E{\mu_{\a^*_{\!{max}}} - \mu_{a_i^{(p)}}}.
    \]
    From Lemma~\ref{lem: B-pass eps_b opt}, we have $\E{\mu_{\a^*} - \mu_{\k_{p-1}}}\leq 2\eps_{p-1}$. Since $\k_{p-1}$ is the first arm fed into BAI($\eps_p,\delta$), applying the third property of \Cref{prop:BAIblackbox}, we have $\E{\mu_{\k_{p-1}} - \mu_{\a^*_{\!{max}}}}\leq \eps_p$. From Lemma~\ref{lem:findbest-perform}, $\E{\mu_{\a^*_{\!{max}}} - \mu_{a_i^{(p)}}}\leq \sqrt{\frac{4}{s^{(p)}}} \le \eps_p$. Combining all above, $\E{\mu_{\a^*} - \mu_{a_i^{(p)}}}\leq 2\eps_{p-1} + 2\eps_p\leq 4\eps_{p-1}$.
\end{proof}

Then we are ready to bound the regret of \Cref{algo:small-m}.
\begin{theorem}
Assume $T\geq (n+1)^2$. For any input instance, the expected regret of \Cref{algo:small-m} in pass p is 
$$
    O\tp{2^{-P+p}\cdot n^{\frac{2^P-1}{2^{P+1}-1}}T^{\frac{2^P}{2^{P+1}-1}} \cdot \tp{\-{ilog}^{(m-1)}(n)}^{\frac{2^P-1}{2^{P+1}-1}}}
$$

Furthermore, the total regret of \Cref{algo:small-m} satisfies
$$
    R(T)\leq O\tp{n^{\frac{2^P-1}{2^{P+1}-1}}T^{\frac{2^P}{2^{P+1}-1}} \cdot \tp{\-{ilog}^{(m-1)}(n)}^{\frac{2^P-1}{2^{P+1}-1}}}.
$$
\end{theorem}
\begin{proof}
    Property $3$ in \Cref{prop:BAIblackbox} indicates $\E{\mu_{\k_{p-1}}-\mu_{\a^*_{\!{max}}}}\leq \eps_{p}$. Together with Lemma~\ref{lem: B-pass eps_b opt}, for each $\a^*_{\!{max}}$, we have
    \[
        \E{\mu_{\a^*}-\mu_{\a^*_{\!{max}}}}\leq \E{\mu_{\a^*}-\mu_{\k_{p-1}}} + \E{\mu_{\k_{p-1}}-\mu_{\a^*_{\!{max}}}} \leq 2\eps_{p-1} + \eps_p< 3\eps_{p-1}.
    \]
    
    Denote the expected regret of the $n$ $\textsc{FindBest}$ processes in pass $p$ as $R^{(p)}_{\textsc{FindBest}}$. Then it can be decomposed into two parts: the regret of each $\textsc{MirrorDescent}$ subroutine and the regret generated from the gap $\E{\mu_{\a^*}-\mu_{\a^*_{\!{max}}}}$. Therefore, $R^{(p)}_{\textsc{FindBest}}$ satisfies
\[
    R^{(p)}_{\textsc{FindBest}}\leq n\sqrt{4s^{(p)}} + ns^{(p)}\cdot 3\eps_{p-1}.
\]

    Then we consider the regret of BAI($\eps_p,\delta$), denoted as $R^{(p)}_{\textsc{BAI}}$. With \Cref{lem:Delta_aip}, we can utilize property $4$ in \Cref{prop:BAIblackbox} to deduce that
    \begin{align*}
        R^{(p)}_{\textsc{BAI}}&\leq O\tp{\frac{n\eps_{p-1}}{\eps_p^2}\tp{\log \tp{\frac{1}{\delta}} +\-{ilog}^{(m-1)}(n)}}\\
        &\leq O\tp{\frac{n\eps_{p-1}}{\eps_p^2}{\-{ilog}^{(m-1)}(n)}}.
    \end{align*}
    Let $R_p$ be the expected regret of pass $p$. Since $\lambda_p=\frac{2^{P-p+1}-1}{2^{P+1}-1}$ and $\eps_p={2^{P-p+1}\cdot \tp{\frac{(n+1)\cdot \-{ilog}^{(m-1)}(n+1)}{T}}^{\frac{1-\lambda_p}{2}}}$, by direct calculation,
    \[
        R_p=R^{(p)}_{\textsc{FindBest}} + R^{(p)}_{\textsc{BAI}} \leq O\tp{2^{-P+p}\cdot n^{\frac{2^P-1}{2^{P+1}-1}}T^{\frac{2^P}{2^{P+1}-1}} \cdot \tp{\-{ilog}^{(m-1)}(n)}^{\frac{2^P-1}{2^{P+1}-1}}}.
    \]

For the exploitation phase, the expected regret $R_{\!{expt}}$ can be bounded by $T\cdot \E{\mu_{\a^*}-\mu_{\k_P}}\leq 2\eps_P T$. Therefore, the total regret satisfies
\begin{align*}
    R(T)&= \sum_{p=1}^P R_p + R_{\!{expt}} \\
    &\leq \sum_{p=1}^P O\tp{2^{-P+p}\cdot n^{\frac{2^P-1}{2^{P+1}-1}}T^{\frac{2^P}{2^{P+1}-1}} \cdot \tp{\-{ilog}^{(m-1)}(n)}^{\frac{2^P-1}{2^{P+1}-1}}} + 2\eps_P T\\
    &=O\tp{n^{\frac{2^P-1}{2^{P+1}-1}}T^{\frac{2^P}{2^{P+1}-1}} \cdot \tp{\-{ilog}^{(m-1)}(n)}^{\frac{2^P-1}{2^{P+1}-1}}}.
\end{align*}

\end{proof}

\section{Informal Elaboration on the Lower Bounds}\label{sec:lb-informal}
Before delving into the proof of the regret lower bound, we first provide an informal elaboration on how we obtain lower bounds in \Cref{thm:lb-regret} on the special case that the memory size $m=n-1$. Similar to the case of the upper bound, our treatment for this special case already showcases our main idea for establishing lower bounds. The rigorous proof for the general lower bounds, i.e., the proof of \Cref{thm:lb-regret}, will be presented in the next section. 


Recall that we say a pass ends once the last arm in the pass is read into the memory. This means that any exploitation of the arms in the memory after this point is treated as samples in the next pass. For the ease of expression, we make the same assumption as in \Cref{sec:ub-l-simple} that the player can remember the identity of discarded arms. This only strengthens our lower bound result.

Note that many ambiguous terms such as ``tell apart'', ``distinguish'' used in the explanation below will be made rigorous in \Cref{sec:lb}, mainly via a \emph{likelihood argument}.


\subsection{Single pass}
We start from the simplest case, a lower bound for a single pass, i.e., $P=1$. We first look at what our \Cref{algo:large-m} does in this case. Generally speaking, since there is only one pass, the algorithm read in the first $m=n-1$ arms into the memory, explore them for $L$ rounds, and then drop a poor-performing arm to read in the last one. It then uses an optimal MAB algorithm to play for the remaining $T-L$ rounds. Clearly there is a trade-off for picking the parameter $L$:
\begin{itemize}
    \item In instances where the optimal arm comes at last, the main regret is caused by the first $L$ rounds, which can be bounded by $O\tp{L}$. 
    \item When the optimal arm comes among the first $m$ arms, one has to ensure that the optimal arm is not dropped after the first $L$ rounds. This is a BAR task. We showed that $L$ samples are sufficient to retain a good arm with expected mean reward gap of $\frac{1}{n}\cdot \frac{1}{\sqrt{L}}$ compared to $\a^*$. The main regret in these instances is due to the failure to retain $\a^*$, which is bounded by $O\tp{ \frac{1}{n\sqrt{L}}\cdot(T-L)}$. 
\end{itemize}

To balance the two cases, we choose $L\approx \tp{\frac{T}{n}}^{\frac{2}{3}}$ to minimize the worst case regret:
\begin{align*}
    &\phantom{{}={}}\min_{L} \max\set{O\tp{L}, O\tp{\frac{1}{n}\cdot \frac{(T-L)}{\sqrt{L}} }}  \approx O\tp{\tp{T/n}^{\frac{2}{3}}}.
\end{align*}

Then we analyze this strategy from the perspective of lower bounds. We claim that such a strategy with $L\approx \tp{{T}/{n}}^{\frac{2}{3}}$ is necessary. Let $\eps\approx \frac{1}{\sqrt{L}}$. We choose a number $j\in [n-1]$ uniformly at random and construct an instance $H$ such that the $n$ input arms have Bernoulli rewards with mean $\frac{1}{2},\dots,\frac{1}{2},\frac{1}{2}+\eps,\frac{1}{2},\dots, \frac{1}{2},1$ respectively where the $j$-th arm has mean $\frac{1}{2}+\eps$. In other words, in $H$, the first $j-1$ arms and  the $(j+1)$-th to $(n-1)$-th arms are fair coins, and the $j$-th arm is slightly better while the last arm in the stream is the optimal arm which has mean reward $1$. We construct another instance $H'$ by replacing the last arm of $H$ in stream with a fair coin.

Note that no algorithm can tell apart $H$ and $H'$ without seeing the last arm in the stream. Therefore, if the player samples for more than $L$ rounds before the first dropping in expectation, the regret on $H'$ is $\Omega\tp{L}\approx\Omega\tp{\tp{{T}/{n}}^{\frac{2}{3}}}$. However, if the sample time is less than $L\approx \frac{1}{\eps^2}$, the probability of failing to distinguish $\a_j$ with a fair coin is $\Omega(1)$. Since the biased arm is concealed among $m=n-1$ arms on $H'$, the probability of failing to retain that arm is approximately $\Omega\tp{\frac{1}{n}}$. So the total regret generated on $H'$ is $\Omega\tp{\frac{1}{n}\cdot \eps (T-L)}\approx\Omega\tp{\tp{{T}/{n}}^{\frac{2}{3}}}$. This lower bound analysis in turn prescribes a reasonable behavior to achieve low regret, which is exactly the same with our algorithm.

\subsection{Two-pass}
We can generalize the above arguments to the $P=2$ case. First examine the behavior of an algorithm for the two-pass streaming MAB. The algorithm still works in an explore-drop way. The crucial difference is that when allowed to see the arms for two passes, a clever player should explore less in the first pass to avoid the regret incurred in case that the optimal arm comes at last. Conversely, the player can explore more in the second pass, since after one pass, there is a reasonable probability that the optimal arm is already in the memory.

In general, our algorithm plays $L_1$ rounds in the first pass, drop one arm to incorporate the last arm in the stream and start the exploration of the second pass. In the second pass, we will explore the $m$ arms in memory for $L_2$ rounds and then drop one among them to read in the remaining arm in the stream. Then the algorithm performs an optimal $(T-L_1-L_2)$-round MAB algorithm with arms in the memory. 

For those instances where a significantly better arm comes at last in the first pass, the regret is mainly contributed by the first $L_1$ rounds, which is bounded by $O\tp{L_1}$.
For other instances, we have proved that a good arm with expected mean reward gap smaller than $\frac{1}{n}\cdot \frac{1}{\sqrt{L_p}}$ compared to $\a^*$ will be retained at the end of pass $p$ ($p=1,2$). There are two possible situations: 
\begin{itemize}
    \item The main regret is caused by the $L_2$ rounds of exploration in pass two, which is due to dropping $\a^*$ by mistake in pass one. We can bound this by $O\tp{\frac{1}{n\sqrt{L_1}}\cdot L_2}$;
    \item The main regret is generated during the exploitation due to dropping $\a^*$ by mistake in pass two. This regret is bounded by $O\tp{\frac{1}{n\sqrt{L_2}}\cdot (T-L_1-L_2)}$.
\end{itemize}
Similar to the single pass case, we will choose $L_1$ and $L_2$ to minimize
\[
    \max\set{O\tp{L_1}, O\tp{\frac{L_2}{n\sqrt{L_1}}}, O\tp{\frac{(T-L_1-L_2)}{n\sqrt{L_2}}}}.
\]
By direct calculation, the optimal choice is $L_1\approx {T^{\frac{4}{7}}}/{n^{\frac{6}{7}}}$ and $L_2\approx {T^{\frac{6}{7}}}/{n^{\frac{2}{7}}}$, resulting in a regret bound of $O\tp{T^{\frac{4}{7}}/n^{\frac{6}{7}}}$.

Then we argue that this kind of strategy is indeed optimal by providing a family of hard instances. Choose $\eps_0=\frac{1}{2}$, $\eps_1\approx\frac{1}{\sqrt{L_1}}$ and $\eps_2\approx\frac{1}{\sqrt{L_2}}$. Our hard instances consist of arms with Bernoulli rewards. That is, in each instance the $i$-th arm has rewards drawn from a Bernoulli distribution with parameter $\mu_i$. For $p=0,1,2$ and $j\in[n]$, we construct the instance $H_{j}^{(p)}$ with $n$ arms as: the $j$-th arm has $\mu_j=\frac{1}{2} + \eps_p$; for every $k\in[n]\setminus\set{j}$, $\mu_k=\frac{1}{2}$. 


Note that this index does not necessarily represent the order of arms in stream. For $H_j^{(p)}$, we require the arms in each of the first $p$ passes to come in an order from arm $1$ to arm $n$ (here the order refers to the order of arm identity, or equivalently, arm index). Furthermore, the optimal arm, arm $j$, should arrive \emph{at the end} in the $p+1$-th pass (if $p<2$). The order of arms of this instance in the other passes can be arbitrary\footnote{We slight abuse notation here. Ideally $H_j^{(p)}$ should be a collection of instances since the orders in passes $p+2,\dots,P$ are not fixed. However, when we say an instance $H_j^{(p)}$, we mean \emph{one instance among} this collection of instances where the order and types of arms in each pass fulfilling the requirements specified before.}. We use $\wh L_1,\wh L_2$ to denote the actual number of explorations performed by an algorithm in pass one and pass two respectively. 


Let $\+E_1$ (or $\+E_2$) be the event that the optimal arm is dropped at the end of pass one (or pass two).
Our reasoning for the lower bound goes as follows. We assume that a given algorithm can achieve $c'\cdot{T^{\frac{4}{7}}/{n^{\frac{6}{7}}}}$ minimax regret for some small universal constant $c'>0$. Then we analyze the behavior of this algorithm on the hard instances. For some universal constant $c>0$ (depending on $c'$), the following holds:

\paragraph{[A regret of $c'\cdot T^{\frac{4}{7}}/{n^{\frac{6}{7}}}$ $\implies $ $\wh L_1\le c\cdot L_1$ on $H_n^{(0)}$ whp]}\footnote{The term ``whp'' stands for ``with high probability'' here and below, while its precise meaning will be clean in \Cref{sec:lb}} Since the algorithm can achieve $c'\cdot T^{\frac{4}{7}}/{n^{\frac{6}{7}}}$ regret, we have $\wh L_1\le c\cdot L_1$ on $H_n^{(0)}$. Otherwise, if $\wh L_1> c\cdot L_1$, the exploration phase in pass one would incur a regret of $\frac{L_1}{2} = \frac{c}{2}\cdot {T^{\frac{4}{7}}}/{n^{\frac{6}{7}}}$ on $H_n^{(0)}$, which is a contraction by picking $c\ge 2c'+1$.

\paragraph{[$\wh L_1\le c\cdot L_1$ on $H_n^{(0)}$ $\implies $ $\wh L_1\le c\cdot L_1$ on both $H_j^{(1)}$ and $H_j^{(2)}$ for some $j\in [n]$ whp]} 
The reasoning here is similar to the single pass case. Any algorithm cannot distinguish between arms with mean reward $\frac{1}{2}$, $\frac{1}{2}+\eps_1$ and $\frac{1}{2}+\eps_2$ with sufficiently large probability in the $\wh L_1\le c\cdot L_1\approx \frac{c}{\eps_1^2}$ samples when $c$ is a small enough constant. As a result, $\wh L_1\le c\cdot L_1$ also holds for $H_j^{(1)}$ and $H_j^{(2)}$ for some $j\in [n]$ whp.


\paragraph{[$\wh L_1\le c\cdot L_1$ on $H_j^{(1)}$ $\implies $ $\wh L_2\le c\cdot L_2$ on $H_j^{(1)}$ conditioned on $\+E_1$]} On the other hand, for small enough $c$, if the number of samples $\le c\cdot L_1\approx \frac{c}{\eps_1^2}$ before the first dropping in pass one, the algorithm cannot recognize the biased arm with large enough probability on $H_j^{(1)}$ and the bad event $\+E_1$ happens with probability $\Omega\tp{\frac{1}{n}}$. Once $\+E_1$ happens, each round of exploration in pass two will incur a regret of $\eps_1$ in expectation since by our construction, the optimal $\frac{1}{2}+\eps_1$ is at the end of pass two. By picking an appropriate $c$, the number of explorations in pass two cannot exceed $c\cdot L_2$ on $H_j^{(1)}$ conditioned on $\+E_1$. Otherwise, the exploration phase in pass two will incur $\approx \frac{c}{n}\cdot \eps_1 \cdot L_2 > c'\cdot T^{\frac{4}{7}}/n^{\frac{6}{7}}$ regret on $H_j^{(1)}$. 


\paragraph{[$\wh L_2\le c\cdot L_2$ on $H_j^{(1)}$ conditioned on $\+E_1\implies\wh L_2\le c\cdot L_2$ on $H_j^{(2)}$ whp]} Note that the $c\cdot L_1$ samples in pass one help little to distinguish $H_j^{(1)}$ and $H_j^{(2)}$. When $\+E_1$ happens, at the beginning of pass two, the arms in memory are all fair coins on $H_j^{(1)}$. For small enough $c$, $c\cdot L_2$ samples are not enough to tell apart $H_j^{(2)}$ and an all fair coin instance. In other words, the algorithm cannot decide whether $\+E_1$ has happened or the input is actually $H_j^{(2)}$. Since we have $\wh L_2 \le c\cdot L_2$ on $H_j^{(1)}$ conditioned on $\+E_1$, this bound also holds for $H_j^{(2)}$.

\paragraph{[$\wh L_1\le c\cdot L_1$ and $\wh L_2\le c\cdot L_2$ on $H_j^{(2)}$ $\implies$ a regret of $\Omega\tp{T^{\frac{4}{7}}/{n^{\frac{6}{7}}}}$ on $H_j^{(2)}$ whp]} For small enough $c$, with the total samples bounded by $c\cdot(L_1+L_2)\approx \frac{c}{\eps_2^2}$, the bad event $\+E_2$ will happen with probability $\Omega\tp{\frac{1}{n}}$ on $H_j^{(2)}$. This causes a regret of $\Omega\tp{\frac{1}{n}\cdot \eps_2 (T-L_1-L_2)}\ge c'\cdot\tp{{T^{\frac{4}{7}}}/{n^{\frac{6}{7}}}}$ in the final exploitation phase on $H_j^{(2)}$ by picking an appropriate $c$.

\bigskip
\noindent Noting that we allow $c'$ to be an arbitrarily small constant, it is clear that we can always find an appropriate constant $c>0$ fulfilling all requirements above.

Summarizing, we prove that for some instance $H_j^{(2)}$, any algorithm with optimal regret must behave exactly as \Cref{algo:large-m} on $H_j^{(2)}$ in the sense that the number of explorations in pass $p$ satisfies $\wh L_p \approx L_p$ for $p=1,2$. On the other hand, the algorithm has $\Omega(T^{\frac{4}{7}}/n^{\frac{6}{7}})$ regret on this instance. 

\subsection{Multi-pass}
We then consider the general multi-pass case. This is a direct generalization of the two-pass argument above. 

For the upper bound, we play $L_p$ rounds in pass $p$ for each $p\in[P]$, where $L_p$ is set to be increasing in $p$. If some instances have an outstanding arm (whose mean reward is larger than the second best one by $\Theta(1)$) comes at the end of pass one, the regret generated by the explorations in pass one will be $O(L_1)$. Furthermore, this constitutes a major portion of the total regret for these instances since the explorations in subsequent passes are sufficient to identify and retain that outstanding arm.

For an instance where the reward gap $\Delta$ between the best and second-best arm is smaller, there also exists a specific pass $p$ among the $P$ passes, which is challenging for this instance (The value of $p$ depends on $\Delta$. The smaller $\Delta$ is, the larger $p$ is). Intuitively, the regret before the $p$-th pass, $O\tp{\Delta\cdot\tp{\sum_{i=1}^{p-1} L_i}}$, is negligible because $\Delta$ is relatively small. The regret after that pass is also a small part since subsequent samples are enough to recognize the gap $\Delta$ and thus can tell apart the optimal arm and other suboptimal ones. We only need to consider the regret of pass $p$ caused by dropping the optimal arm by mistake in pass $p-1$. 
We are guaranteed that at the end of pass $p-1$, the best arm in memory has expected mean reward gap of at most $\frac{1}{n\sqrt{L_{p-1}}}$ compared to $\a^*$. Therefore, the regret for such instances would be $O\tp{\frac{1}{n\sqrt{L_{p-1}}}\cdot L_p}$. Similarly, for those instances that the main regret is generated during the exploitation phase due to discarding the optimal arm at the end of pass $P$, the regret is $O\tp{\frac{1}{n\sqrt{L_{P}}}\cdot (T-\sum_{p=1}^P L_p)}$.

Thus, the algorithm should choose $L_1,\dots,L_P$ to minimize
\[
    \max\set{O\tp{L_1}, O\tp{\frac{1}{n\sqrt{L_{1}}}\cdot L_2},\dots, O\tp{\frac{1}{n\sqrt{L_{P-1}}}\cdot L_P}, O\tp{\frac{1}{n\sqrt{L_{P}}}\cdot (T-\sum_{p=1}^P L_p))}}.
\]
This indicates the optimal choice is $L_p\approx n^{-2\lambda_p}T^{1-\lambda_p}$ where $\lambda_p\approx\frac{2^{P-p+1}-1}{2^{P+1}-1}$. The total regret bound is $O\tp{n^{\frac{2-2^{P+1}}{2^{P+1}-1}}T^{\frac{2^{P}}{2^{P+1}-1}}}$.

To see the lower bound, we consider $P+1$ families of Bernoulli instances. Let $\eps_p=\frac{1}{\sqrt{L_p}}$ for $p\in[P]$ and $\eps_0=\frac{1}{2}$. For $p\in[P]\cup\set{0}$ and $j\in[n]$, we construct the instance $H_{j}^{(p)}$ with $n$ arms as: the $j$-th arm has mean $\mu_j=\frac{1}{2} + \eps_p$; for every $k\in[n]\setminus\set{j}$, $\mu_k=\frac{1}{2}$. 

For $H_j^{(p)}$, we require the arms in each of the first $p$ passes to come in an order from arm $1$ to arm $n$ (here the order refers to the order of arm identity, or equivalently, arm index). In pass $p+1$ (if $p<P$), we require the optimal arm, arm $j$, arrives at the end in stream for $H_j^{(p)}$. The order of arms for $H_j^{(p)}$ in the other passes can be arbitrary.
 Similar to the $2$-pass case, we use $\wh L_p$ to denote the actual number of explorations performed by an algorithm in pass $p$. Let $\+E_p$ be the event that the optimal arm is dropped at the end of pass $p$. We apply the following inductive approach based on the same logic with $2$-pass case. For sufficiently small constant $c'>0$, there exists a constant $c>0$ satisfying the following. 

\paragraph{[A regret of $c'\cdot n^{\frac{2-2^{P+1}}{2^{P+1}-1}}T^{\frac{2^{P}}{2^{P+1}-1}} \implies \wh L_{1}\le c\cdot L_{1}$ on $H_n^{(0)}$ and $H_j^{(p_2)}$ for each $1\leq p_2\leq P$ whp]} Assume an algorithm can achieve a regret of $c'\cdot n^{\frac{2-2^{P+1}}{2^{P+1}-1}}T^{\frac{2^{P}}{2^{P+1}-1}}$. Then in the first pass, the exploration before the first dropping on $H_n^{(0)}$ should not exceed $c\cdot L_1$, or the regret will be $\frac{c}{2}\cdot L_1>c'\tp{n^{\frac{2-2^{P+1}}{2^{P+1}-1}}T^{\frac{2^{P}}{2^{P+1}-1}}}$. Such a sample complexity bound also holds for other $H_j^{(p_2)}$ ($1\leq p_2\leq P$) since the $c\cdot L_1\approx \frac{c}{\eps_1^2}$ samples are not sufficient to distinguish between arms with mean reward $\frac{1}{2}$ and $\frac{1}{2}+\eps_{p_2}$ with sufficiently large probability.

\paragraph{[Fix $p\in [P-1]$. $\forall p_1\leq p, \wh L_{p_1}\le c\cdot L_{p_1}$ on $H_j^{(p_2)}$ for all $p\leq p_2\leq P \implies\wh L_{p+1}\le c\cdot L_{p+1}$ on $H_j^{(p_3)}$ for all $p<p_3\leq P$ whp]} This can be verified by induction on $p$. The $p=1$ case is verified above. For larger $p$, by the inductive assumption, the samples on $H_j^{(p)}$ in the first $p$ passes is bounded by $c\cdot \tp{L_1+\cdots+L_p}\approx \frac{c}{\eps_p^2}$, which are not sufficient to recognize an arm with mean reward $\frac{1}{2}+\eps_p$ among fair coins with sufficiently large probability when $c$ is small. The bad event $\+E_p$ will happen with probability $\Omega\tp{\frac{1}{n}}$ on $H_j^{(p)}$. Once $\+E_p$ happens, on $H_j^{(p)}$, each sample in pass $p+1$ before reading into the last arm will incur a regret of $\eps_p$ in expectation. Therefore, we need $\wh L_{p+1}=c\cdot L_{p+1}$ for $H_j^{(p)}$ conditioned on $\+E_p$, or we may suffer a regret of $\frac{c}{n}\cdot \eps_p L_{p+1}> c'\cdot {n^{\frac{2-2^{P+1}}{2^{P+1}-1}}T^{\frac{2^{P}}{2^{P+1}-1}}}$ in pass $p+1$. 

When $\+E_p$ happens on $H_j^{(p)}$, the arms in memory at the beginning of pass $p$ are all fair coins. However, the $c\cdot L_{p+1}$ samples are not enough to tell apart $H_j^{(p_3)}$ and an all fair coin instance for any $p<p_3\leq P$. In other words, the algorithm cannot decide whether $\+E_p$ has happened or the input is actually $H_j^{(p_3)}$. Moreover, the $c\cdot \tp{L_1+\cdots+L_p}\approx \frac{c}{\eps_p^2}$ samples in previous passes cannot provide much information to help distinguish as well. Therefore, $\wh L_{p+1}\le c\cdot L_{p+1}$ also holds on $H_j^{(p_3)}$ for all $p<p_3\leq P$. 

\paragraph{[$\forall p\leq P, \wh L_{p}\le c\cdot L_{p}$ on $H_j^{(P)}$ $\implies$ a regret of $\Omega\tp{n^{\frac{2-2^{P+1}}{2^{P+1}-1}}T^{\frac{2^{P}}{2^{P+1}-1}}}$]} With above induction, we show that $\forall p\leq P,\wh L_{p}=O(L_{p})$ on instance $H_j^{(P)}$ when $c$ is sufficiently small. The total number of samples is $c\cdot \tp{L_1+\cdots+L_P}\approx \frac{c}{\eps_P^2}$. When $c$ is sufficiently small, this number of samples does not help in identifying the optimal arm and the probability of dropping the optimal arm at the end of pass $P$ is $\Omega\tp{\frac{1}{n}}$ on $H_j^{(P)}$. This leads to a regret of $\Omega\tp{\frac{1}{n}\cdot \eps_P \tp{T-\sum_{p=1}^P L_p}}\approx\Omega\tp{n^{\frac{2-2^{P+1}}{2^{P+1}-1}}T^{\frac{2^{P}}{2^{P+1}-1}}}$ during exploitation on $H_j^{(P)}$. Therefore, a regret lower bound of $\Omega\tp{n^{\frac{2-2^{P+1}}{2^{P+1}-1}}T^{\frac{2^{P}}{2^{P+1}-1}}}$ holds.

Summarizing, we prove that for some instance $H_j^{(P)}$, any algorithm with optimal regret must behave exactly as \Cref{algo:large-m} on $H_j^{(P)}$ in the sense that the number of explorations in pass $p$ satisfies $\wh L_p \approx L_p$ for any $p\in [P]$. On the other hand, the algorithm has $\Omega\tp{n^{\frac{2-2^{P+1}}{2^{P+1}-1}}T^{\frac{2^{P}}{2^{P+1}-1}}}$ regret on this instance. 

\section{Lower Bound Analysis}\label{sec:lb}
In this section, we will prove a tight (up to logarithm factors in $n$ and $P$) asymptotic regret lower bound for streaming MAB algorithms with number of passes $1\leq P\leq \log {\log T} - \log \tp{14\log 8(n-m)}$, number of arms $n\geq 3$ and arbitrary memory size $m$ satisfying $2\le m\le n-1$. 
In the following, we always assume that the time horizon $T$ is sufficiently large.

We pick $\eps_0=\frac{1}{2}$ and $\eps_1,\eps_2,\dots,\eps_P\in(0,1/4]$ as parameters for our hard instances.
Our hard instances consist of arms with Bernoulli rewards. That is, in each instance the $i$-th arm has rewards drawn from a Bernoulli distribution with parameter $\mu_i$. For $p\in[P]\cup\set{0}$ and $i\in[n]$, we construct the instance $H_{i}^{(p)}$ with $n$ arms as: the $i$-th arm has $\mu_i=\frac{1}{2} + \eps_p$; for every $k\in[n]\setminus\set{i}$, $\mu_k=\frac{1}{2}$ (note that this arm index does not represent the arm order in the stream). Let $H_0$ be the instance with $n$ fair coins. 

We also design an order for these instances: For $H_0$, we require the arms in each pass to come in an order from arm $1$ to arm $n$. For $H_j^{(p)}$, the arm order should be the same with $H_0$ in the first $p$ passes (here the order refers to the order of arm identity, or equivalently, arm index). In pass $p+1$ (if $p<P$), we require the optimal arm, arm $j$, arrives at the end in stream for $H_j^{(p)}$. The order of arms for $H_j^{(p)}$ in the other passes can be arbitrary.

\paragraph{Additional notations and terminologies}

Let us fix some notations and terminologies in the lower bound proof. Let $\Omega$ be the set of all possible histories during the $T$-round game. That is, we can regard each $\omega\in \Omega$ as a sequence of ${(a_t,b_t,m_t)}$ with length $T$, where $a_t\in[n]$ is the arm that the algorithm pulls at round $t$, $b_t$ is the corresponding reward and $m_t$ records the current arms in memory. A streaming algorithm, randomized or deterministic, naturally induces a probability measure with outcomes $\Omega$. So we will casually call $\omega\in \Omega$ an outcome and $S\subseteq \Omega$ an event. Recall that we say a pass \emph{ends} once the last arm in the pass is read into the memory. This means that any exploitation of the arms in the memory after this point is treated as samples in the next pass. 
For any $\omega\in \Omega$, let $\omega(1:p)$ be the truncated sequence of $\omega$ which contains only the history of first $p$ passes. We use $\omega(1:P+1)$ to denote $\omega$ itself. 

We emphasize that the algorithm we consider is allowed to be randomized. As a result, in the subsequent analysis, the sources of randomness for all the expectations and probabilities come from both the input instance and the algorithm itself. 

\paragraph{Overview of the proof}
Our proof is mainly based on the simple yet powerful \emph{likelihood} argument. This technique has been widely used in statistics and in online learning problem to derive lower bounds (e.g.,~\cite{MT04}). For any $\omega \in \Omega$, the likelihood ratio between two input instances $H$ and $H'$ on $\omega$ is defined as $\frac{\Pr[H]{\omega}}{\Pr[H']{\omega}}$.

Assume $H$ and $H'$ are two instances with Bernoulli arms, and they differ in only one arm. In instance $H$, the mean reward of this arm is $\frac{1}{2}$, while in $H'$, it is $\frac{1}{2}+\eps$. The likelihood argument states that if an $\omega\in \Omega$ has bounded sample times and concentrated sample results, then $\Pr[H]{\omega}\approx \Pr[H']{\omega}$. This means that any algorithm cannot effectively distinguish between $H$ and $H'$ when seeing this $\omega$.

We observe that it is impossible for an algorithm to keep small regret in both exploration and exploitation phases. Let $r$ be the regret lower bound we are aiming to prove in \Cref{thm:lb-regret}. Assume in contradiction that the expected regret of a $P$-pass algorithm is less than $r$ on all input instances. We choose $\set{\eps_p}_{p\in[P]\cup\set{0}}$ satisfying: $\eps_1 \approx {\sqrt{\frac{n-m}{ r}}}$ and $\frac{(n-m)}{\eps_p^2} \approx {\frac{rn}{(n-m)\eps_{p-1}}}$. Let $L_p$ be the sample times of the algorithm in pass $p$.

Then we have $\E[H_0]{L_1}<\tilde O\tp{\frac{n-m}{\eps_1^2}}$. This is because the algorithm cannot distinguish between $H_0$ and $H_j^{(0)}$ before observing the last arm, and excessive exploration in the early stages would result in significant regret on $H^{(0)}_j$ for any $j\in[n]$. Due to the limitation on $\E[H_0]{L_1}$, there must be some $j\in[n]$ such that with high probability, arm $j$ is discarded after being pulled $\tilde O\tp{\frac{1}{\eps_1^2}}$ times and its empirical mean is concentrated. From a likelihood argument, arm $j$ in $H^{(1)}_j$ might also be discarded after pass $1$. Then we consider the second pass and can derive that 
\[
    \E[H^{(1)}_j]{L_2\mid \mbox{arm } j \mbox{ is discarded in pass }1}\leq  \frac{r}{\eps_1\Pr[H^{(1)}_j]{\mbox{arm } j \mbox{ is discarded in pass }1}}.
\]

Then we apply the likelihood argument again. If the optimal arm of $H^{(1)}_j$ is dropped in pass~$1$, those arms in memory in the early stages of the second pass all have mean reward $\frac{1}{2}$. Note that the samples from the first pass does not provide much assistance in identifying a $\eps_1$-optimal arm, so the algorithm cannot differentiate between $H_0$ and $H^{(1)}_j$ before seeing the last arm in pass~$2$. Therefore, $L_2$ should also be small on $H_0$. In fact, through a more subtle analysis, we show that $\E[H_0]{L_2\mid S_1}\leq O\tp{\frac{rn}{(n-m)\eps_1}}\approx{\frac{n-m}{\eps_2^2}}$ for some  high probability event $S_1\subseteq \Omega$. Recursively doing this analysis, we can get that with non-zero probability, the total sample times on $H_0$ should be less than 
\[
    O\tp{\frac{rn}{(n-m)\eps_{P}} + \sum_{p=1}^{P} \frac{n-m}{\eps_p^2 \log(64nP)}}.
\]
But this leads to a contradiction since this value is less than $T$ when $T$ is sufficiently large.

In the following analysis, we will choose $\eps_1=\tilde\Theta\tp{2^{-P}\cdot T^{\frac{-2^{P-1}}{2^{P+1}-1}} (n-m)^{\frac{1-2^{P-1}}{2^{P+1}-1}} n^{\frac{2^P-1}{2^{P+1}-1}}}$ and $\frac{\eps_p}{\eps_k}\leq 2^{p-k+3}\cdot \tp{\frac{T}{n}}^{\frac{2^{P-p}-2^{P-k}}{2^{P+1}-1}}$ for any $1\leq k<p\leq P$. The precise values will be determined later.

\subsection{Likelihood ratio}\label{subsec:likelihood}


We define a number of random variables. For every $\omega\in\Omega$, we use $L_p(\omega)$ to denote the number of samples taken in pass $p$ in $\omega$; we let $N_j^{(p)}(\omega)$ and $[K_j^{(p)}(t)](\omega)$ be the number of samples of the arm $j$ in pass $p$ and the cumulative reward of arm $j$ during its first $t$ samples in pass $p$ respectively in the particular history $\omega\in\Omega$. We adopt the convention to write $L_p, N_j^{(p)}, K_j^{(p)}(t)$ when the history $\omega$ is clear from the context.

Define $r_p=\frac{1}{\eps_p^2 \log(64nP)}$ and $\ell_p=\frac{n-m}{2}\cdot r_p=\frac{n-m}{2\eps_p^2 \log(64nP)}$. We say an event $S_p\subseteq \Omega$ has property $\+Q_p$ if:
\begin{itemize}
    \item it has bounded length. To be specific, for every $\omega\in S_p$, we require for all $k\in[p]$, $L_k\leq \ell_k$;
    \item the outcomes of samples are concentrated. Precisely, for all $\omega\in S_p$ and $k\in [p]$, for all $j\in [n]$, the following inequalities hold:
    \begin{itemize}
        \item[$\circ$] $\max _{1\leq t\leq r_k}\abs{K_j^{(k)}(t)-\frac{t}{2}}\leq \sqrt{\frac{r_k}{2}\log 64nP}$;
        \item[$\circ$] $\max _{1\leq t\leq \ell_k}\abs{K_j^{(k)}(t)-\frac{t}{2}}\leq \sqrt{\frac{\ell_k}{2}\log 64nP}$;
    \end{itemize}
\end{itemize}

Here we slightly abuse the notation and let $\Pr{\omega(1:p)}$ denote the probability of the event that the history of the first $p$ passes is $\omega(1:p)$. The following lemma, which forms the foundation of our analysis, shows that when property $\+Q_{p-1}$ is satisfied, instance $H_0$ and $H_j^{(p)}$ exhibit a strong resemblance if arm $j$ is explored for a limited number of times and its cumulative reward is concentrated around $\frac{1}{2}$ in pass $p$.

\begin{lemma}[likelihood ratio lemma]\label{lem:likelihood}
    Assume $T\geq n^2$ and $P\leq \log {\log T} - \log \tp{14\log 8(n-m)}$. Let $S_{p-1}$ be a set with property $\+Q_{p-1}$. Then for any $\omega\in S_{p-1}$, for any $j$ such that
    \begin{itemize}
        \item $N_j^{(p)}\leq r_p$, 
        \item $\max _{1\leq t\leq r_p}\abs{K_j^{(p)}(t)-\frac{t}{2}}\leq \sqrt{\frac{r_p}{2}\log 64nP}$,
    \end{itemize}
    we have
    \[
       \frac{e^{-4}}{2}< \frac{\Pr[H_0]{\omega(1:p)}}{\Pr[H_j^{(p)}]{\omega(1:p)}}< 2e^4,
    \]
    when $T$ is sufficiently large.
\end{lemma}
\begin{proof}
    Fix an outcome $\omega\in\Omega$. Recall that $K_j^{(p)}(t)$ is the cumulative reward of arm $j$ during its first $t$ samples in the $p$-th pass, and $N_j^{(p)}$ is the total number of samples of arm $j$ in pass $p$. For brevity, we write $K_j^{(p)}(N_j^{(p)})$ as $K_j^{(p)}$. By definition, we have
    \begin{align*}
        \frac{\Pr[H_0]{\omega(1:p)}}{\Pr[H_j^{(p)}]{\omega(1:p)}}&=\prod_{k=1}^{p} \frac{\tp{\frac{1}{2}}^{N_j^{(k)}}}{\tp{\frac{1}{2} + \eps_p}^{K_j^{(k)}} \tp{\frac{1}{2} - \eps_p}^{N_j^{(k)}-K_j^{(k)}}}\\
        &=\prod_{k=1}^{p} \frac{1}{\tp{1-4\eps_p^2}^{K_j^{(k)}}\tp{1-2\eps_p}^{N_j^{(k)}-2K_j^{(k)}}}.
    \end{align*}
    From direct calculation, we know that for any $x\in \left[0, \frac{1}{2}\right]$, $1-x\geq e^{-\sqrt{2}x}$. Then for every $k\in[p-1]$,  we have 
    \[
        \tp{1-4\eps_p^2}^{K_j^{(k)}}\tp{1-2\eps_p}^{N_j^{(k)}-2K_j^{(k)}}\leq (1-2\eps_p)^{-2\sqrt{\frac{\ell_k}{2}\log (64nP)}}\leq e^{4\sqrt{2}\eps_p \sqrt{\frac{\ell_k}{2}\log (64nP)}}=e^{\frac{\sqrt{8(n-m)}\eps_p}{\eps_k}},
    \]
    where the first inequality follows from the concentration assumption. Recall that we promised to choose $\eps_1,\dots,\eps_P$ satisfying $\frac{\eps_p}{\eps_k}\leq 2^{p-k+3}\cdot \tp{\frac{T}{n}}^{\frac{2^{P-p}-2^{P-k}}{2^{P+1}-1}}$ for any $1\leq k<p\leq P$. Since $P\leq \log\log T - \log \tp{14\log 8(n-m)}$, we have $\prod_{k=1}^{p-1} e^{-\frac{\sqrt{8(n-m)}\eps_p}{\eps_k}}\geq e^{-0.1}$. Therefore,
    \begin{align*}
        \frac{\Pr[H_0]{\omega(1:p)}}{\Pr[H_j^{(p)}]{\omega(1:p)}}&\geq e^{-0.1}\cdot \frac{1}{\tp{1-4\eps_p^2}^{K_j^{(p)}}\tp{1-2\eps_p}^{N_j^{(p)}-2K_j^{(p)}}}\\
        &\geq e^{-0.1} \cdot \tp{1-2\eps_p}^{2\sqrt{\frac{r_p}{2}\log (64nP)}} \\
        &>\exp{-4\sqrt{2}\eps_p\sqrt{\frac{1}{2\eps_p^2 \log(64nP)}\log (64nP)} -0.1}> \frac{e^{-4}}{2}.
    \end{align*}
    On the other hand, for $k<p$,
    \begin{align*}
        \tp{1-4\eps_p^2}^{K_j^{(k)}}\tp{1-2\eps_p}^{N_j^{(k)}-2K_j^{(k)}} &\geq (1-4\eps_p^2)^{\ell_k} (1-2\eps_p)^{2\sqrt{\frac{\ell_k}{2}\log (64nP)}}\\
        &\geq e^{-4\sqrt{2}\eps_p^2\cdot \frac{n-m}{2\eps_k^2 \log(64nP)} -4\sqrt{2}\eps_p\sqrt{\frac{n-m}{4\eps_k^2 \log(64nP)}\log (64nP)}}\\
        &\geq e^{-\frac{2\sqrt{2}(n-m)\eps_p^2}{\eps_k^2}-\frac{\sqrt{8(n-m)}\eps_p}{\eps_k}}.
    \end{align*}
    Similarly, we have $\prod_{k=1}^{p-1} e^{\frac{2\sqrt{2}(n-m)\eps_p^2}{\eps_k^2}+\frac{\sqrt{8(n-m)}\eps_p}{\eps_k}}\leq e^{0.11}$.
    Therefore, 
    \begin{align*}
        \frac{\Pr[H_0]{\omega(1:p)}}{\Pr[H_j^{(p)}]{\omega(1:p)}}&\leq e^{0.11} \cdot \frac{1}{\tp{1-4\eps_p^2}^{K_j^{(p)}}\tp{1-2\eps_p}^{N_j^{(p)}-2K_j^{(p)}}}\\
        &\leq e^{0.11}\cdot \tp{1-4\eps_p^2}^{-\frac{r_p}{2} - \frac{1}{\sqrt{2}\eps_p}} \tp{1-2\eps_p}^{-2\sqrt{\frac{r_p}{2}\log (64nP)}} \\
        &< \exp{4\sqrt{2}\eps_p^2 \cdot \frac{1}{2\eps_p^2 \log(64nP)} + 4\eps_p + 0.11}\cdot \exp{4\sqrt{2}\eps_p\sqrt{\frac{1}{2\eps_p^2 \log(64nP)}\log (64nP)}}\\
        & < 2e^4,
    \end{align*}
    where the second inequality follows from the concentration assumption, under which we have $K_j^{(p)}\leq \frac{N_j^{(p)}}{2}  + \sqrt{\frac{r_p}{2}\log (64nP)}\leq \frac{r_p}{2} + \frac{1}{\sqrt{2}\eps_p}$. The last inequality holds since $e^{4\sqrt{2}\eps_p^2 \cdot \frac{1}{2\eps_p^2 \log(64nP)} + 4\eps_p + 0.11}< 2$ when $T$ is sufficiently large.
\end{proof}

\subsection{Analysis for each pass}\label{sec:lb-each-pass}
Intuitively, if an algorithm has small expected regret on any single-pass instance, the sample times in exploration phase should be limited. Otherwise, the regret will explode if the optimal arm appears at the end of the stream. One important observation is that this principle also applies to multi-pass instances to some extent. In this section, we will show the existence of a high probability event with property $\+Q_p$, denoted as $S_p$, such that $\E[H_0]{L_{p+1}\mid S_p}$ is bounded if the regret is small. 


Recall that $r_p=\frac{1}{\eps_p^2 \log(64nP)}$, $\ell_p=\frac{n-m}{2}\cdot r_p=\frac{n-m}{2\eps_p^2 \log(64nP)}$ and we say $S_p\subseteq \Omega$ has property $\+Q_p$ if 
\begin{itemize}
    \item it has bounded length. To be specific, we require for all $k\in[p]$, $L_k\leq \ell_k$;
    \item the sample results are concentrated. Precisely, for all $\omega\in S_p$ and $k\in [p]$, for all $j\in [n]$,  we need
    \begin{itemize}
        \item[$\circ$] $\max _{1\leq t\leq r_k}\abs{K_j^{(k)}(t)-\frac{t}{2}}\leq \sqrt{\frac{r_k}{2}\log 64nP}$ and
        \item[$\circ$] $\max _{1\leq t\leq \ell_k}\abs{K_j^{(k)}(t)-\frac{t}{2}}\leq \sqrt{\frac{\ell_k}{2}\log 64nP}$;
    \end{itemize}
\end{itemize}
 
Then Lemma~\ref{lem:lb-prob-of-Sb} shows that we can recursively construct $S_p$ with a probability $\Pr[H_0]{S_p}\geq \Pr[H_0]{S_{p-1}}-\frac{1}{4P} - \frac{1}{2^{P-p+3}}$ as long as $\E[H_0]{L_{p}\mid S_{p-1}}$ is upper bounded by some specific value.
\begin{lemma}\label{lem:lb-prob-of-Sb}
    Assume $\E[H_0]{L_{p}\mid S_{p-1}}\leq \frac{1}{2^{P-p+3}}\cdot \ell_p$ for some $S_{p-1}\subseteq \Omega$ with property $\+Q_{p-1}$ and $\Pr[H_0]{S_{p-1}}\geq 1-\frac{p-1}{4P}-\frac{1}{2^{P-p+3}}$. There exists $S_p\subseteq S_{p-1}$ with property $\+Q_{p}$ such that $\Pr[H_0]{S_p}\geq 1-\frac{p}{4P}-\frac{1}{2^{P-p+2}}$.
\end{lemma}
\begin{proof}
    Recall that $K_j^{(p)}(t)$ is the cumulative reward of arm $j$ during its first $t$ samples in the $p$-th pass. Let $\Tilde{L}=\min\set{\ell_p, r_p}$ (in fact, $\ell_p<r_p$ only when $m=n-1$). Define the following events:
    \begin{itemize}
        \item $A=\set{L_p< \ell_p}$;
        \item $C_j=\set{\max_{1\leq t\leq \Tilde{L}} \abs{K_j^{(p)}(t)-\frac{t}{2}}\leq \frac{1}{\sqrt{2}\eps_p }}$ for each $j\in[n]$;
        \item $C'_j=\set{\max _{1\leq t\leq \ell_p}\abs{K_j^{(p)}(t)-\frac{t}{2}}\leq \sqrt{\frac{\ell_p}{2}\log 64nP}}$ for each $j\in[n]$.
    \end{itemize}

    From the Markov's inequality, 
    \begin{align*}
        \Pr[H_0]{A\mid S_{p-1}}=1-\Pr[H_0]{L_p> \ell_p\mid S_{p-1}} \ge 1-\frac{\E[H_0]{L_{p}\mid S_{p-1}}}{\ell_p}\ge 1- \frac{1}{2^{P-p+3}}.
    \end{align*} 
    
    For event $C_j$, from Lemma~\ref{lem:lb-concentration},
    \[
        \Pr[H_0]{C_j}\geq 1- 2\exp{-\frac{2\cdot \frac{1}{2\eps_p^2 }}{\frac{1}{\eps_p^2 \log(64nP)}}}>1-\frac{1}{8nP}
    \]
    and thus by the union bound, $\Pr[H_0]{\bigcap_{j=1}^n C_j}\geq 1-\frac{1}{8P}$. For event $C'_j$, we can similarly get $\Pr[H_0]{\bigcap_{j=1}^n C'_j}\geq 1-\frac{1}{8P}$.
    
    Let $S_p=A\cap S_{p-1}\cap \tp{\bigcap_{j=1}^n C_j}\cap \tp{\bigcap_{j=1}^n C'_j}$. Obviously $S_p$ satisfies $\+Q_p$. The probability of $S_p$
    \begin{align*}
        \Pr[H_0]{S_p} &\geq \Pr[H_0]{A\cap S_{p-1}} - \tp{1-\Pr[H_0]{\bigcap_{j=1}^n C_j}} - \tp{1-\Pr[H_0]{\bigcap_{j=1}^n C'_j}} \\
        &\geq \Pr[H_0]{A\mid S_{p-1}}\cdot \Pr[H_0]{S_{p-1}} -\frac{1}{4P}\\
        &\geq \tp{1- \frac{1}{2^{P-p+3}}}\cdot\tp{1-\frac{p-1}{4P}-\frac{1}{2^{P-p+3}}}-\frac{1}{4P}\\
        &\geq 1-\frac{p}{4P}-\frac{1}{2^{P-p+2}}.
    \end{align*}
\end{proof}

The following lemma then gives an upper bound of $\E[H_0]{L_{p+1}\mid S_{p}}$ for any event $S_{p}\subseteq\Omega$ with property $\+Q_p$. Recall that $L_{p}(\omega)$ is the sample times in pass $p$ for any sequence $\omega\in \Omega$. We also use $L_{P+1}(\omega)$ to denote the sample times in exploitation phase. Assume $T\geq n^2$ and $P\leq \log {\log T} - \log \tp{14\log 8(n-m)}$ in the following proof.
\begin{lemma}\label{lem:lb-bound-E[Lb]}
    Assume the expected regret of an algorithm $\+A$ is at most $r$. Let $S_{p}\subseteq\Omega$ be a set with property $\+Q_p$ ($p\in[P]$). Then there exists a universal constant $c_1\geq 1$ such that 
    \[
        \E[H_0]{L_{p+1}\mid S_{p}}\leq \frac{c_1 r n}{(n-m)\eps_p \Pr[H_{0}]{S_{p}}},
    \]
    when $T$ is sufficiently large.
\end{lemma}
\begin{proof}
    In each pass, $\+A$ is required to drop at least $n-m$ arms. Recall that $r_p=\frac{1}{\eps_p^2\log(64nP)}$ and $\ell_p=\frac{n-m}{2\eps_p^2\log(64nP)}$. Since $L_{p}\leq \ell_p$ for each $\omega\in S_{p}$, there exists at least $\lceil \frac{n-m}{2}\rceil$ arms among the first dropped $n-m$ arms satisfying $N_j^{(p)}\leq r_p$.
    
    Let $\+J$ denote the set of the first discarded $n-m$ arms in pass $p$. To simplify the notation, we use $N_j$ to represent $N_j^{(p)}$ when $p$ is clear from the context. Therefore, 
    \begin{align*}
        \E[H_0]{L_{p+1}\mid S_{p}} &= \sum_{\omega\in \Omega} L_{p+1}(\omega)\Pr[H_0]{\omega\mid S_{p}}\\
        &\leq \sum_{\omega\in \Omega}\frac{2}{n-m}\sum_{j\in[n]} \bb I[\omega\mbox{ satisfies }\tp{j\in \+J}\cap\tp{ N_j\leq r_p}]\cdot L_{p+1}(\omega)\Pr[H_0]{\omega\mid S_{p}}\\
        &= \frac{2}{n-m}\sum_{j\in[n]} \sum_{\omega\in \Omega}\bb I[\omega\mbox{ satisfies }\tp{j\in \+J}\cap\tp{ N_j\leq r_p}]\cdot L_{p+1}(\omega)\Pr[H_0]{\omega\mid S_{p}}\\
        &=\frac{2}{n-m}\sum_{j\in[n]} \E[H_0]{\bb I[\tp{j\in \+J}\cap\tp{ N_j\leq r_p}]\cdot L_{p+1} \mid S_{p}}\\
        &=\frac{2}{n-m}\sum_{j\in[n]}\tp{\Pr[H_0]{\tp{j\in \+J}\cap\tp{ N_j\leq r_p}\mid S_{p}}\E[H_0]{L_{p+1}\mid \tp{j\in \+J}\cap\tp{ N_j\leq r_p}\cap S_{p}} + 0}\\
        &=\frac{2}{n-m}\sum_{j\in[n]}\Pr[H_0]{\tp{j\in \+J}\cap\tp{ N_j\leq r_p}\mid S_{p}}\E[H_0]{L_{p+1}\mid \tp{j\in \+J}\cap\tp{ N_j\leq r_p}\cap S_{p}} 
    \end{align*}
    
    Once the event ${j\in\+J}$ occurs, since either arm $j$ comes as the last arm in pass $p+1$ for instance $H_j^{(p)}$ (if $p<P$), or the optimal arm will not appear anymore (if $p=P$), all the arms sampled in pass $p+1$ (if $p<P$) or the exploitation phase (if $p=P$) have means of $\frac{1}{2}$.
    This indicates that $\frac{\Pr[H_0]{\omega(1:p+1)}}{\Pr[H_j^{(p)}]{\omega(1:p+1)}}=\frac{\Pr[H_0]{\omega(1:p)}}{\Pr[H_j^{(p)}]{\omega(1:p)}}$ when ${j\in\+J}$ happens.
    For each $\omega$ satisfying $ S_{p}\cap \tp{N_j\leq r_p}$, from Lemma~\ref{lem:likelihood},
    \[
        \frac{e^{-4}}{2}< \frac{\Pr[H_0]{\omega(1:p)}}{\Pr[H_j^{(p)}]{\omega(1:p)}} < 2e^4.
    \]
    
    Define $\Omega_p= \set{\omega(1:p)\mid \omega\in \Omega}$, which is the set of histories of the first $p$ passes. Recall that $L_p(\omega)$ is the sample times in pass $p$ for a fixed $\omega\in \Omega$. When the history of only the first $p$ passes $\omega_p\in \Omega_p$ is given, we also use $L_p(\omega_p)$ to denote the sample times of $\omega_p$ in pass $p$.
    Here we slightly abuse the notation and let $\Pr{\omega_p}$ be the probability of the event that the history of the first $p$ passes is $\omega_p$. It is important to note that once $\omega_{p}$ is fixed, whether $\tp{j\in \+J}\cap\tp{ N_j\leq r_p}\cap S_{p}$ happens is known. We use $\omega_{k}\in \set{\tp{j\in \+J}\cap\tp{ N_j\leq r_p}\cap S_{p}}$ to denote that this event happens when the first $k$ passes history is $\omega_k$ for some $k\geq p$. Hence,
    \begin{align*}
        &\phantom{{}={}}\E[H_0]{L_{p+1}\mid \tp{j\in \+J}\cap\tp{ N_j\leq r_p}\cap S_{p}}\\ 
        &= \sum_{\omega\in \Omega}L_{p+1}(\omega) \cdot \Pr[H_0]{\omega\mid \tp{j\in \+J}\cap\tp{ N_j\leq r_p}\cap S_{p}}\\
        &= \sum_{\omega_{p+1}\in \Omega_{p+1}} L_{p+1}\tp{\omega_{p+1}} \cdot \Pr[H_0]{\omega_{p+1}\mid \tp{j\in \+J}\cap\tp{ N_j\leq r_p}\cap S_{p}} \\
        &= \sum_{\omega_{p+1}\in \Omega_{p+1}} L_{p+1}\tp{\omega_{p+1}} \cdot  \frac{\Pr[H_0]{\omega_{p+1}}\cdot \bb I[\omega_{p+1}\in\set{ \tp{j\in \+J}\cap\tp{ N_j\leq r_p}\cap S_{p}}]}{\Pr[H_0]{\tp{j\in \+J}\cap\tp{ N_j\leq r_p}\cap S_{p}}} \\
        &\leq 4e^8 \sum_{\omega_{p+1}\in \Omega_{p+1}} L_{p+1}\tp{\omega_{p+1}}\cdot  \frac{\Pr[H_j^{(p)}]{\omega_{p+1}}\cdot \bb I[\omega_{p+1}\in\set{ \tp{j\in \+J}\cap\tp{ N_j\leq r_p}\cap S_{p}}]}{\Pr[H_j^{(p)}]{\tp{j\in \+J}\cap\tp{ N_j\leq r_p}\cap S_{p}}} \\
        &=4e^8\E[H_j^{(p)}]{L_{p+1}\mid \tp{j\in \+J}\cap\tp{ N_j\leq r_p}\cap S_{p}},
    \end{align*}
    where the inequality follows from Lemma~\ref{lem:likelihood}

    Note that if arm $j$ is discarded in pass $p$, each pull in pass $p+1$ (or in the exploitation phase of $p=P$) will cause an expected regret of $\eps_{p}$ on $H_j^{(p)}$. Since the total regret is bounded by $r$, we can conclude that 
    \[
        \eps_{p}\cdot\Pr[H_j^{(p)}]{ \tp{j\in \+J}\cap\tp{ N_j\leq r_p}\cap S_{p}}\E[H_j^{(p)}]{L_{p+1}\mid \tp{j\in \+J}\cap\tp{ N_j\leq r_p}\cap S_{p}}\leq r.
    \]
    Combining above analysis, we have
    \begin{align*}
        \E[H_0]{L_{p+1}\mid S_{p}}&\leq \frac{8e^8}{n-m}\sum_{j=1}^n \frac{r}{\eps_{p}\Pr[H_{0}]{S_{p}}}\cdot \frac{\Pr[H_{0}]{\tp{j\in \+J}\cap\tp{ N_j\leq r_p}\cap S_{p}}}{\Pr[H_j^{(p)}]{\tp{j\in \+J}\cap\tp{ N_j\leq r_p}\cap S_{p}}}\\
        &\leq \frac{16e^{12} r n}{(n-m)\eps_p \Pr[H_{0}]{S_{p}}}
    \end{align*}
    where the second inequality follows from Lemma~\ref{lem:likelihood}.
\end{proof}

\subsection{Regret lower bound}\label{sec:lb-lb}
Now we can proceed to establish a lower bound for the regret. Equipped with Lemma~\ref{lem:lb-prob-of-Sb} and~\ref{lem:lb-bound-E[Lb]}, we can choose $\set{\eps_p}_{p\in[P]}$ satisfying $\frac{c_1 r n}{(n-m)\eps_{p-1} \Pr[H_{0}]{S_{p-1}}} = \frac{1}{2^{P-p+3}}\cdot \ell_p$ to construct such event $S_p$ for each pass. Recall that we use $L_{P+1}$ to denote the sample complexity of the exploitation phase. This recursive process allows us to ultimately obtain an upper bound on $\E[H_0]{L_{P+1}\mid S_{P}}$. This will lead to a conflict since $\sum_{p=1}^{P+1}L_p $ should always equal to $T$.
We then formalize this intuition to establish an asymptotic lower bound on the regret in the following proof of \Cref{thm:lb-regret}.

\begin{proof}[Proof of \Cref{thm:lb-regret}]
    Assume 
    \[
        R(T,\+A)\leq r=\frac{1}{2^5c_1}\cdot T^{\frac{2^P}{2^{P+1}-1}} \cdot (n-m)^{1+\frac{(2^P-2)}{2^{P+1}-1}}\cdot  n^{\frac{2-2^{P+1}}{2^{P+1}-1}}\cdot \tp{\log(64nP)}^{\frac{1-2^P}{2^{P+1}-1}}
    \]    
    on any instances, where $c_1$ is the constant in Lemma~\ref{lem:lb-bound-E[Lb]}. 
    Let $\lambda_p=\frac{2^{P-p+1}-1}{2^{P+1}-1}$. Choose
    \begin{align}
        \eps_p&=(n-m)^{2-\frac{3}{2^p}} (c_1n)^{-1+\frac{1}{2^{p-1}}} (r\log(64nP))^{-1+\frac{1}{2^p}}\cdot 2^{-t_p}\notag\\
         &= 2^{-t_p+5-\frac{5}{2^p}}\cdot c_1^{\frac{1}{2^p}} \cdot T^{\frac{\lambda_p-1}{2}} \cdot (n-m)^{\frac{1}{2}-\frac{3\lambda_p}{2}}\cdot n^{\lambda_p} \cdot \tp{\log(64nP)}^{\frac{\lambda_p-1}{2}} \label{eq:lb-epsb}
    \end{align}
    where $t_p=P-p+6$. 
    
    Then we have $\E[H_0]{L_1}=\E[H_j^{(0)}]{L_1}\leq 2r$ for any $j\in[n]$ since the optimal arm of $H_j^{(0)}$ comes at the end and thus each pull during pass $1$ will contribute an expected regret of $\frac{1}{2}$. By our choice of $\eps_1$, $2r\leq \frac{1}{2^{P+2}}\cdot\frac{n-m}{2\eps_1^2 \log(64nP)}$. Let $S_0=\Omega$. From Lemma~\ref{lem:lb-prob-of-Sb}, we can find $S_1\subseteq \Omega$ with property $\+Q_1$ and $\Pr[H_0]{S_1}\geq 1-\frac{1}{4P}-\frac{1}{2^{P+1}}$. Then from Lemma~\ref{lem:lb-bound-E[Lb]}, $\E[H_0]{L_2\mid S_1}\leq \frac{c_1 rn}{(n-m)\eps_1 \Pr[H_0]{S_1}}< \frac{2c_1 rn}{(n-m)\eps_1 }$.
    
    Note that for $p=[P]\setminus\set{1}$, we have $\frac{2c_1 rn}{(n-m)\eps_{p-1}} =\frac{1}{2^{P-p+3}} \cdot \frac{n-m}{2\eps_p^2 \log(64nP)}$.  Assume $\E[H_0]{L_p\mid S_{p-1}}\leq \frac{2c_1 rn}{(n-m)\eps_{p-1}}$ for some $p\geq 2$ where $S_{p-1}\subseteq \Omega$ has property $\+Q_{p-1}$ and satisfies $\Pr[H_0]{S_{p-1}}\geq 1-\frac{p-1}{4P}-\frac{1}{2^{P-p+3}}$. According to Lemma~\ref{lem:lb-prob-of-Sb}, we can find $S_p\subseteq S_{p-1}$ with property $\+Q_p$ and satisfies $\Pr[H_0]{S_{p}}\geq 1-\frac{p}{4P}-\frac{1}{2^{P-p+2}}$. Then from Lemma~\ref{lem:lb-bound-E[Lb]}, $\E[H_0]{L_{p+1}\mid S_{p}}\leq \frac{2c_1 rn}{(n-m)\eps_{p}}$.
    
    Therefore, we can prove by induction that $\E[H_0]{L_{P+1}\mid S_{P}}\leq \frac{2c_1 rn}{(n-m)\eps_{P}}$ and $\Pr[H_0]{S_{P}}\geq \frac{1}{2}$. However, conditioned on $S_P$, we have $L_p\leq \ell_p$ for each $p\in [P]$. This violates the fact $\sum_{p=1}^{P+1} L_p = T$ since
    \begin{align*}
        \E[H_0]{\sum_{p=1}^{P+1} L_p\mid S_{P}}& \leq \sum_{p=1}^{P} \ell_p + \frac{2c_1 rn}{(n-m)\eps_{P}} \\
        &< \sum_{p=1}^{P} \frac{n-m}{2\eps_p^2 \log(64nP)} \\
        &\quad + \frac{1}{\sqrt{2}}n^{\frac{2-2^{P+1}}{2^{P+1}-1} +1 -\lambda_P}\cdot (n-m)^{\frac{2^P-2}{2^{P+1}-1} -\frac{1}{2}+ \frac{3\lambda_P}{2}}\cdot T^{\frac{2^P}{2^{P+1}-1}-\frac{\lambda_P-1}{2}}\cdot \tp{\log 64nP}^{\frac{1-2^P}{2^{P+1}-1}-\frac{\lambda_P-1}{2}}\\
        &\leq\frac{T}{2} + \frac{T}{\sqrt{2}}<T,
    \end{align*}
    where the third inequality holds because $\sum_{p=1}^{P} \frac{n-m}{2\eps_p^2 \log(64nP)}\leq \frac{T}{2}$ when $P\leq \log {\log T} - \log \tp{14\log 8(n-m)}$ and $T$ is sufficiently large.

\end{proof}

\bibliography{arxiv_version}

\begin{thebibliography}{SWXZ22}

\bibitem[AB09]{AB09}
Jean{-}Yves Audibert and S{\'{e}}bastien Bubeck.
\newblock Minimax policies for adversarial and stochastic bandits.
\newblock In {\em Proceedings of the 22nd Conference on Learning Theory (COLT)}, 2009.

\bibitem[AKP22]{AKP22}
Arpit Agarwal, Sanjeev Khanna, and Prathamesh Patil.
\newblock A sharp memory-regret trade-off for multi-pass streaming bandits.
\newblock In {\em Proceedings of the 35th Conference on Learning Theory (COLT)}, 2022.

\bibitem[AW20]{AW20}
Sepehr Assadi and Chen Wang.
\newblock Exploration with limited memory: streaming algorithms for coin tossing, noisy comparisons, and multi-armed bandits.
\newblock In {\em Proceedings of the 52nd Annual ACM SIGACT Symposium on Theory of Computing (STOC)}, 2020.

\bibitem[AW22]{AW22}
Sepehr Assadi and Chen Wang.
\newblock Single-pass streaming lower bounds for multi-armed bandits exploration with instance-sensitive sample complexity.
\newblock In {\em Proceedings of 36th Annual Conference on Neural Information Processing Systems 2021 (NeurIPS)}, 2022.

\bibitem[AW24]{AW24}
Sepehr Assadi and Chen Wang.
\newblock The best arm evades: Near-optimal multi-pass streaming lower bounds for pure exploration in multi-armed bandits.
\newblock In {\em Proceedings of the 37th Conference on Learning Theory (COLT)}, 2024.

\bibitem[CK20]{CK20}
Arghya~Roy Chaudhuri and Shivaram Kalyanakrishnan.
\newblock Regret minimisation in multi-armed bandits using bounded arm memory.
\newblock In {\em Proceedings of the AAAI Conference on Artificial Intelligence (AAAI)}, 2020.

\bibitem[Fel71]{Fe71}
William Feller.
\newblock {\em An introduction to probability theory and its applications, Volume 2, 2nd Edition}.
\newblock Wiley, 1971.

\bibitem[FOP20]{FOP20}
Moein Falahatgar, Alon Orlitsky, and Venkatadheeraj Pichapati.
\newblock Optimal sequential maximization: One interview is enough!
\newblock In {\em Proceedings of the 37th International Conference on Machine Learning (ICML)}, 2020.

\bibitem[FS97]{FS97}
Yoav Freund and Robert~E Schapire.
\newblock A decision-theoretic generalization of on-line learning and an application to boosting.
\newblock {\em Journal of computer and system sciences}, 55:119--139, 1997.

\bibitem[JHTX21]{JHTX21}
Tianyuan Jin, Keke Huang, Jing Tang, and Xiaokui Xiao.
\newblock Optimal streaming algorithms for multi-armed bandits.
\newblock In {\em Proceedings of the 38th International Conference on Machine Learning (ICML)}, 2021.

\bibitem[LG21]{LG21}
Tor Lattimore and Andras Gyorgy.
\newblock Mirror descent and the information ratio.
\newblock In {\em Proceedings of the 34th Conference on Learning Theory (COLT)}, 2021.

\bibitem[LS20]{LS20}
Tor Lattimore and Csaba Szepesv{\'a}ri.
\newblock {\em Bandit algorithms}.
\newblock Cambridge University Press, 2020.

\bibitem[LSPY18]{LSPY18}
David Liau, Zhao Song, Eric Price, and Ger Yang.
\newblock Stochastic multi-armed bandits in constant space.
\newblock In {\em Proceedings of the 21st International Conference on Artificial Intelligence and Statistics (AISTATS)}, 2018.

\bibitem[MPK21]{MPK21}
Arnab Maiti, Vishakha Patil, and Arindam Khan.
\newblock Multi-armed bandits with bounded arm-memory: Near-optimal guarantees for best-arm identification and regret minimization.
\newblock In {\em Proceedings of the 35th Annual Conference on Neural Information Processing Systems 2021 (NeurIPS)}, 2021.

\bibitem[MT04]{MT04}
Shie Mannor and John~N Tsitsiklis.
\newblock The sample complexity of exploration in the multi-armed bandit problem.
\newblock {\em Journal of Machine Learning Research}, 5:623--648, 2004.

\bibitem[PR23]{PR23}
Binghui Peng and Aviad Rubinstein.
\newblock Near optimal memory-regret tradeoff for online learning.
\newblock In {\em Proceedings of the 64th Annual Symposium on Foundations of Computer Science (FOCS)}, 2023.

\bibitem[PZ23]{PZ23}
Binghui Peng and Fred Zhang.
\newblock Online prediction in sub-linear space.
\newblock In {\em Proceedings of the 2023 Annual ACM-SIAM Symposium on Discrete Algorithms (SODA)}, 2023.

\bibitem[Rat21]{Rat21}
Santanu Rathod.
\newblock On reducing the order of arm-passes bandit streaming algorithms under memory bottleneck.
\newblock {\em arXiv preprint arXiv:2112.06130}, 2021.

\bibitem[Rob52]{Rob52}
Herbert Robbins.
\newblock {Some aspects of the sequential design of experiments}.
\newblock {\em Bulletin of the American Mathematical Society}, 58:527 -- 535, 1952.

\bibitem[SWXZ22]{SWXZ22}
Vaidehi Srinivas, David~P Woodruff, Ziyu Xu, and Samson Zhou.
\newblock Memory bounds for the experts problem.
\newblock In {\em Proceedings of the 54th Annual ACM SIGACT Symposium on Theory of Computing (STOC)}, 2022.

\bibitem[Wan23]{Wang23}
Chen Wang.
\newblock Tight regret bounds for single-pass streaming multi-armed bandits.
\newblock In {\em Proceedings of the 40th International Conference on Machine Learning (ICML)}, 2023.

\bibitem[WZZ23]{WZZ23}
David~P Woodruff, Fred Zhang, and Samson Zhou.
\newblock Streaming algorithms for learning with experts: Deterministic versus robust.
\newblock {\em arXiv preprint arXiv:2303.01709}, 2023.

\bibitem[XZ21]{XZ21}
Xiao Xu and Qing Zhao.
\newblock Memory-constrained no-regret learning in adversarial multi-armed bandits.
\newblock {\em {IEEE} Trans. Signal Process.}, 69:2371--2382, 2021.

\bibitem[Yao77]{Yao77}
Andrew Chi-Chin Yao.
\newblock Probabilistic computations: Toward a unified measure of complexity.
\newblock In {\em Proceedings of the 18th Annual Symposium on Foundations of Computer Science (FOCS)}, 1977.

\end{thebibliography}
\bibliographystyle{alpha}

\appendix

\section{The BAI Algorithm Used in \Cref{sec:ub-s}}\label{sec:ub-s-BAI}
In \Cref{sec:ub-s}, an $(\eps,\delta)$-PAC streaming algorithm, denoted as $\-{BAI}(\eps,\delta)$, is used as the main building block for the streaming MAB algorithm. Let $\+S\subseteq[n]$ be the set of arms that are input into $\-{BAI}(\eps,\delta)$ and let $\a^*_{\+S}$ be the optimal arm in $\+S$. We need to prove the following proposition.
\begin{proposition}[\Cref{prop:BAIblackbox} restated]
    There exists a $(\eps,\delta)$-PAC streaming algorithm $\-{BAI}(\eps,\delta)$ with input arm set $\+S$ satisfying:
    \begin{enumerate}
        \item \textbf{Correctness}: $\-{BAI}(\eps,\delta)$ returns an arm $\k\in \+S$ such that $\E{\mu_{\a^*_{\+S}} - \mu_{\k}}\leq \eps$.
        \item \textbf{Storage}: $\-{BAI}(\eps,\delta)$ uses a memory size of at most $m$. Furthermore, there is one position among $m$ memory slots specifically allocated for storing the newly arrived arm. 
        \item \textbf{Benchmark arm}: $\-{BAI}(\eps,\delta)$ maintains a good arm $\a^*_{\!{max}}$ in memory during its process and $\a^*_{\!{max}}$ always satisfies $\E{\mu_{k_1}-\mu_{\a^*_{\!{max}}}}\leq \eps$, where $k_i$ is the $i$-th arm fed into $\-{BAI}(\eps,\delta)$.
        \item \textbf{Regret guarantee}: When $\+S$ is a random set and for every $i$, $\mu-\E{\mu_{k_i}}\leq \eps'$ for some fixed numbers $\mu\in(0,1),\eps'\in(\eps,1)$, 
        the expected regret generated by the $\-{BAI}(\eps,\delta)$ process with regard to an arm with mean $\mu$ is bounded by $O\tp{\frac{n\eps'}{\eps^2}\tp{\log\tp{\frac{1}{\delta}} + \-{ilog}^{(m-1)}(n)}}$. (Here the expectation includes the randomness of $\+S$.)
    \end{enumerate}
\end{proposition}






We will show that the multi-level $(\eps,\delta)$-PAC algorithm designed in~\cite{AW20} and~\cite{MPK21} satisfies all these four properties. Recall that $\-{ilog}^{(k)}(a)=\max\set{\log\tp{\-{ilog}^{(k-1)}(a)}, 1}$ for any $a\geq 1$, $k\in \bb N^+$. This algorithm can find an $\eps$-best arm among $n$ input arms with probability $1-\delta$ by storing at most $m$ arms $(2\leq m\leq \log^* n+1)$ using $O\tp{\frac{n}{\eps^2}\tp{\log \tp{\frac{1}{\delta}} +\-{ilog}^{(m-1)}(n)}}$ samples.

\subsection{The algorithm}
\paragraph{Overview of the algorithm} 

The algorithm maintains at most $r=m-1$ arms, $a^*_1,\dots,a^*_{r}$. For each $\ell\in [r]$, during the execution, $a_\ell^*$ is the current best arm of level $\ell$. For each arriving $\!{arm}_i$, it starts from the first level and tries to challenge $a_\ell^*$ for increasing $\ell=1,2,\dots$. Basically, the challenge proceeds as follows. At level $\ell$, the algorithm samples the challenger arm for $s_\ell$ times and compares the empirical mean with that of the current $a_\ell^*$. If the challenger loses, it will be dropped and the algorithm reads the next arm from the stream. Otherwise, it substitutes $a_\ell^*$.

Once the challenger finishes the challenge (either being dropped or substituting the previous $a_\ell^*$), a counter $C_\ell$ increases by one. As long as the counter meets a threshold $c_\ell$, the best arm $a_\ell^*$ will level up and start to challenge $a_{\ell+1}^*$, in a way similar to level $\ell$. After this, all our previous statistics at level $\ell$ will be reset. The net effect is that, for each $\ell$, every time a group of $c_\ell$ challengers arrives at the $\ell$-th level, the algorithm sends the winner of these $c_\ell$ arms to level $\ell+1$. Finally, after processing the whole stream, the algorithm outputs the best of $a_\ell^*$ for $\ell\in [r]$.

\paragraph{The description of the algorithm}

Let $\eps\in (0,1)$ and $\delta\in (0, \frac{1}{4}]$. Let $r=m-1$ where $m$ is the memory size. We assume $2\leq m\leq \log^*n +1$. For the convenience of designing streaming MAB algorithm, we allow the input arm set of the BAI algorithm to be some $\+S\subseteq[n]$. Let $n'=\abs{\+S}$ be the total number of arms in the stream the BAI algorithm needs to handle. It is guaranteed that $n'\leq n$, but $n'$ itself is not known beforehand. Actually, when executing the BAI subroutine in our streaming MAB algorithm, $\+S$ can be a random set. We would show that property $1$, $2$ and $3$ in \Cref{prop:BAIblackbox} holds whatever $\+S$ is and therefore, without loss of generality, we regard $\+S$ as deterministic and assume $\+S=[n']$ when analyzing these properties. We only consider the randomness of $\+S$ when proving the fourth property.

Set the following parameters:
\begin{itemize}
    \item $\set{c_{\ell}}_{\ell= 1}^r:  c_{\ell}=\lceil \-{ilog}^{(r-\ell)}(n) \rceil$, which is the number of arms needed to arrive at each level before a promotion;
    \item $\set{s_{\ell}}_{\ell= 1}^r: s_{\ell} = \lceil \frac{2^{2\ell+3}}{\eps^2}\tp{\log \frac{2^{\ell+2}c_{\ell}}{\delta}} \rceil$, which is the sample times per arm at each level.
\end{itemize}
The algorithm is described in \Cref{algo:epsBAI}. Obviously, property $2$ of \Cref{prop:BAIblackbox} is satisfied in \Cref{algo:epsBAI}.

\begin{algorithm}[ht]
    \caption{Find an $\eps$-best arm in a stream $\a_1,\a_2,\dots,\a_{n'}$}
    \label{algo:epsBAI}
\begin{algorithmic}[1]
    \Procedure{BAI}{$\eps,\delta$}
    \ForAll{$\ell\in [r]$}
        \State $C_\ell\gets 0$, $\hat\mu_\ell^*\gets 0$, $a_\ell^*\gets \*{NULL}$ \Comment{Initialize}
    \EndFor
    \For{each arriving $\a_i$}
        \State $\a^o\gets \a_i$, $\ell\gets 1$
        \While{$\ell\le r$}
            \State sample $\a^o$ for $s_\ell$ times and compute its empirical mean $\hat\mu_{\a^o}$
        \If{$\hat\mu_{\a^o}<\hat\mu_\ell^*$}
            \State drop $\a^o$ from memory
        \Else
            \State $a_\ell^*\gets \a^o$, $\hat\mu_\ell^*\gets \hat\mu_{\a^o}$
        \EndIf
        \State $C_\ell\gets C_\ell+1$
        \If{$C_\ell = c_\ell$}
            \State $\a^o\gets a_\ell^*$, $C_\ell=0$, $\hat\mu_\ell^*\gets 0$, $a_\ell^*\gets \*{NULL}$
            \State $\ell\gets \ell+1$
        \Else{ \textbf{break}}
        \EndIf
        \EndWhile
    \EndFor
    \State sample each non-NULL $a^*_{\ell}$ ($\ell\in [r]$) for $s_r$ times and compute new $\hat \mu^*_{\ell}$\label{line: tail in}
    \State $\k\gets \mbox{the arm with highest $\hat\mu^*_{\ell}$ for $\ell\in[r]$}$
        \State \Return{$\k$}
    \EndProcedure
\end{algorithmic}
\end{algorithm}

\subsection{The analysis}
\subsubsection{Property $1$ and $3$ of \Cref{prop:BAIblackbox}}
In the following discussion, we fix a time during the execution of \Cref{algo:epsBAI} and consider a level $L\le r$ such that $a_L^*\ne\*{NULL}$. It follows from our algorithm that $a_L^*$ has ``competed with'' $C_L$ arms (counting itself) at the same level. We use $A_L$ to denote this set of $C_L$ arms and let $\a^*_{A_L}\defeq \argmax_{a\in A_L} \mu_{a}$, namely the arm in $A_L$ with the largest mean. The next lemma states that with high probability the mean $\mu_{a_L^*}$ is large comparing to $\mu_{\a^*_{A_L}}$. This indicates that the best arm considered by the algorithm is comparable with the actual best arm.

\begin{lemma}\label{lem:prob-armL}
    For any $\Delta>0$, it holds that
    \(
        \Pr{\mu_{a_L^*}\le \mu_{\a^*_{A_L}} -\Delta} \le c_L\cdot \exp{-\frac{s_L\Delta^2}{2}}.
    \)
\end{lemma}
\begin{proof}
    Consider the last duration that $C_L$ increases from $0$ to its current value. Our algorithm effectively does the following: It samples each arm in $A_L$ for $s_L$ times and let $a_L^*$ be the one with the largest empirical mean. As a result, if the event ``$\mu_{a_L^*}\le \mu_{\a^*_{A_L}} -\Delta$'' happens, then there must be some arm $a$ with $\mu_a \le \mu_{\a^*_{A_L}}-\Delta$ who beats $\a^*_{A_L}$. This implies
    \begin{align*}
        \Pr{\mu_{a_L^*}\le \mu_{\a^*_{A_L}} -\Delta} 
        &\le \Pr{\exists a\in A_L\colon \tp{\mu_a \le \mu_{\a^*_{A_L}}-\Delta} \land \tp{\hat\mu_a >\hat\mu_{\a^*_{A_L}}}}\\
        &\le \sum_{a\in A_L}\Pr{\hat\mu_a >\hat\mu_{\a^*_{A_L}}\mid \mu_a\le \mu_{\a^*_{A_L}}-\Delta}.
    \end{align*}
    The lemma then follows from the Hoeffding's inequality and the fact that $\abs{A_L} = C_L\le c_L$
\end{proof}

Lemma~\ref{lem:prob-armL} captures the smooth failure probability of the algorithm. This property was initially mentioned in~\cite{Wang23}. Actually, the essential aim of a good MAB algorithm to obtain an $\eps$-optimal arm in expectation after the stream ends. This is also what the smooth failure probability property is to fulfil in~\cite{Wang23}. The following technical lemma will help us convert high probability results into expected results we want. The proof of \Cref{lem:prob-to-exp} is deferred to \Cref{subsec: pf-ub-prob2exp}.

\begin{lemma}\label{lem:prob-to-exp}
    Let $X$ be a real-valued random variable. For any $L\in [r]$, if it holds that $\Pr{X\ge t} \le c_L\cdot\exp{-\frac{s_L\cdot t^2}{2}}$ for every $t\ge 0$, then $\E{X}\le \eps\cdot 2^{-L}$.
\end{lemma}

Recall that $\a^*_{A_L}\defeq \argmax_{a\in A_L} \mu_{a}$ is the arm in $A_L$ with the largest mean. With Lemma~\ref{lem:prob-to-exp}, we can convert the result of Lemma~\ref{lem:prob-armL} into a bound of expected gap between $\a^*_{A_L}$ and $a_L^*$.
\begin{corollary}\label{cor:expectation-armL}
    \(
        \E{\mu_{\a^*_{A_L}}-\mu_{a_L^*}} \le \eps\cdot 2^{-L}.
    \)
\end{corollary}

Each arm $a\in A_L$ was once $a_{L-1}^*$ and therefore defeated $c_{L-1}$ arms at the $L-1$ level. As a result, it has been effectively compared with the empirical mean of consecutive $\prod_{\ell=1}^{L-1} c_\ell$ arms in the stream (counting itself). The arm $a_L^*$ therefore has been effectively compared with $C_L\cdot \prod_{\ell=1}^{L-1} c_\ell$ distinct arms. We use $D_L$ to denote this set of arms. Then we can obtain the following lemma, which guarantees that the mean reward of $a^*_L$ is expected to be good compared to the best arm in $D_L$.

\begin{lemma}\label{lem:expectation-all-level}
    Let $\a^*_{D_L} = \arg\max_{a\in D_L} \mu_a$. Then $\E{\mu_{\a^*_{D_L}}-\mu_{a_L^*}}\le \eps$.
\end{lemma}
\begin{proof}
    Applying Corollary~\ref{cor:expectation-armL} repeatedly for every $\ell\in [L]$, we obtain
    \(
        \E{\mu_{\a^*_{D_L}}-\mu_{a_L^*}}\le \sum_{\ell\in [L]}\eps\cdot 2^{-\ell}\le \eps.
    \)
\end{proof}

At any time, let $\ell_{\!{max}}=\max\set{\ell\in[r]: a^*_\ell\neq \*{NULL}}$ be the highest level $\ell$ such that $a_\ell^*\ne\*{NULL}$. Let $\mu_1$ be the mean of the first arm in the stream. Clearly the first arm belongs to $D_{\ell_{\!{max}}}$. The following corollary is immediate from Lemma~\ref{lem:expectation-all-level}. It bounds the gap between the mean of $a^*_{\ell_{\!{max}}}$ and the first arm in stream.

\begin{corollary}[Property $3$ of \Cref{prop:BAIblackbox}]\label{cor:BAIcorrect}
     $\E{ \mu_1 - \mu_{a^*_{\ell_{\!{max}}}} }\leq \eps$ holds at any point during the execution of \Cref{algo:epsBAI}.
\end{corollary}
Choosing $\mu_{a^*_{\ell_{\!{max}}}}$ as the $\mu_{\a^*_{\!{max}}}$, \Cref{cor:BAIcorrect} proves the third property of \Cref{prop:BAIblackbox} for \Cref{algo:epsBAI}.

Equipped with above lemmas, we are then ready to prove that the \Cref{algo:epsBAI} always returns an $\eps$-best arm \emph{in expectation}. This indicates \Cref{algo:epsBAI} satisfies the third property we require. Let $\a^*$ be the optimal arm in stream.

\begin{theorem}[Property $1$ of \Cref{prop:BAIblackbox}]\label{thm:BAIcorrect}
    The returned arm $\k$ of \Cref{algo:epsBAI} satisfies
    \[
        \E{\mu_{\a^*}-\mu_{\k}}\leq \eps.
    \]
\end{theorem}
\begin{proof}
    At the end of the algorithm, the optimal arm $\a^*$ must belong to the set $D_{\ell'}$ for some $a_{\ell'}^*$. We know from Corollary~\ref{cor:expectation-armL} that $\E{\mu_{\a^*}-\mu_{a_{\ell'}^*}}\le \sum_{j\in[\ell']}\eps\cdot 2^{-j}\le \sum_{j\in[r]}\eps\cdot 2^{-j}$.

    On the other hand, our algorithm samples each non-NULL $a_\ell^*$ for $\ell\in [r]$ for $s_r$ times. Since $c_r = n$, by an argument similar to that in the proof of Lemma~\ref{lem:prob-armL}, for every $t>0$, $\Pr{\mu_{a_{\ell'}^*} - \mu_\k\ge t}\le c_r\cdot \exp{-\frac{s_r\cdot t^2}{2}}$. So by Lemma~\ref{lem:prob-to-exp}, $\E{\mu_{a_{\ell'}^*} - \mu_\k}\le \eps\cdot 2^{-r}$. The theorem then follows.
\end{proof}

\subsubsection{Property $4$ of \Cref{prop:BAIblackbox}}
It remains to analyze the regret of the BAI algorithm. We first consider the sample complexity of \Cref{algo:epsBAI}. Note that the sample complexity only depends on the number of total arms $n'$. We can prove the following sample complexity bound of \Cref{algo:epsBAI} for any $n\ge n'$.
\begin{theorem}[sample complexity of \Cref{algo:epsBAI}]\label{thm:sample_complexity}
    For any integer $1\leq r\leq \log^*n$, the worst case sample complexity of \Cref{algo:epsBAI} is no larger than $\sum_{\ell=1}^r s_{\ell} \cdot \big\lfloor\frac{n}{\prod_{j=1}^{\ell-1}c_j} \big\rfloor + r s_{r}$. This can be further bounded by $O\tp{\frac{n}{\eps^2}\tp{\log \tp{\frac{1}{\delta}} +\-{ilog}^{(r)}(n)}}$.
\end{theorem}
\begin{proof}
    Note that there are at most $\frac{n}{\prod_{i=1}^{\ell-1}c_{i}}$ arms that reaches level $\ell$. Then the total sample complexity is no larger than
    \[
        \sum_{\ell=1}^r \frac{n\cdot s_{\ell}}{\prod_{i=1}^{\ell-1}c_{i}} + r\cdot s_r,
    \]
    for any $n'\leq n$.

    Recall that $1\leq r\leq \log^* n$, $c_\ell = \lceil \-{ilog}^{(r-\ell)}(n) \rceil$ and $s_{\ell} = \ceil{ \frac{2^{2\ell+3}}{\eps^2}\tp{\log \frac{2^{\ell+2}c_{\ell}}{\delta}} }$. When $\ell=1$, we have
    \[
        n\cdot s_{1} =n\cdot \ceil{\frac{2^{5}}{\eps^2}\tp{\log \frac{2^{3}c_{1}}{{\delta}} }} = O\tp{\frac{n}{\eps^2}\tp{\log \tp{\frac{1}{\delta}} + \log c_1}} = O\tp{\frac{n}{\eps^2}\tp{\log \tp{\frac{1}{\delta}} + \-{ilog}^{(r)}(n)}}.
    \]
    For the term $r\cdot s_r$, we can bound it by 
    \[
        r\cdot s_r = O\tp{\frac{r\cdot 2^{2r+3}}{\eps^2}\log \tp{\frac{2^{r+2}n}{\delta}}}\leq O\tp{\frac{n}{\eps^2}\log\tp{\frac{1}{\delta}}}.
    \]
    Then it remains to deal with $ \sum_{\ell=2}^r \frac{n\cdot s_{\ell}}{\prod_{i=1}^{\ell-1}c_{i}}$. Let $c_0=1$. We have
    \begin{align*}
        \sum_{\ell=2}^r \frac{n\cdot s_{\ell}}{\prod_{i=1}^{\ell-1}c_{i}} &\leq 2\sum_{\ell=2}^r \frac{n}{\prod_{i=1}^{\ell-1}c_{i}}\cdot \frac{2^{2\ell+3}}{\eps^2}\tp{\log \frac{2^{\ell+2}}{\delta} + \log \tp{\-{ilog}^{(r-\ell)}(n)}}\\
        &\leq 2\sum_{\ell=2}^r \frac{n}{c_{\ell-1}c_{\ell-2}}\cdot \frac{2^{2\ell+3}}{\eps^2}\tp{\log \frac{2^{\ell+2}}{\delta} + c_{\ell-1}}\\
        &\leq 2\sum_{\ell=2}^r \frac{n}{c_{\ell-2}}\cdot \frac{2^{2\ell+3}}{\eps^2} + 2\sum_{\ell=2}^r \frac{n}{c_{\ell-1}c_{\ell-2}}\cdot \frac{2^{2\ell+3}}{\eps^2} \log \frac{2^{2\ell}}{\delta^{2\ell}}\\
        &\leq \frac{4n}{\eps^2}\sum_{\ell=2}^r \frac{4^{\ell+1}}{\-{ilog}^{(r-\ell+2)}(n)} + \frac{8n}{\eps^2}\log \frac{2}{\delta} \sum_{\ell=2}^r \frac{4^{\ell+2}}{\-{ilog}^{(r-\ell+1)}(n)}
    \end{align*}
    where the third inequality is due to $\frac{2^{\ell+2}}{\delta}\leq \tp{\frac{2}{\delta}}^{2\ell}$ since $\delta\leq \frac{1}{4}$ and the last inequality is due to $\frac{\ell}{c_{\ell-2}}\leq 4$. From direct calculation, we know that $\sum_{\ell=2}^{r+1} \frac{4^{\ell+1}}{\-{ilog}^{(r-\ell+2)}(n)}=O(1)$. Therefore, combining above results, the sample complexity is no larger than
    \[
        \sum_{\ell=1}^r \frac{n\cdot s_{\ell}}{\prod_{i=1}^{\ell-1}c_{i}} + r\cdot s_r = O\tp{\frac{n}{\eps^2}\tp{\log \tp{\frac{1}{\delta}} + \-{ilog}^{(r)}(n)}}.
    \]
\end{proof}

We then prove the fourth property we require in the follow lemma. Note that we actually allow the input set $\+S\subseteq [n]$ of BAI algorithm to be a random set. We define $k_i\in[n]$ to be the $i$-th arriving arm in \Cref{algo:epsBAI}. The following lemma states that if the quality of each input arm is guaranteed, then the total regret of the BAI process can be bounded.
\begin{theorem}[Property $4$ of \Cref{prop:BAIblackbox}]
    When $\+S$ is a random set and for every $i$, $\mu-\E{\mu_{k_i}}\leq \eps'$ for some fixed numbers $\mu\in(0,1),\eps'\in(\eps,1)$, 
    the expected regret generated by the $\-{BAI}(\eps,\delta)$ process with regard to an arm with mean $\mu$ is bounded by $O\tp{\frac{n\eps'}{\eps^2}\tp{\log\tp{\frac{1}{\delta}} + \-{ilog}^{(m-1)}(n)}}$. (Here the expectation includes the randomness of $\+S$.)

\end{theorem}
\begin{proof}
        Denoted the expected regret with regard to an arm with mean $\mu$ generated by \Cref{algo:epsBAI} as $R_{\textsc{BAI}}$. For $\ell\in [r]$, let $U_\ell$ be the set of arms that arrives at level $\ell$ and let $U_{r+1}$ be the set of arms sampled in Line~\ref{line: tail in} of \Cref{algo:epsBAI}. We know that $\abs{U_{\ell}}\leq \big\lfloor\frac{n}{\prod_{j=1}^{\ell-1}c_j} \big\rfloor$ and $\abs{U_{r+1}}\leq r$. Then we have
        \begin{align*}
            R_{\textsc{BAI}} &= \sum_{\ell=1}^r \E{\sum_{a\in U_{\ell}} (\mu-\mu_{a})s_{\ell}} + \E{\sum_{a\in U_{r+1}} (\mu-\mu_{a})s_{r}}\\
            &= \sum_{\ell=1}^r s_{\ell} \E{\sum_{a\in U_{\ell}} (\mu-\mu_{a})} + s_{r}\E{\sum_{a\in U_{r+1}} (\mu-\mu_{a})}.
        \end{align*}
    
        For each $a^*_{\ell}$, from Lemma~\ref{lem:expectation-all-level}, we have $\E{\mu_{\a^*_{D_\ell}}-\mu_{a_\ell^*}}\le \eps$ where $\a^*_{D_\ell} = \arg\max_{a\in D_\ell} \mu_a$. From our assumption, for each $a\in D_{\ell}$, $\E{\mu - \mu_a}\leq \eps'$ and thus $\E{\mu - \mu_{a_\ell^*}}\leq \eps' + \eps\leq 2\eps'$.
    
        Note that if $\abs{U_{\ell}}<\big\lfloor\frac{n}{\prod_{j=1}^{\ell-1}c_j} \big\rfloor$, we can construct a set $U'_{\ell}$ by adding some virtual arms $a'$ with $\mu_{a'}=\mu$ into $U_{\ell}$ such that $\abs{U'_{\ell}} = \big\lfloor\frac{n}{\prod_{j=1}^{\ell-1}c_j} \big\rfloor$ and $U_{\ell}\subseteq U'_{\ell}$ always holds. Then we have $\sum_{a\in U_{\ell}} (\mu-\mu_{a})= \sum_{a'\in U'_{\ell}} (\mu-\mu_{a'})$ for all $\ell\in[r+1]$. Since each $a\in U_{\ell}$ used to be $a^*_{\ell-1}$ (if $\ell>1$), we know that $\E{\mu - \mu_a}\leq 2\eps'$. Therefore, for each $a'\in U'_\ell$, we also have $\E{\mu - \mu_{a'}}\leq 2\eps'$. Then from the linearity of expectation,
        \begin{align*}
            R_{\textsc{BAI}} &= \sum_{\ell=1}^r s_{\ell} \E{\sum_{a'\in U'_{\ell}} (\mu-\mu_{a'})} + s_{r}\E{\sum_{a'\in U'_{r+1}} (\mu-\mu_{a'})} \\
            &\leq \sum_{\ell=1}^r s_{\ell} \cdot \left\lfloor\frac{n}{\prod_{j=1}^{\ell-1}c_j} \right\rfloor \cdot 2\eps' + r s_{r}\cdot 2\eps'.
        \end{align*}

    From \Cref{thm:sample_complexity}, the worst case sample complexity is bounded by
    \[
        \sum_{\ell=1}^r s_{\ell} \cdot \left\lfloor\frac{n}{\prod_{j=1}^{\ell-1}c_j} \right\rfloor + r s_{r} = O\tp{\frac{n}{\eps^2}\tp{\log \tp{\frac{1}{\delta}} +\-{ilog}^{(m-1)}(n)}}.
    \]
    So the total regret generated by the BAI process is 
    \begin{align*}
        R_{\textsc{BAI}}&\leq O\tp{\frac{n}{\eps^2}\tp{\log \tp{\frac{1}{\delta}} +\-{ilog}^{(m-1)}(n)}}\cdot 2\eps'
    \end{align*}
\end{proof}

\section{Technical Proofs}\label{sec: proofs}
\subsection{Proof of Lemma~\ref{lem:lb-concentration}} \label{subsec: pf-lb-concentration}
Lemma~\ref{lem:lb-concentration} is a generalization of Hoeffding's inequality. For completeness and consistency of the paper, we provide its proof here.

\begin{lemma}[Hoeffding's Lemma]\label{lem: hoeffding lemma}
    Let $X$ be a random variable with mean $0$ and $a\leq X\leq b$ holds almost surely. Then for any $\alpha\in \bb R$, $\E{e^{\alpha X}}\leq e^{\frac{\alpha^2(b-a)^2}{8}}$.
\end{lemma}

\begin{lemma}[Lemma~\ref{lem:lb-concentration} restated]
    Let $X_1,\dots,X_N$ be $N$ independent random variables defined on a common probability space and taking values in $[a,b]$. Assume $\E{X_t}=0$ for any $t\in[N]$. Then for any $s>0$,
    \[
        \Pr{\max_{1\leq t\leq N} \abs{\sum_{j=1}^t X_j} \geq s}\leq 2\exp{-\frac{2s^2}{N(b-a)^2}}.
    \]
\end{lemma}
\begin{proof}
    Let $Z_0=1$ and $Z_t = e^{\alpha \sum_{j=1}^t X_j}$ for $t\in[N]$, where $\alpha$ is a positive real number to be determined later. Let $\+F_t=\sigma\tp{X_1,\dots,X_t}$. From Jensen's inequality, for any $t\in[N]$
    \[
        \E{Z_{t}\mid \+F_{t-1}} = Z_{t-1}\E{e^{\alpha X_{t}}\mid \+F_{t-1}} = Z_{t-1}\E{e^{\alpha X_{t}}} \geq Z_{t-1} e^{\alpha \E{X_{t}}} =Z_{t-1}.
    \]
    Therefore, $\set{Z_t}_{t\in[N]}$ is a submartingale with regard to $\set{\+F_t}_{t\in[N]}$. From the Doob's submartingale inequality,
    \begin{align*}
        \Pr{\max_{1\leq t\leq N} Z_t \geq e^{\alpha s}} & \leq e^{-\alpha s}\E{Z_N} = e^{-\alpha s}\prod_{t=1}^N \E{e^{\alpha X_t}} .
    \end{align*}

    According to the Hoeffding's lemma (Lemma~\ref{lem: hoeffding lemma}), $\E{e^{\alpha X_t}}\leq e^{\frac{\alpha^2(b-a)^2}{8}}$. Choosing $\alpha = \frac{4s}{N(b-a)^2}$, we have
    \begin{equation}
        \Pr{\max_{1\leq t\leq N} \sum_{j=1}^t X_j \geq s} = \Pr{\max_{1\leq t\leq N} Z_t \geq e^{\alpha s}}\leq e^{-\alpha s + \frac{N\alpha^2(b-a)^2}{8}}\leq e^{-\frac{2s^2}{N(b-a)^2}}. \label{eq:concentration-1}
    \end{equation}
    We can similarly derive that 
    \begin{equation}
        \Pr{\max_{1\leq t\leq N} -\sum_{j=1}^t X_j \geq s} \leq e^{-\frac{2s^2}{N(b-a)^2}}. \label{eq:concentration-2}
    \end{equation}
    Combining \Cref{eq:concentration-1,eq:concentration-2},  
    \begin{align*}
        \Pr{\max_{1\leq t\leq N} \abs{\sum_{j=1}^t X_j} \geq s} &\leq  \Pr{\max_{1\leq t\leq N} \sum_{j=1}^t X_j \geq s} + \Pr{\max_{1\leq t\leq N} -\sum_{j=1}^t X_j \geq s} \leq 2e^{-\frac{2s^2}{N(b-a)^2}}.
    \end{align*}
\end{proof}

\subsection{Proof of Lemma~\ref{lem:prob-to-exp}}\label{subsec: pf-ub-prob2exp}
\begin{lemma}[\Cref{lem:prob-to-exp} restated]
    Let $X$ be a real-valued random variable. For any $L\in [r]$, if it holds that $\Pr{X\ge t} \le c_L\cdot\exp{-\frac{s_L\cdot t^2}{2}}$ for every $t\ge 0$, then $\E{X}\le \eps\cdot 2^{-L}$.
\end{lemma}
\begin{proof}
Note that for any real-valued random variable $X$, letting $X^+=\max\set{X,0}$ be the positive part of $X$, we have the elementary inequality (see e.g.,~\cite[Section \uppercase\expandafter{\romannumeral5}.6]{Fe71}) 
\[
    \E{X}\le \E{X^+} = \int_{0}^\infty\Pr{X^+\ge t} \dd{t} = \int_{0}^\infty\Pr{X\ge t}\dd{t}.
\]
It then follows that
\begin{align*}
    \E{X}\le \int_{0}^\infty \Pr{X\ge t}\dd{t} 
    = \int_{0}^{\frac{\eps}{2^{L+1}}} \Pr{X\ge t}\dd{t}+\int_{\frac{\eps}{2^{L+1}}}^\infty \Pr{X\ge t}\dd{t}.
\end{align*}
Now we bound the two integrals respectively. For the first term, we use the trivial bound that $\Pr{X\ge t}\le 1$. So 
\[
    \int_{0}^{\frac{\eps}{2^{L+1}}} \Pr{X\ge t}\dd{t} \le \eps\cdot 2^{-(L+1)}.
\]
For the second term, observe that
\[
    \Pr{X\ge t}\le c_L\cdot \exp{-\frac{s_L\cdot t^2}{2}}\le c_L \tp{\frac{1}{2^{L+4}c_L}}^{\frac{2^{2L+2}t^2}{\eps^2}}\le 2^{-(L+4)2^{2L+2}(t/\eps)^2}.
\]
Therefore, we have 
\[
    \int_{\frac{\eps}{2^{L+1}}}^\infty \Pr{X\ge t}\dd{t} \le \int_0^\infty 2^{-(L+4)2^{2L+2}(t/\eps)^2} \dd{t}\le \eps\cdot 2^{-(L+1)}. 
\]
\end{proof}

\section{Details of the OSMD Algorithm Pertaining to Proposition~\ref{prop:ub-l-OSMD}}\label{sec:OSMD-detail}
For completeness, we provide the description of the OSMD algorithm we used in \Cref{algo:large-m} and \Cref{algo:small-m} here. For more detailed information, please refer to the work of~\cite{LG21}. 

Let $\Delta_{(n-1)}=\set{\*q\in \bb R_{\geq 0}: \sum_{i=1}^n \*q(i)=1}$ be the probability simplex with $n-1$ dimension. For a vector $\*q$ in $\Delta_{(n-1)}$, $\*q(i)$ denotes the value at its $i$-th position. Consider a function $F:\bb R^n\to \bb R\cup\set{\infty}$. The Bregman divergence with respect to $F$ is defined as $B_F(\*q,\*p) = F(\*q)-F(\*p) - \inner{\nabla F(\*p)}{\*q-\*p}$ for any $\*q,\*p\in \bb R^n$. 

The algorithm in~\cite{LG21} is designed for loss cases. That is, one pull gives a loss of the corresponding arm rather than a reward. To fit their algorithm into our setting, we can do a simple reduction by constructing the loss of each arm $\ell_t(i)$ as $1-r_t(i)$ where $r_t(i)$ is the reward of $\a_i$. It is easy to verify that the results in~\cite{LG21} also holds for the reward setting. Let $\eta$ be the learning rate and $F:\bb R^{\abs{\+S}}\to \bb R\cup\set{\infty}$ be the potential function, where $\+S$ is the arm set. Without loss of generality, we index the arms in $\+S$ by $[\abs{\+S}]$.
\begin{algorithm}[h]
    \caption{Online Stochastic Mirror Descent(\cite{LG21})}
    \label{algo:OSMD}
    \Input{a set of arms $\+S$ and the number of rounds $T$}
\begin{algorithmic}[1]
    \Procedure{\textsc{MirrorDescent}}{{$\+S,T$}}
    \State $Q_1\gets \arg\min_{\*q\in \Delta_{(\abs{\+S}-1)}} F(\*q)$\;
    \For{$t=1,2\dots, T$}
        \State Sample arm $A_t\sim Q_t$, observe reward $r_t(A_t)$ and let $\ell_t(A_t) = 1-r_t(A_t)$\;
        \State Compute reward estimator $\hat \ell_t$ as 
                \[
                    \hat \ell_t(i) = \1{A_t=i}\tp{\ell_t(i) -\frac{1}{2} + \frac{\eta}{8}\tp{1+\frac{1}{Q_t(i) + \sqrt{Q_t(i)}}}} - \frac{\eta Q_t(A_t)}{8\tp{ Q_t(i) + \sqrt{Q_t(i)} }}
                \]
        \State Set $Q_{t+1} = \arg\min_{\*q\in \Delta_{(\abs{\+S}-1)}} \inner{\*q}{\hat\ell_t} + \frac{1}{\eta}\cdot B_F(\*q,Q_t)$\;
    \EndFor
    \EndProcedure
\end{algorithmic}
\end{algorithm}


By choosing $\eta = \sqrt{\frac{8}{L}}$ and $F(\*q) = -2\sum_{i=1}^{\abs{\+S}} \sqrt{\*q(i)}$, the conclusion in Proposition~\ref{prop:ub-l-OSMD} can be directly derived from Theorem 11 in~\cite{LG21}.

\section{More on the BAR algorithm}\label{sec:ub-l-warmup}

In \Cref{sec:ub-l}, \Cref{algo:large-m} essentially implements a streaming-friendly BAR algorithm in each pass. This BAR algorithm serves two purposes: (1) After the $p$-th pass, it guarantees that an $\eps_p$-optimal arm is retained in memory, denoted as $\k_p$; (2) it incurs relatively small regret in pass $p$. For a better understanding of this process, we provide an algorithm for BAR problem in \Cref{algo:BAR} in the offline setting, that is, all arms are present in the memory. It is logically equivalent to the routine in \Cref{algo:large-m}.


Let $\eps_1<\eps_0$ be two parameters and $m\ge \frac{8n}{9}$. Our BAR algorithm executes the $\textsc{FindBest}$ subroutine twice. The first one is short and is performed on all the input arms. The purpose of this $\textsc{FindBest}$ subroutine is to identify an $\eps_0$-optimal arm, denoted as $\a'_1$. Since the subroutine is short, its regret is small.

The second one, which is longer, is performed on a set $\+S'$ that includes $\a'_1$ and $n-m+1$ randomly selected arms. As a result, the regret caused by this subroutine is with respect to the $\eps_0$-optimal arm $\a'_1$, so it can be bounded as well. In the second $\textsc{FindBest}$ subroutine, we obtain $\a'_2$. The objective of the second one is to eliminate $n-m$ inferior arms and output a set of size $m$ containing an $O\tp{\eps_1}$-optimal arm.

Set $L_1=\ceil{\frac{2n}{\eps_0^2}}$ and $L_2=\ceil{\frac{2(n-m+2)}{\eps_1^2}}$. The algorithm is shown in \Cref{algo:BAR}.
\begin{algorithm}[ht]
    \caption{The algorithm for Best Arm Retention}
    \label{algo:BAR}
    \Input{a set of $n$ arms $\+S$,  parameters $\eps_0,\eps_1\in(0,1)$ where $\eps_0<\eps_1$}\\
    \Output{an arm set of size $m$}
\begin{algorithmic}[1]
    \Procedure{\textsc{BAR}}{$\+S,\eps_0,\eps_1$}
    \State $\a'_{1}$ = \Call{FindBest}{$\+S,L_{1}$}\;
    \State Choose $n-m+1$ arms from $\+S\setminus\set{\a'_{1}}$ arms uniformly at random and let $\+S'$ be the set of these arms plus  $\a'_{1}$\label{line:warmup-l-sample}\;
    \State $\a'_{2}$ = \Call{FindBest}{$\+S',L_{2}$}\;
    \State Choose $n-m$ arms in $\+S'\setminus\set{\a'_{2},\a'_{1}}$ uniformly at random to drop\;
    \State \Return the remaining arms\;
    \EndProcedure
\end{algorithmic}
\end{algorithm}


Let $\k$ be the best arm among the output set of \Cref{algo:BAR}. Then we have the following lemma.
\begin{lemma}\label{lem:BAR}
    We have $\E{\mu_{\a^*} - \mu_{\k}}\leq \frac{2(n-m)\eps_1}{m}$. 
\end{lemma}
\begin{proof}
    Note that $\a^*$ is dropped only if $\a^*\in \+S'\setminus\set{\a'_{2},\a'_{1}}$. Then we have
    \begin{align*}
        &\phantom{{}={}}\E{\mu_{\a^*} -\mu_{\k}}\\
        &= \Pr{\a^*\in \+S'\setminus\set{\a'_{1}}}\cdot \E{\mu_{\a^*} - \mu_{\k}\mid \a^*\in \+S'\setminus\set{\a'_{1}}}\\ 
        &\leq \Pr{\a^*\in \+S'\mid \a^*\neq \a'_{1}}\Pr{\a^*\neq\a'_{1}}  \cdot \E{\mu_{\a^*} - \mu_{\a'_{2}}\mid \a^*\in \+S'\setminus\set{\a'_{1}}} \\
        & = \frac{n-m+1}{n-1} \cdot\Pr{\a^*\neq\a'_{1}} \E{\mu_{\a^*} - \mu_{\a'_{2}}\mid \a^*\in \+S'\setminus\set{\a'_{1}}}\\
        & \leq \frac{n-m+1}{n-1} \cdot \E{\mu_{\a^*} - \mu_{\a'_{2}}\mid \a^*\in \+S'\setminus\set{\a'_{1}}}  \\
        &\overset{(\spadesuit)}{\leq} \frac{n-m+1}{n-1}\sqrt{\frac{2(n-m+2)}{L_{2}}} \leq \frac{2(n-m)}{n}\eps_1,
    \end{align*}
    where $(\spadesuit)$ follows from Lemma~\ref{lem:ub-bandit-subroutine-simple}.
\end{proof}

Lemma~\ref{lem:BAR} shows why we choose to include an additional $n-m+1$ arms in the Line~\ref{line:warmup-l-sample} of \Cref{algo:BAR} instead of $n-m$ arms, and why we need to ensure that $\a'_{1}$ will not be dropped. By doing so, the premise condition for the output not containing $\a^*$ is that $\a^*\in  \+S'\setminus\set{\a'_{1}}$. The probability of this event can then be bounded by $\frac{n-m+1}{n-1}$.

Now we show a regret bound of \Cref{algo:BAR} in the following lemma.
\begin{lemma}
    The expected regret generated by \Cref{algo:BAR} is no larger than $O\tp{{\frac{n}{\eps_0}} + \frac{(n-m)\eps_0}{\eps_1^2}}$.
\end{lemma}
\begin{proof}
    The regret of the first $\textsc{FindBest}$ subroutine can be bounded by $\sqrt{2nL_1}$ according to Proposition~\ref{prop:ub-l-OSMD}. From Lemma~\ref{lem:ub-bandit-subroutine-simple}, we have
    \[
        \E{\mu_{\a^*} -\mu_{\a'_1}}\leq  \sqrt{\frac{2n}{L_1}} \leq \eps_0.
    \]
    Then the regret of the second $\textsc{FindBest}$ subroutine can be decomposed into two parts: the regret of $\textsc{MirrorDescent}$, which can be bounded by $\sqrt{2(n-m+2)L_2}$ and the regret generated due to not containing $\a^*$ in $\+S'$, which can be bounded by $\E{\mu_{\a^*}-\mu_{\a'_1}}\cdot L_2$.
    
    Therefore, the total regret is no larger than
    \begin{align*}
        \sqrt{2nL_1} + \sqrt{2(n-m+2)L_2} + \E{\mu_{\a^*}-\mu_{\a'_1}}\cdot L_2 &\leq \sqrt{2nL_1} + \sqrt{2(n-m+2)L_2} + \eps_0\cdot L_2 \\
        &=O\tp{{\frac{n}{\eps_0}} + \frac{(n-m)\eps_0}{\eps_1^2}}.
    \end{align*}
\end{proof}

Now let us take a closer look at \Cref{algo:large-m}. Essentially, what we want to do is perform two $\textsc{FindBest}$ operations on the $m$ arms in memory at the beginning of each pass, then discard $n-m$ arms, and read in the remaining $n-m$ arms. However, in order to adapt to the streaming setting, we cannot do this directly. This is because we cannot determine whether an arriving arm is one of the discarded $n-m$ arms or one of the $n-m$ arms that we have not encountered since we are only allowed to store the statistics of arms in memory. Therefore, we first discard $\frac{m}{2}$ arms and then select $n-m$ arms from the newly arrived $\frac{m}{2}$ arms to drop. At this point, there remains only $n-m$ arms in stream that are not in current memory. This satisfies the requirements of the streaming model while still achieving the desired effectiveness of our BAR algorithm.

\end{document}